\documentclass{article}

\usepackage[accepted]{icml2013}
\usepackage{amsmath,amsthm,amssymb}
\usepackage{algorithm,algorithmic}
\usepackage{stmaryrd,dsfont}
\usepackage[pdftex]{graphicx}
\usepackage{subfigure}
\usepackage{hyperref}
\usepackage[normalem]{ulem}
\usepackage[usenames,dvipsnames]{xcolor}
\usepackage{comment}
\usepackage{enumerate}

\newtheorem{lem}{Lemma}
\newtheorem{thm}[lem]{Theorem}

\newcommand{\R}{{\mathbb R}}
\newcommand{\Z}{{\mathbb Z}}
\newcommand{\RR}{{\cal R}}

\renewcommand{\H}{{\cal H}}
\newcommand{\X}{{\cal X}}

\renewcommand{\L}{{\cal L}}

\newcommand{\Y}{{\cal Y}}

\newcommand{\F}{{\cal F}}
\newcommand{\W}{{\cal W}}

\renewcommand{\S}{{\cal S}}

\renewcommand{\P}{{\cal P}}

\newcommand{\M}{{\cal M}}
\newcommand{\B}{{\cal B}}
\newcommand{\ZZ}{{\cal Z}}

\newcommand{\bc}[1]{\left\{{#1}\right\}}
\newcommand{\br}[1]{\left({#1}\right)}
\newcommand{\bs}[1]{\left[{#1}\right]}
\newcommand{\abs}[1]{\left| {#1} \right|}
\newcommand{\norm}[1]{\left\| {#1} \right\|}

\newcommand{\bsd}[1]{\left\llbracket{#1}\right\rrbracket}

\renewcommand{\O}[1]{{\cal O}\br{{#1}}}
\newcommand{\softO}[1]{\tilde{\cal O}\br{{#1}}}
\newcommand{\Om}[1]{\Omega\br{{#1}}}
\newcommand{\E}[1]{{\mathbb E}\bsd{{#1}}}

\newcommand{\EE}[2]{\underset{#1}{\mathbb E}\bsd{{#2}}}

\renewcommand{\Pr}[1]{{\mathbb P}\bs{{#1}}}
\newcommand{\Prr}[2]{\underset{#1}{\mathbb P}\bs{{#2}}}
\newcommand{\ip}[2]{\left\langle{#1},{#2}\right\rangle}

\renewcommand{\vec}[1]{{\mathbf{#1}}}

\newcommand{\vecx}{\vec{x}}

\newcommand{\vecz}{\vec{z}}

\newcommand{\vecw}{\vec{w}}
\newcommand{\vecW}{\vec{W}}
\newcommand{\vecX}{\vec{X}}

\newcommand{\vecsigma}{\boldsymbol\sigma}
\newcommand{\vecmu}{\boldsymbol\mu}
\newcommand{\vecr}{\vec{r}}

\newcommand{\veczero}{\vec{0}}

\newcommand{\ind}{\mathds{1}}

\renewcommand{\th}{^{\text{th}}}

\newcommand{\rk}{{\mathfrak r}}
\newcommand{\Rk}{{\mathfrak R}}
\newcommand{\Vk}{{\mathfrak V}}
\newcommand{\Wk}{{\mathfrak W}}
\newcommand{\Dk}{{\mathfrak D}}

\newcommand{\mrs}{{\bfseries RS-x }}
\newcommand{\mmrs}{{\bfseries RS-x$\mathbf{^2}$ }}
\newcommand{\rsorig}{{\bfseries RS }}
\newcommand{\olp}{{\bfseries OLP }}
\newcommand{\Lbuft}{\hat\L^{\text{buf}}_t}
\newcommand{\Ltbuft}{\tilde\L^{\text{buf}}_t}
\newcommand{\Rkbufn}{\Rk_n^{\text{buf}}}

\def\papertitle{On the Generalization Ability of Online Learning Algorithms for Pairwise Loss Functions}

\graphicspath{{.}{figs/}}

\icmltitlerunning{\papertitle}

\begin{document} 

\twocolumn[
\icmltitle{\papertitle}

\icmlauthor{Purushottam Kar}{purushot@cse.iitk.ac.in}
\icmladdress{Department of Computer Science and Engineering, Indian Institute of
Technology, Kanpur, UP 208 016, INDIA.}
\icmlauthor{Bharath K Sriperumbudur}{bs493@statslab.cam.ac.uk}
\icmladdress{Statistical Laboratory, Centre for Mathematical Sciences,
Wilberforce Road, Cambridge, CB3 0WB, ENGLAND.}
\icmlauthor{Prateek Jain}{prajain@microsoft.com}
\icmladdress{Microsoft Research India, ``Vigyan'', \#9, Lavelle Road, Bangalore,
KA 560 001, INDIA.}
\icmlauthor{Harish C Karnick}{hk@cse.iitk.ac.in}
\icmladdress{Department of Computer Science and Engineering, Indian Institute of
Technology, Kanpur, UP 208 016, INDIA.}

\icmlkeywords{Generalization bounds, online learning, kernel learning, buffer
management}

\vskip 0.3in
]

\begin{abstract}
In this paper, we study the generalization properties of online learning based
stochastic methods for supervised learning problems where the loss function is
dependent on more than one training sample (e.g., metric learning, ranking). We
present a generic decoupling technique that enables us to provide Rademacher
complexity-based generalization error bounds. Our bounds are in general tighter
than those obtained by \citet{online-batch-pairwise} for the same problem. Using
our decoupling technique, we are further able to obtain fast convergence rates
for strongly convex pairwise loss functions. We are also able to analyze a class
of memory efficient online learning algorithms for pairwise learning problems that use
only a bounded subset of past training samples to update the hypothesis at each
step. Finally, in order to complement our generalization bounds, we propose a
novel memory efficient online learning algorithm for higher order learning
problems with bounded regret guarantees.
\end{abstract}

\section{Introduction}
\label{sec:intro}
Several supervised learning problems involve working with pairwise or higher order loss functions, i.e., loss functions that depend on more than one training sample. Take for example the \emph{metric learning} problem \cite{reg-metric-learn}, where the goal is to learn a metric $M$ that brings points of a similar label together while keeping differently labeled points apart. In this case the loss function used is a pairwise loss function $\ell(M,(\vecx,y),(\vecx',y')) = \phi\br{yy'\br{1- M(\vecx,\vecx')}}$ where $\phi$ is the hinge loss function. In general, a pairwise loss function is of the form $\ell: \H \times \X \times \X \rightarrow \R^+$ where $\H$ is the hypothesis space and $\X$ is the input domain.
Other examples include preference learning \cite{XingNJR02}, ranking \cite{AgarwalN09}, AUC maximization \cite{oam-icml} and multiple kernel learning \cite{two-stage-mkl-general}. 

In practice, algorithms for such problems use intersecting pairs of training samples to learn. Hence the training data pairs are not i.i.d.~and consequently, standard generalization error analysis techniques do not apply to these algorithms. Recently, the analysis of \emph{batch} algorithms learning from such coupled samples has received much attention \cite{gen-bound-metric-learn, u-stat-rank, BrefeldS05} where a dominant idea has been to use an alternate representation of the U-statistic and provide uniform convergence bounds. Another popular approach has been to use algorithmic stability \cite{AgarwalN09,reg-metric-learn} to obtain algorithm-specific results. 

While batch algorithms for pairwise (and higher-order) learning problems have been studied well theoretically,  online learning based stochastic algorithms are more popular in practice due to their scalability. However, their generalization properties were not studied until recently. \citet{online-batch-pairwise} provided the first generalization error analysis of online learning methods applied to  pairwise loss functions. In particular, they showed that such higher-order online learning methods also admit online to batch conversion bounds (similar to those for first-order problems \cite{online-batch-single}) which can be combined with regret bounds to obtain generalization error bounds. However, due to their proof technique and dependence on $L_\infty$ covering numbers of function classes, their bounds are not tight and have a strong dependence on the dimensionality of the input space.

In literature, there are several instances where Rademacher complexity based techniques achieve sharper bounds than those based on covering numbers \cite{rad-bounds}. However, the coupling of different input pairs in our problem does not allow us to use such techniques directly.

In this paper we introduce a generic technique for analyzing online learning algorithms for higher order learning problems. Our technique, that uses an extension of Rademacher complexities to higher order function classes (instead of covering numbers), allows us to give bounds that are tighter than those of \cite{online-batch-pairwise} and that, for several learning scenarios, have no dependence on input dimensionality at all.

Key to our proof is a technique we call \emph{Symmetrization of Expectations} which acts as a decoupling step and allows us to reduce excess risk estimates to Rademacher complexities of function classes. \cite{online-batch-pairwise}, on the other hand, perform a symmetrization with probabilities which, apart from being more involved, yields suboptimal bounds. Another advantage of our technique is that it allows us to obtain \emph{fast} convergence rates for learning algorithms that use {\em strongly convex} loss functions. Our result, that uses a novel two stage proof technique, extends a similar result in the first order setting by \citet{online-batch-strongly-convex} to the pairwise setting.

\citet{online-batch-pairwise} (and our results mentioned above) assume an online learning setup in which a stream of points $\vecz_1,\ldots,\vecz_n$ is observed and the penalty function used at the $t\th$ step is $\hat\L_t(h) = \frac{1}{t-1}\sum_{\tau=1}^{t-1} \ell (h, \vecz_t, \vecz_\tau)$. Consequently, the results of \citet{online-batch-pairwise} expect regret bounds with respect to these \emph{all-pairs} penalties $\hat\L_t$. This requires one to use/store all previously seen points which is computationally/storagewise expensive and hence in practice, learning algorithms update their hypotheses using only a bounded subset of the past samples \cite{oam-icml}.

%
In the above mentioned setting, we are able to give generalization bounds that only require algorithms to give regret bounds with respect to \emph{finite-buffer} penalty functions such as $\Lbuft(h) = \frac{1}{\abs{B}}\sum_{\vecz \in B}\ell(h,\vecz_t,\vecz)$ where $B$ is a \emph{buffer} that is updated at each step.
Our proofs hold for any \emph{stream oblivious} buffer update policy including FIFO and the widely used reservoir sampling policy \cite{vitter-rs, oam-icml}\footnote{Independently, \citet{online-batch-pairwise-arxiv} also extended their proof to give similar guarantees. However, their bounds hold only for the FIFO update policy and have worse dependence on dimensionality in several cases (see Section~\ref{sec:finite-buffer}).}.

To complement our online to batch conversion bounds, we also provide a memory efficient online learning algorithm that works with bounded buffers. Although our algorithm is constrained to observe and learn using the \emph{finite-buffer} penalties $\Lbuft$ alone, we are still able to provide high confidence regret bounds with respect to the \emph{all-pairs} penalty functions $\hat\L_t$. We note that \citet{oam-icml} also propose an algorithm that uses finite buffers and claim an \emph{all-pairs} regret bound for the same. However, their regret bound does not hold due to a subtle mistake in their proof. 

We also provide empirical validation of our proposed online learning algorithm on AUC maximization tasks and show that our algorithm performs competitively with that of \cite{oam-icml}, in addition to being able to offer theoretical regret bounds. 

{\bf Our Contributions}:
\begin{enumerate}[(a)]
	\item We provide a generic online-to-batch conversion technique for higher-order supervised learning problems offering bounds that are sharper than those of \cite{online-batch-pairwise}.
	\item We obtain fast convergence rates when loss functions are \emph{strongly convex}.
	\item We analyze online learning algorithms that are constrained to learn using a finite buffer.
	\item We propose a novel online learning algorithm that works with finite buffers but is able to provide a high confidence regret bound with respect to the \emph{all-pairs} penalty functions.
\end{enumerate}

\section{Problem Setup}
\label{sec:setup}
For ease of exposition, we introduce an online learning model for higher order supervised learning problems in this section; concrete learning instances such as AUC maximization and metric learning are given in Section~\ref{sec:apps}. For sake of simplicity, we restrict ourselves to pairwise problems in this paper; our techniques can be readily extended to higher order problems as well. 

For pairwise learning problems, our goal is to learn a real valued \emph{bivariate} function $h^\ast : \X \times \X \rightarrow \Y$, where $h^* \in \H$, under some loss function $\ell : \H \times \ZZ \times \ZZ \rightarrow \R^+$ where $\ZZ = \X \times \Y$.

The online learning algorithm is given sequential access to a stream of elements $\vecz_1, \vecz_2, \ldots, \vecz_n$ chosen i.i.d. from the domain $\ZZ$. Let $Z^t := \bc{\vecz_1, \ldots, \vecz_t}$. At each time step $t = 2 \ldots n$, the algorithm posits a hypothesis $h_{t-1} \in \H$ upon which the element $\vecz_t$ is revealed and the algorithm incurs the following penalty: 
\begin{equation}
\hat\L_t(h_{t-1}) = \frac{1}{t-1}\sum_{\tau = 1}^{t-1}\ell(h_{t-1},\vecz_t,\vecz_\tau).
\label{eq:infinite-buffer-loss}
\end{equation}
For any $h \in \H$, we define its expected risk as:
\begin{equation}
\label{eq:loss_true}\L(h) := \EE{\vecz,\vecz'}{\ell(h,\vecz,\vecz')}.
\end{equation}
Our aim is to present an ensemble $h_1,\ldots, h_{n-1}$ such that the expected risk of the ensemble is small. More specifically, we desire that, for some small $\epsilon > 0$,
\[
\frac{1}{n-1}\sum_{t=2}^n\L(h_{t-1}) \leq \L(h^\ast) + \epsilon,
\]
where $h^\ast = \underset{h \in \H}{\arg\min}\ \L(h)$ is the population risk minimizer. Note that this allows us to do hypothesis selection in a way that ensures small expected risk. Specifically, if one chooses a hypothesis as $\hat{h} := \frac{1}{(n-1)}\sum_{t=2}^n h_{t-1}$ (for convex $\ell$) or $\hat{h} := \underset{t = 2,\ldots,n}{\arg\min}\ \L(h_t)$ then we have $\L(\hat h) \leq \L(h^\ast) + \epsilon$.


Since the model presented above requires storing all previously seen points, it becomes unusable in large scale learning scenarios. Instead, in practice, a \emph{sketch} of the stream is maintained in a buffer $B$ of capacity $s$.  
 At  each step, the penalty is now incurred only on the pairs $\bc{(\vecz_t,\vecz) : \vecz \in B_t}$ where $B_t$ is the state of the buffer at time $t$. That is, \vspace*{-1ex}
\begin{equation}
\Lbuft(h_{t-1}) = \frac{1}{\abs{B_t}}\sum_{\vecz \in B_t}\ell(h_{t-1},\vecz_t,\vecz).
\label{eq:finite-buffer-loss}
\end{equation}
We shall assume that the buffer is updated at each step using some \emph{stream oblivious} policy such as FIFO or Reservoir sampling \cite{vitter-rs} (see Section~\ref{sec:finite-buffer}).

In Section~\ref{sec:simply-convex}, we present online-to-batch conversion bounds for online learning algorithms that give regret bounds w.r.t.~penalty functions given by \eqref{eq:infinite-buffer-loss}. In Section~\ref{sec:strongly-convex}, we extend our analysis to algorithms using strongly convex loss functions. In Section~\ref{sec:finite-buffer} we provide generalization error bounds for algorithms that give regret bounds w.r.t.~\emph{finite-buffer} penalty functions given by \eqref{eq:finite-buffer-loss}. Finally in section~\ref{sec:algo} we present a novel memory efficient online learning algorithm with regret bounds. 

\section{Online to Batch Conversion Bounds for Bounded Loss Functions}
\label{sec:simply-convex}
We now present our generalization bounds for algorithms that provide regret bounds with respect to the \emph{all-pairs} loss functions
(see Eq.~\eqref{eq:infinite-buffer-loss}). Our results give tighter bounds and have a much better dependence on input dimensionality than the bounds given by \citet{online-batch-pairwise}. See Section~\ref{sec:khardon-comparison} for a detailed comparison. 

As was noted by \cite{online-batch-pairwise}, the generalization error analysis of online learning algorithms in this setting does not follow from existing techniques for first-order problems (such as \cite{online-batch-single, online-batch-strongly-convex}). The reason is that the terms $V_t = \hat\L_t(h_{t-1})$ do not form a martingale due to the intersection of training samples in $V_t$ and $V_{\tau}, \tau<t$. 

Our technique, that aims to utilize the Rademacher complexities of function classes in order to get tighter bounds, faces yet another challenge at the \emph{symmetrization} step, a precursor to the introduction of Rademacher complexities. It turns out that, due to the coupling between the ``head'' variable $\vecz_t$ and the ``tail'' variables $\vecz_\tau$ in the loss function $\hat\L_t$, a standard symmetrization between true $\vecz_\tau$ and ghost $\tilde\vecz_\tau$ samples does not succeed in generating Rademacher averages and instead yields complex looking terms.

More specifically, suppose we have \emph{true} variables $\vecz_t$ and \emph{ghost} variables $\tilde \vecz_t$ and are in the process of bounding the expected excess risk by analyzing expressions of the form
\[
E_{\text{orig}} = \ell(h_{t-1},\vecz_t,\vecz_\tau) - \ell(h_{t-1},\tilde \vecz_t,\tilde \vecz_\tau).
\]
Performing a traditional symmetrization of the variables $\vecz_\tau$ with $\tilde\vecz_\tau$ would give us expressions of the form
\[
E_{\text{symm}} = \ell(h_{t-1},\vecz_t,\tilde \vecz_\tau) - \ell(h_{t-1},\tilde \vecz_t,\vecz_\tau).
\]
At this point the analysis hits a barrier since unlike first order situations, we cannot relate $E_{\text{symm}}$ to $E_{\text{orig}}$ by means of introducing Rademacher variables.


We circumvent this problem by using a technique that we call \emph{Symmetrization of Expectations}. The technique allows us to use standard symmetrization to obtain Rademacher complexities. More specifically, we analyze expressions of the form
\[
E'_{\text{orig}} = \EE{\vecz}{\ell(h_{t-1},\vecz,\vecz_\tau)} - \EE{\vecz}{\ell(h_{t-1},\vecz,\tilde \vecz_\tau)}
\]
which upon symmetrization yield expressions such as
\[
E'_{\text{symm}} = \EE{\vecz}{\ell(h_{t-1},\vecz,\tilde \vecz_\tau)} - \EE{\vecz}{\ell(h_{t-1},\vecz,\vecz_\tau)}
\]
which allow us to introduce Rademacher variables since $E'_{\text{symm}} = -E'_{\text{orig}}$. This idea is exploited by the lemma given below that relates the expected risk of the ensemble to the penalties incurred during the online learning process. In the following we use the following extension of Rademacher averages \cite{rad-bounds} to bivariate function classes:
\[
\RR_n(\H) = \E{\underset{h \in \H}{\sup}\ \frac{1}{n}\sum_{\tau=1}^n \epsilon_\tau h(\vecz,\vecz_\tau)}
\]
where the expectation is over $\epsilon_\tau$, $\vecz$ and $\vecz_\tau$. We shall denote composite function classes as follows : $\ell\circ\H := \bc{(\vecz,\vecz') \mapsto \ell(h,\vecz,\vecz'), h \in \H}$.



\begin{lem}
\label{lem:bounded-forward}
Let $h_1,\ldots,h_{n-1}$ be an ensemble of hypotheses generated by an online learning algorithm working with a bounded loss function $\ell : \H \times \ZZ \times \ZZ \rightarrow [0,B]$. Then for any $\delta > 0$, we have with probability at least $1 - \delta$,
\begin{eqnarray*}
\lefteqn{\frac{1}{n-1}\sum_{t=2}^n\L(h_{t-1}) \leq \frac{1}{n-1}\sum_{t=2}^n\hat\L_t(h_{t-1})}\\
									 	&& \mbox{} + \frac{2}{n-1}\sum_{t=2}^n\RR_{t-1}(\ell\circ\H) + 3B\sqrt\frac{\log\frac{n}{\delta}}{n-1}.
\end{eqnarray*}
\end{lem}
The proof of the lemma involves decomposing the excess risk term into a martingale difference sequence and a residual term in a manner similar to \cite{online-batch-pairwise}. The martingale sequence, being a bounded one, is shown to converge using the Azuma-Hoeffding inequality. The residual term is handled using uniform convergence techniques involving Rademacher averages. The complete proof of the lemma is given in the Appendix~\ref{appsec:proof-lem-bounded-forward}.

Similar to Lemma~\ref{lem:bounded-forward}, the following converse relation between the population and empirical risk of the population risk minimizer $h^\ast$ can also be shown. 
\begin{lem}
\label{lem:bounded-converse}
For any $\delta > 0$, we have with probability at least $1 - \delta$,\vspace*{-3ex}

\setlength{\arraycolsep}{0.0em}
\begin{eqnarray*}
\frac{1}{n-1}\sum\limits_{t=2}^n\hat\L_t(h^\ast) &{}\leq{}& \L(h^\ast) + \frac{2}{n-1}\sum\limits_{t=2}^n\RR_{t-1}(\ell\circ\H)\\
&{}{}&\qquad+ 3B\sqrt\frac{\log\frac{1}{\delta}}{n-1}. 
\end{eqnarray*}
\end{lem}
An online learning algorithm will be said to have an \emph{all-pairs} regret bound $\Rk_n$ if it presents an ensemble $h_1, \ldots, h_{n-1}$ such that
\[
\sum_{t=2}^n\hat\L_t(h_{t-1}) \leq \underset{h \in \H}{\inf}\sum_{t=2}^n\hat\L_t(h) + \Rk_n.
\]
Suppose we have an online learning algorithm with a regret bound $\Rk_n$. Then combining Lemmata~\ref{lem:bounded-forward} and \ref{lem:bounded-converse} gives us the following online to batch conversion bound:
\begin{thm}
\label{thm:simply-convex-final-bound}
Let $h_1,\ldots,h_{n-1}$ be an ensemble of hypotheses generated by an online learning algorithm working with a $B$-bounded loss function $\ell$ that guarantees a regret bound of $\Rk_n$. Then for any $\delta > 0$, we have with probability at least $1 - \delta$,
\setlength{\arraycolsep}{0.0em}
\begin{eqnarray*}
\frac{1}{n-1}\sum_{t=2}^n\L(h_{t-1}) &{}\leq {}& \L(h^\ast) + \frac{4}{n-1}\sum_{t=2}^n\RR_{t-1}(\ell\circ\H)\\
									 	&{}{}&+ \frac{\Rk_n}{n-1}  + 6B\sqrt\frac{\log\frac{n}{\delta}}{n-1}.
\end{eqnarray*}
\end{thm}
As we shall see in Section~\ref{sec:apps}, for several learning problems, the Rademacher complexities behave as $\RR_{t-1}(\ell\circ\H) \leq C_d\cdot\O{\frac{1}{\sqrt{t-1}}}$ where $C_d$ is a constant dependent only on the dimension $d$ of the input space and the $\O{\cdot}$ notation hides constants dependent on the domain size and the loss function. This allows us to bound the excess risk as follows: \vspace*{-4ex}

{\small
\begin{align*}
\frac{\sum_{t=2}^n\L(h_{t-1})}{n-1} \leq {} & \L(h^\ast) + \frac{\Rk_n}{n-1}+  \O{\frac{C_d+ \sqrt{\log(n/\delta)}}{\sqrt{n-1}}}. 
\end{align*}
}\vspace*{-3ex}

Here, the error decreases with $n$ at a standard $1/\sqrt{n}$ rate (up to a $\sqrt{\log n}$ factor), similar to that obtained by \citet{online-batch-pairwise}. However, for several problems the above bound can be significantly tighter than those offered by covering number based arguments. We provide below a detailed comparison of our results with those of \citet{online-batch-pairwise}. 

\subsection{Discussion on the nature of our bounds}
\label{sec:khardon-comparison}
As mentioned above, our proof enables us to use Rademacher complexities which are typically easier to analyze and provide tighter bounds \cite{rad-bounds}. In particular, as shown in Section~\ref{sec:apps}, for $L_2$ regularized learning formulations, the Rademacher complexities are dimension independent i.e. $C_d = 1$. Consequently, unlike the bounds of \cite{online-batch-pairwise} that have a linear dependence on $d$, our bound becomes independent of the input space dimension.
For sparse learning formulations with $L_1$ or trace norm regularization, we have $C_d = \sqrt{\log d}$ giving us a mild dependence on the input dimensionality. 

Our bounds are also tighter that those of \cite{online-batch-pairwise} in general. Whereas we provide a confidence bound of $\delta < \exp\br{-n\epsilon^2 + \log n}$, \cite{online-batch-pairwise} offer a weaker bound $\delta < (1/\epsilon)^d\exp\br{-n\epsilon^2 + \log n}$. 

An artifact of the proof technique of \cite{online-batch-pairwise} is that their proof is required to exclude a constant fraction of the ensemble ($h_1, \dots, h_{cn}$) from the analysis, failing which their bounds turn vacuous. 
Our proof on the other hand is able to give guarantees for the \emph{entire} ensemble. 

In addition to this, as the following sections show, our proof technique enjoys the flexibility of being extendable to give fast convergence guarantees for strongly convex loss functions as well as being able to accommodate learning algorithms that use finite buffers.



\section{Fast Convergence Rates for Strongly Convex Loss Functions}
\label{sec:strongly-convex}
In this section we extend results of the previous section to give \emph{fast} convergence guarantees for online learning algorithms that use strongly convex loss functions of the following form: 
$\ell(h,\vecz,\vecz') = g(\ip{h}{\phi(\vecz,\vecz')}) + r(h),$
where $g$ is a convex function and $r(h)$ is a $\sigma$-strongly convex regularizer (see Section~\ref{sec:apps} for examples) i.e.
$\forall h_1, h_2 \in \H$ and $\alpha \in [0,1]$, we have 
\begin{align*}
r(\alpha h_1 + (1-\alpha)h_2) \leq {} & \alpha r(h_1) + (1-\alpha)r(h_2)\\
									  & - \frac{\sigma}{2}\alpha(1-\alpha)\norm{h_1 - h_2}^2. 
\end{align*}
For any norm $\norm{\cdot}$, let $\norm{\cdot}_\ast$ denote its dual norm.
Our analysis reduces the pairwise problem to a first order problem and a martingale convergence problem. We require the following \emph{fast} convergence bound in the standard first order {\em batch} learning setting: 
\begin{thm}
\label{thm:fast-batch}
Let $\F$ be a closed and convex set of functions over $\X$. Let $\wp(f,\vecx) = p(\ip{f}{\phi(\vecx)}) + r(f)$, for a $\sigma$-strongly convex function $r$, be a loss function with $\P$ and $\hat\P$ as the associated population and empirical risk functionals and $f^\ast$ as the population risk minimizer. Suppose $\wp$ is $L$-Lipschitz and $\norm{\phi(\vecx)}_\ast \leq R, \forall \vecx \in \X$.
Then w.p. $1 - \delta$, for any $\epsilon > 0$, we have for all $f\in \F$,
\begin{align*}
\P(f) - \P(f^\ast) \leq {} & (1 + \epsilon)\br{\hat\P(f) - \hat\P(f^\ast)} + \frac{C_\delta}{\epsilon \sigma n}
\end{align*}
where $C_\delta = C_d^2\cdot (4(1 + \epsilon)LR)^2\br{32 + \log(1/\delta)}$ and $C_d$ is the dependence of the Rademacher complexity of the class $\F$ on the input dimensionality $d$.
\end{thm}
The above theorem is a minor modification of a similar result by \citet{fast-rate-reg-obj} and the proof (given in Appendix~\ref{appsec:proof-thm-fast-batch}) closely follows their proof as well. 
We can now state our online to batch conversion result for strongly convex loss functions.

\begin{thm}
\label{thm:strongly-convex-forward}
Let $h_1,\ldots,h_{n-1}$ be an ensemble of hypotheses generated by an online learning algorithm working with a $B$-bounded, $L$-Lipschitz and $\sigma$-strongly convex loss function $\ell$. 
 Further suppose the learning algorithm guarantees a regret bound of $\Rk_n$. Let $\Vk_n = \max\bc{{\Rk_n},2C_d^2\log n\log(n/\delta)}$ Then for any $\delta > 0$, we have with probability at least $1 - \delta$,
\begin{eqnarray*}
\lefteqn{\frac{1}{n-1}\sum_{t=2}^n\L(h_{t-1}) \leq \L(h^\ast) + \frac{\Rk_n}{n-1}}\\
											 &&\qquad\qquad + C_d\cdot\O{\frac{\sqrt{\Vk_n\log n\log(n/\delta)}}{n-1}},
\end{eqnarray*}
where the $\O{\cdot}$ notation hides constants dependent on domain size and the loss function such as $L, B$ and $\sigma$.
\end{thm}
The decomposition of the excess risk in this case is not made explicitly but rather emerges as a side-effect of the proof progression. The proof starts off by applying Theorem~\ref{thm:fast-batch} to the hypothesis in each round with the following loss function $\wp(h,\vecz') := \EE{\vecz}{\ell(h,\vecz,\vecz')}$. Applying the regret bound to the resulting expression gives us a martingale difference sequence which we then bound using Bernstein-style inequalities and a proof technique from \cite{online-batch-strongly-convex}. The complete proof is given in Appendix~\ref{appsec:proof-thm-strongly-convex-forward}.

We now note some properties of this result. The effective dependence of the above bound on the input dimensionality is $C_d^2$ since the expression $\sqrt\Vk_n$ hides a $C_d$ term. We have $C_d^2 = 1$ for non sparse learning formulations and $C_d^2 = \log d$ for sparse learning formulations. We note that our bound matches that of \citet{online-batch-strongly-convex} (for {\em first}-order learning problems) up to a logarithmic factor.  


\section{Analyzing Online Learning Algorithms that use Finite Buffers}
\label{sec:finite-buffer}

In this section, we present our online to batch conversion bounds for algorithms that work with \emph{finite-buffer} loss functions $\Lbuft$. Recall that an online learning algorithm working with finite buffers incurs a loss $\Lbuft(h) = \frac{1}{\abs{B_t}}\sum_{\vecz \in B_t}\ell(h_{t-1},\vecz_t,\vecz)$ at each step where $B_t$ is the state of the buffer at time $t$.

An online learning algorithm will be said to have a \emph{finite-buffer} regret bound $\Rkbufn$ if it presents an ensemble $h_1, \ldots, h_{n-1}$ such that \vspace*{-2ex}
\begin{align*}
\sum_{t=2}^n\Lbuft(h_{t-1}) - \underset{h \in \H}{\inf}\sum_{t=2}^n\Lbuft(h) \leq \Rkbufn. 
\end{align*}

\vspace*{-2ex}
For our guarantees to hold, we require the buffer update policy used by the learning algorithm to be \emph{stream oblivious}. More specifically, we require the buffer update rule to decide upon the inclusion of a particular point $\vecz_i$ in the buffer based only on its stream index $i \in [n]$.  Popular examples of stream oblivious policies include Reservoir sampling \cite{vitter-rs} (referred to as \rsorig henceforth) and FIFO. 
Stream oblivious policies allow us to decouple buffer construction randomness from training sample randomness which makes analysis easier; we leave the analysis of \emph{stream aware} buffer update policies as a topic of future research.  





In the above mentioned setting, we can prove the following online to batch conversion bounds:
\begin{thm}
\label{thm:finite-buffer-online-to-batch}
Let $h_1,\ldots,h_{n-1}$ be an ensemble of hypotheses generated by an online learning algorithm working with a finite buffer of capacity $s$ and a $B$-bounded loss function $\ell$. Moreover, suppose that the algorithm guarantees a regret bound of $\Rkbufn$. Then for any $\delta > 0$, we have with probability at least $1 - \delta$,
\begin{align*}
\frac{\sum_{t=2}^n\L(h_{t-1})}{n-1} \leq {} & \L(h^\ast) + \frac{\Rkbufn}{n-1} + \O{\frac{C_d}{\sqrt s} + B\sqrt\frac{\log\frac{n}{\delta}}{s}}
\end{align*}
If the loss function is Lipschitz and strongly convex as well, then with the same confidence, we have
\begin{align*}
\frac{\sum_{t=2}^n\L(h_{t-1})}{n-1} \leq {} & \L(h^\ast) + \frac{\Rkbufn}{n-1} + C_d\cdot\O{\sqrt{\frac{\Wk_n\log\frac{n}{\delta}}{sn}}}
\end{align*}
where ${\Wk_n} = \max\bc{\Rkbufn,\frac{2C_d^2n\log(n/\delta)}{s}}$ and $C_d$ is the dependence of $\RR_n(\H)$ on the input dimensionality $d$.
\end{thm}
The above bound guarantees an excess error of $\softO{1/s}$ for algorithms (such as Follow-the-leader \cite{log-regret}) that offer logarithmic regret $\Rkbufn = \O{\log n}$. We stress that this theorem is not a direct corollary of our results for the \emph{infinite buffer} case (Theorems~\ref{thm:simply-convex-final-bound} and \ref{thm:strongly-convex-forward}). Instead, our proofs require a more careful analysis of the excess risk in order to accommodate the finiteness of the buffer and the randomness (possibly) used in constructing it.
   
More specifically, care needs to be taken to handle randomized buffer update policies such as \rsorig which introduce additional randomness into the analysis. A naive application of techniques used to prove results for the unbounded buffer case would result in bounds that give non trivial generalization guarantees only for large buffer sizes such as $s = \omega(\sqrt{n})$. Our bounds, on the other hand, only require $s = \tilde\omega(1)$.

Key to our proofs is a conditioning step where we first analyze the conditional excess risk by conditioning upon randomness used by the buffer update policy. Such conditioning is made possible by the stream-oblivious nature of the update policy and thus, stream-obliviousness is required by our analysis. Subsequently, we analyze the excess risk by taking expectations over randomness used by the buffer update policy. The complete proofs of both parts of Theorem~\ref{thm:finite-buffer-online-to-batch} are given in Appendix~\ref{appsec:proof-finite-buffer-thms}.

Note that the above results only require an online learning algorithm to provide regret bounds w.r.t.~the \emph{finite-buffer} penalties $\Lbuft$ and do not require any regret bounds w.r.t~the \emph{all-pairs} penalties $\hat{\L_t}$. 

For instance, the finite buffer based online learning algorithms OAM$_{\text{seq}}$ and OAM$_{\text{gra}}$ proposed in \cite{oam-icml} are able to provide a regret bound w.r.t.~$\Lbuft$ \citep[Lemma 2]{oam-icml} but are not able to do so w.r.t~the \emph{all-pairs} loss function (see Section~\ref{sec:algo} for a discussion). 
Using Theorem~\ref{thm:finite-buffer-online-to-batch}, we are able to give a generalization bound for OAM$_{\text{seq}}$ and OAM$_{\text{gra}}$ and hence explain the good empirical performance of these algorithms as reported in \cite{oam-icml}. Note that \citet{online-batch-pairwise-arxiv} are not able to analyze OAM$_{\text{seq}}$ and OAM$_{\text{gra}}$ since their analysis is restricted to algorithms that use the (deterministic) FIFO update policy whereas OAM$_{\text{seq}}$ and OAM$_{\text{gra}}$ use the (randomized) \rsorig policy of \citet{vitter-rs}.

\section{Applications}
\label{sec:apps}


In this section we make explicit our online to batch conversion bounds for several learning scenarios and also demonstrate their dependence on input dimensionality by calculating their respective Rademacher complexities. Recall that our definition of Rademacher complexity for a pairwise function class is given by,
\[
\RR_n(\H) = \E{\underset{h \in \H}{\sup}\ \frac{1}{n}\sum_{\tau=1}^n \epsilon_\tau h(\vecz,\vecz_\tau)}. 
\]
For our purposes, we would be interested in the Rademacher complexities of \emph{composition classes} of the form $\ell\circ\H := \bc{(\vecz,\vecz') \mapsto \ell(h,\vecz,\vecz'), h \in \H}$ where $\ell$ is some Lipschitz loss function. Frequently we have $\ell(h,\vecz,\vecz') = \phi\br{h(\vecx,\vecx')Y(y,y')}$ where $Y(y,y') = y - y'$ or $Y(y,y') = yy'$ and $\phi : \R \rightarrow \R$ is some margin loss function \cite{steinwart-book}. Suppose $\phi$ is $L$-Lipschitz and $Y = \underset{y,y' \in \Y}{\sup}\abs{Y(y,y')}$. Then we have
\begin{thm}
\label{thm:contraction-rademacher}
$\RR_n(\ell\circ\H) \leq LY\RR_n(\H)$.
\end{thm}
The proof uses standard contraction inequalities and is given in Appendix~\ref{appsec:proof-thm-contraction-rademacher}. This reduces our task to computing the values of $\RR_n(\H)$ which we do using a two stage proof technique (see Appendix~\ref{appsec:apps-supp}). For any subset $X$ of a Banach space and any norm $\norm{\cdot}_p$, we define $\norm{X}_p := \underset{\vecx \in X}{\sup} \norm{\vecx}_p$. Let the domain $\X \subset \R^d$.

\textbf{AUC maximization} \cite{oam-icml}: the goal here is to maximize the area under the ROC curve for a linear classification problem where the hypothesis space $\W \subset \R^d$. We have $h_{\vecw}(\vecx,\vecx') = \vecw^\top\vecx-\vecw^\top\vecx'$ and $\ell(h_\vecw,\vecz,\vecz') = \phi\br{(y-y')h_{\vecw}(\vecx,\vecx')}$ where $\phi$ is the hinge loss.
In case our classifiers are $L_p$ regularized for $p > 1$, we can show that $\RR_n(\W) \leq 2\norm{\X}_q\norm{\W}_p\sqrt{\frac{q - 1}{n}}$ where $q = p/(p-1)$. Using the sparsity promoting $L_1$ regularizer gives us $\RR_n(\W) \leq 2\norm{\X}_\infty \norm{\W}_1\sqrt{\frac{e\log d}{n}}$. Note that we obtain dimension independence, for example when the classifiers are $L_2$ regularized which allows us to bound the Rademacher complexities of kernelized function classes for bounded kernels as well.

\textbf{Metric learning} \cite{reg-metric-learn}: the goal here is to learn a Mahalanobis metric $M_\vecW(\vecx,\vecx') = (\vecx-\vecx')^\top\vecW(\vecx-\vecx')$ using the loss function $\ell(\vecW,\vecz,\vecz') = \phi\br{yy'\br{1- M_\vecW^2(\vecx,\vecx')}}$ for a hypothesis class $\W \subset \R^{d\times d}$. In this case it is possible to use a variety of mixed norm $\norm{\cdot}_{p,q}$ and Schatten norm $\norm{\cdot}_{S(p)}$ regularizations on matrices in the hypothesis class. In case we use trace norm regularization on the matrix class, we get $\RR_n(\W) \leq \norm{\X}_2^2\norm{\W}_{S(1)}\sqrt{\frac{e\log d}{n}}$. The $(2,2)$-norm regularization offers a dimension independent bound $\RR_n(\W) \leq \norm{\X}_2^2\norm{\W}_{2,2}\sqrt{\frac{1}{n}}$. The mixed $(2,1)$-norm regularization offers $\RR_n(\W) \leq \norm{\X}_2\norm{\X}_\infty \norm{\W}_{2,1}\sqrt{\frac{e\log d}{n}}$.

\textbf{Multiple kernel learning} \cite{two-stage-mkl-general}: the goal here is to improve the SVM classification algorithm by learning a \emph{good} kernel $K$ that is a positive combination of \emph{base} kernels $K_1,\ldots,K_p$ i.e. $K_{\vecmu}(\vecx,\vecx')=\sum_{i=1}^p\vecmu_iK_i(\vecx,\vecx')$ for some $\vecmu \in \R^p, \vecmu \geq 0$. The base kernels are bounded, i.e. for all $i$, $\abs{K_i(\vecx,\vecx')} \leq \kappa^2$ for all $\vecx,\vecx' \in \X$ The notion of goodness used here is the one proposed by \citet{good-sim} and involves using the loss function $\ell(\vecmu,\vecz,\vecz') = \phi\br{yy'K_{\vecmu}(\vecx,\vecx')}$ where $\phi(\cdot)$ is a margin loss function meant to encode some notion of alignment. The two hypothesis classes for the combination vector $\vecmu$ that we study are the $L_1$ regularized unit simplex $\Delta(1) = \bc{\vecmu : \norm{\vecmu}_1 = 1, \vecmu \geq 0}$ and the $L_2$ regularized unit sphere $\S_2(1) = \bc{\vecmu : \norm{\vecmu}_2 = 1, \vecmu \geq 0}$. We are able to show the following Rademacher complexity bounds for these classes: $\RR_n(\S_2(1)) \leq \kappa^2\sqrt{\frac{p}{n}}$ and $\RR_n(\Delta(1)) \leq \kappa^2\sqrt{\frac{e\log p}{n}}$.

The details of the Rademacher complexity derivations for these problems and other examples such as similarity learning can be found in Appendix~\ref{appsec:apps-supp}.

\section{\olp: Online Learning with Pairwise Loss Functions}
\label{sec:algo}


\begin{algorithm}[t]
	\caption{\small \mrs: Stream Subsampling with Replacement}
	\label{algo:rs-x}
	\begin{algorithmic}[1]
		\small{
			\REQUIRE Buffer $B$, new point $\vecz_t$, buffer size $s$, timestep $t$.
			\IF[There is space]{$|B| < s$}
				\STATE $B \leftarrow B \cup \bc{\vecz_t}$
			\ELSE[Overflow situation]
				\IF[Repopulation step]{$t = s+1$}
					\STATE $\text{TMP} \leftarrow B \cup \bc{\vecz_t}$
					\STATE Repopulate $B$ with $s$ points sampled uniformly with replacement from TMP.
				\ELSE[Normal update step]
					\STATE Independently, replace each point of $B$ with $\vecz_t$ with probability $1/t$.
				\ENDIF
			\ENDIF
		}
	\end{algorithmic}
\end{algorithm}

\begin{algorithm}[t]
	\caption{\small \olp: Online Learning with Pairwise Loss Functions}
	\label{algo:olp}
	\begin{algorithmic}[1]
		\small{
			\REQUIRE Step length scale $\eta$, Buffer size $s$
			\ENSURE An ensemble $\vecw_2,\ldots,\vecw_n \in \W$ with low regret
			\STATE $\vecw_0 \leftarrow \veczero$, $B \leftarrow \phi$
			\FOR{$t = 1$ \TO $n$}
				\STATE Obtain a training point $\vecz_t$
				\STATE Set step length $\eta_t \leftarrow \frac{\eta}{\sqrt t}$
				\STATE $\vecw_t \leftarrow \Pi_\W\bs{\vecw_{t-1} +  \frac{\eta_t}{\abs{B}}\sum_{\vecz \in B}\nabla_\vecw\ell(\vecw_{t-1},\vecz_t,\vecz)}$\newline
				\mbox{}\hfill\COMMENT{$\Pi_\W$ projects onto the set $\W$}
				\STATE $B \leftarrow \text{Update-buffer}(B,\vecz_t,s,t)$\COMMENT{using \mrs}
			\ENDFOR
			\RETURN{$\vecw_2,\ldots,\vecw_n$}
		}
	\end{algorithmic}
\end{algorithm}

In this section, we present an online learning algorithm for learning with pairwise loss functions in a finite buffer setting. The key contribution in this section is a buffer update policy that when combined with a variant of the GIGA algorithm \cite{zinkevich} allows us to give high probability regret bounds.



In previous work, \citet{oam-icml} presented an online learning algorithm that uses finite buffers with the \rsorig policy and proposed an \emph{all-pairs} regret bound. The \rsorig policy ensures, over the randomness used in buffer updates, that at any given time, the buffer contains a uniform sample from the preceding stream. Using this property, \citep[Lemma 2]{oam-icml} claimed that $\E{\Lbuft(h_{t-1})} = \hat\L_t(h_{t-1})$ where the expectation is taken over the randomness used in buffer construction. However, a property such as $\E{\Lbuft(h)} = \hat\L_t(h)$ holds only for functions $h$ that are either fixed or obtained independently of the random variables used in buffer updates (over which the expectation is taken). Since $h_{t-1}$ is learned from points in the buffer itself, the above property, and consequently the regret bound, does not hold.

We remedy this issue by showing a relatively weaker claim; we show that with high probability we have $\hat\L_t(h_{t-1}) \leq \Lbuft(h_{t-1}) + \epsilon$. At a high level, this claim is similar to showing uniform convergence bounds for $\Lbuft$. However, the reservoir sampling algorithm is not particularly well suited to prove such uniform convergence bounds as it essentially performs sampling without replacement (see Appendix~\ref{appsec:regret-rs} for a discussion). We overcome this hurdle by proposing a new buffer update policy \mrs (see Algorithm~\ref{algo:rs-x}) that, at each time step, guarantees $s$ i.i.d.~samples from the preceding stream (see Appendix~\ref{appsec:rs-x-analysis} for a proof).

Our algorithm uses this buffer update policy in conjunction with an online learning algorithm \olp (see Algorithm~\ref{algo:olp}) that is a variant of the well-known GIGA algorithm \cite{zinkevich}. We provide the following \emph{all-pairs} regret guarantee for our algorithm: 
\begin{thm}
\label{thm:rs-x-regret}
Suppose the \olp algorithm working with an $s$-sized buffer generates an ensemble $\vecw_1,\ldots,\vecw_{n-1}$. Then with probability at least $1 - \delta$,
\[
\frac{\Rk_n}{n-1} \leq \O{C_d\sqrt{\frac{\log \frac{n}{\delta}}{s}} + \sqrt{\frac{1}{n-1}}}
\]
\end{thm}
\vspace*{-1ex}
See Appendix~\ref{appsec:proof-thm-rs-x-regret} for the proof. A drawback of our bound is that it offers sublinear regret only for buffer sizes $s=\omega(\log n)$. A better regret bound for constant $s$ or a lower-bound on the regret is an open problem. 

\section{Experimental Evaluation}
\label{sec:exp}


\begin{figure}[t]
	\centering
	\hspace*{-7ex}
	\subfigure[Sonar\hspace*{-6ex}]{
		\includegraphics[width=0.54\linewidth]{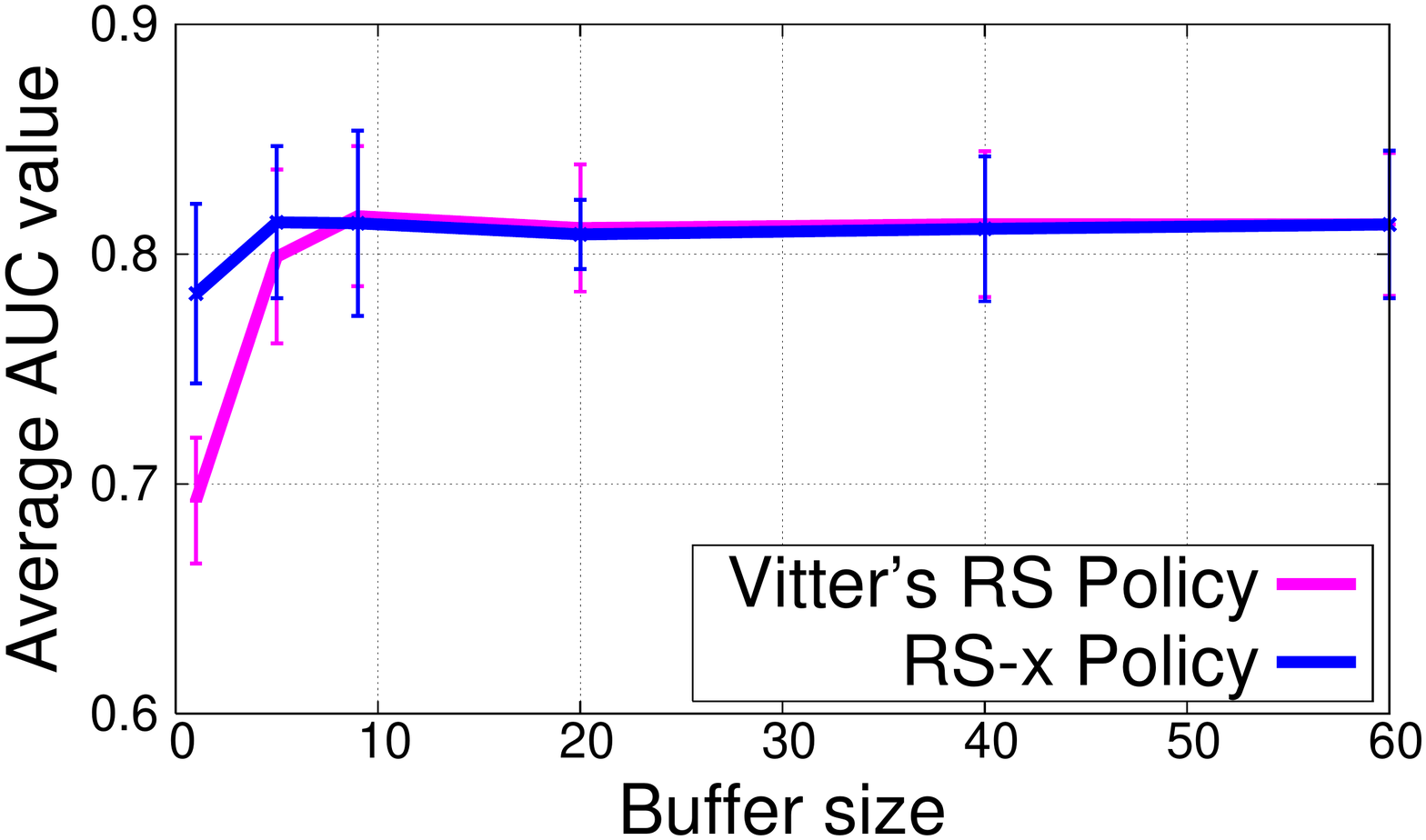}\hspace*{-6ex}
		\label{subfig:sonar-res}
	}
	\subfigure[Segment\hspace*{-6ex}]{
		\includegraphics[width=0.54\linewidth]{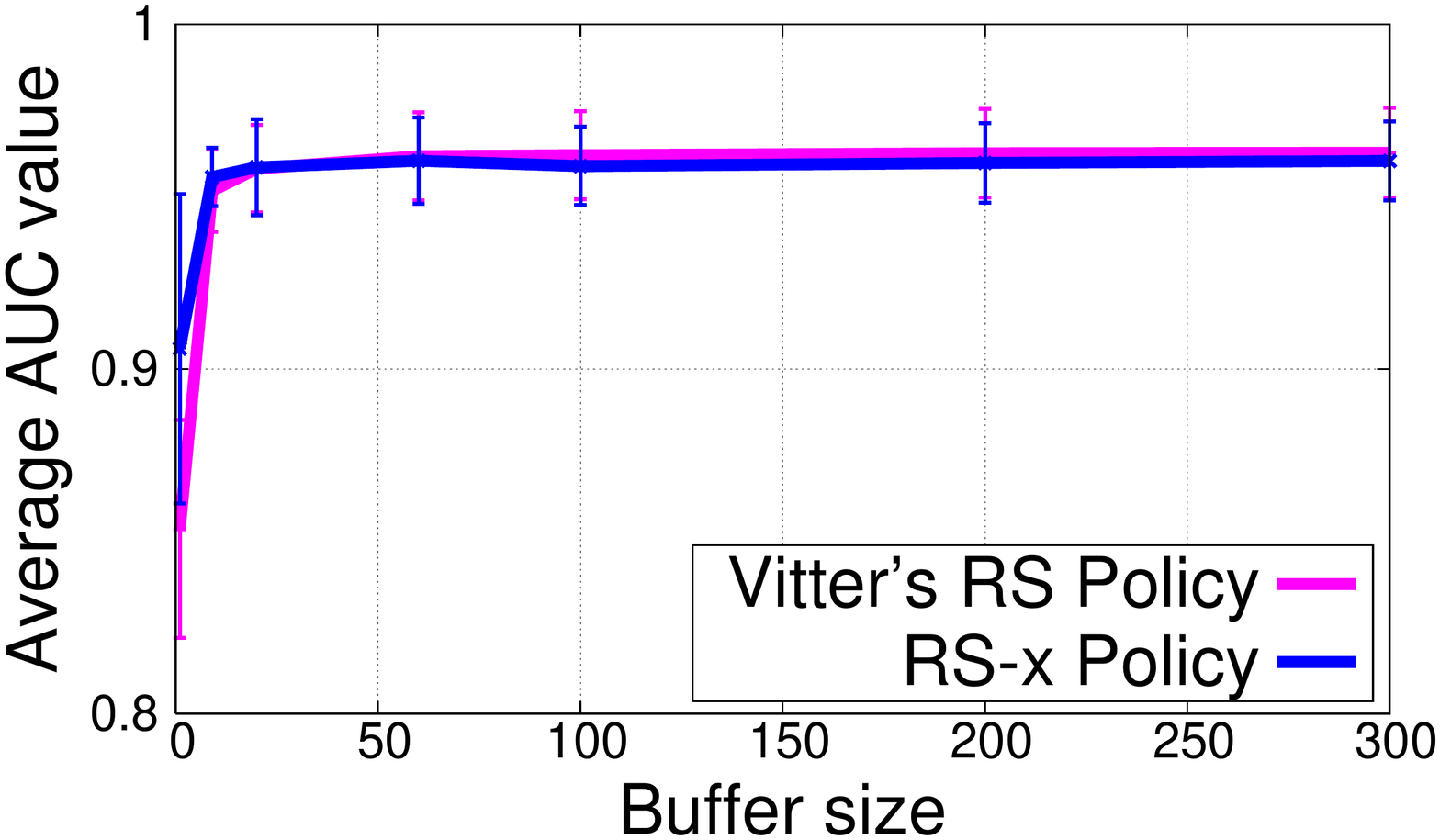}\hspace*{-6ex}
		\label{subfig:segment-res}
	}\\\vspace*{-6ex}\hspace*{-6ex}
	\subfigure[IJCNN\hspace*{-6ex}]{
		\includegraphics[width=0.54\linewidth]{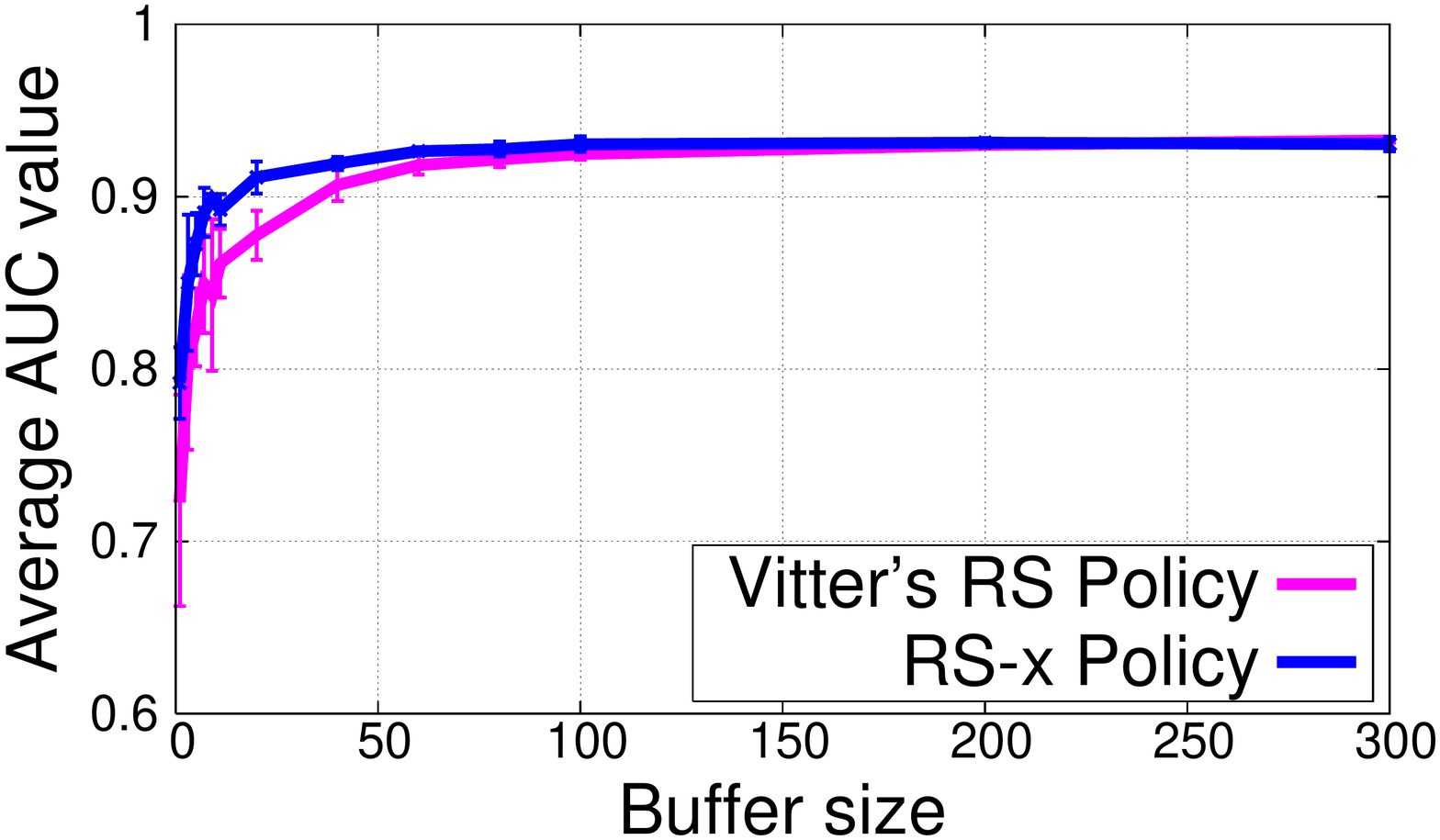}\hspace*{-6ex}
		\label{subfig:ijcnn-res}
	}
	\subfigure[Covertype\hspace*{-6ex}]{
		\includegraphics[width=0.54\linewidth]{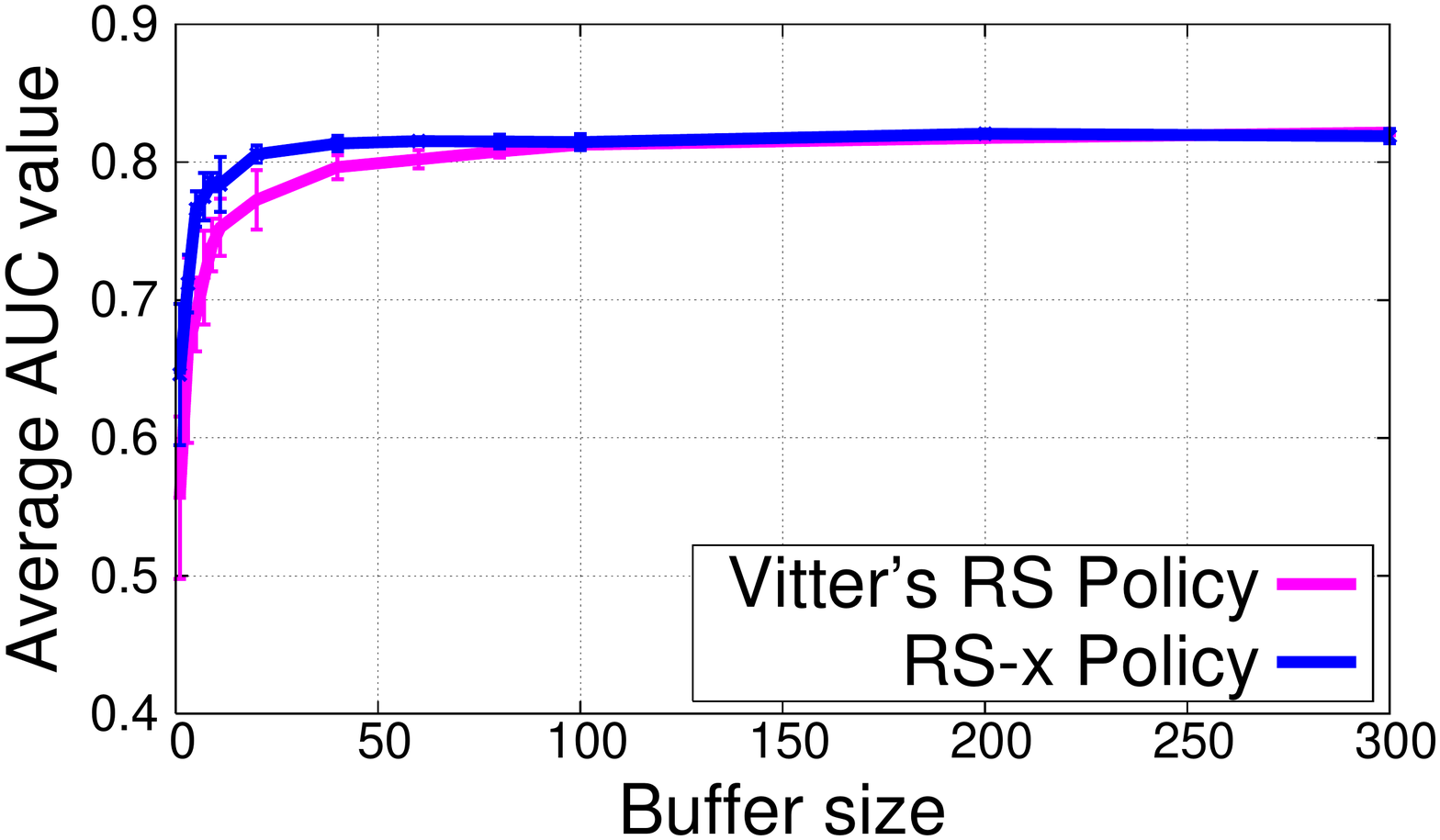}\hspace*{-6ex}
		\label{subfig:covtype-res}
	}
	\caption{Performance of \olp (using \mrs\!\!) and OAM$_{\text{gra}}$ (using \rsorig\!\!) by \cite{oam-icml} on AUC maximization tasks with varying buffer sizes.}
	\label{fig:auc-expts}
\end{figure}

In this section we present experimental evaluation of our proposed \olp algorithm. We stress that the aim of this evaluation is to show that our algorithm, that enjoys high confidence regret bounds, also performs competitively in practice with respect to the OAM$_{\text{gra}}$ algorithm proposed by \citet{oam-icml} since our results in Section~\ref{sec:finite-buffer} show that OAM$_{\text{gra}}$ does enjoy good generalization guarantees despite the lack of an \emph{all-pairs} regret bound.



In our experiments, we adapted the \olp algorithm to the AUC maximization problem and compared it with OAM$_{\text{gra}}$ on 18 different benchmark datasets. We used $60\%$ of the available data points up to a maximum of 20000 points to train both algorithms. We refer the reader to Appendix~\ref{appsec:rs-x-implementation} for a discussion on the implementation of the \mrs algorithm. Figure~\ref{fig:auc-expts} presents the results of our experiments on 4 datasets across 5 random training/test splits. Results on other datasets can be found in Appendix~\ref{appsec:expt-supp}. The results demonstrate that \olp performs competitively to OAM$_{\text{gra}}$ while in some cases having slightly better performance for small buffer sizes.

\vspace{-1.5ex}
\section{Conclusion}
In this paper we studied the generalization capabilities of online learning algorithms for pairwise loss functions from several different perspectives.
Using the method of \emph{Symmetrization of Expectations}, we first provided sharp online to batch conversion bounds for algorithms that offer \emph{all-pairs} regret bounds. Our results for bounded and strongly convex loss functions closely match their first order counterparts. We also extended our analysis to algorithms that are only able to provide \emph{finite-buffer} regret bounds using which we were able to explain the good empirical performance of some existing algorithms. Finally we presented a new memory-efficient online learning algorithm that is able to provide \emph{all-pairs} regret bounds in addition to performing well empirically.

Several interesting directions can be pursued for future work, foremost being the development of online learning algorithms that can guarantee sub-linear regret at constant buffer sizes or else a regret lower bound for finite buffer algorithms. Secondly, the idea of a \emph{stream-aware} buffer update policy is especially interesting both from an empirical as well as theoretical point of view and would possibly require novel proof techniques for its analysis. Lastly, scalability issues that arise when working with higher order loss functions also pose an interesting challenge.

%

\section*{Acknowledgment}
The authors thank the anonymous referees for comments that improved the presentation of the paper.
PK is supported by the Microsoft Corporation and Microsoft Research India under a Microsoft Research India Ph.D. fellowship award.


\bibliographystyle{icml2013}
\bibliography{ref}

\begin{thebibliography}{27}
\providecommand{\natexlab}[1]{#1}
\providecommand{\url}[1]{\texttt{#1}}
\expandafter\ifx\csname urlstyle\endcsname\relax
  \providecommand{\doi}[1]{doi: #1}\else
  \providecommand{\doi}{doi: \begingroup \urlstyle{rm}\Url}\fi

\bibitem[Agarwal \& Niyogi(2009)Agarwal and Niyogi]{AgarwalN09}
Agarwal, Shivani and Niyogi, Partha.
\newblock {Generalization Bounds for Ranking Algorithms via Algorithmic
  Stability}.
\newblock \emph{JMLR}, 10:\penalty0 441--474, 2009.

\bibitem[Balcan \& Blum(2006)Balcan and Blum]{good-sim}
Balcan, Maria-Florina and Blum, Avrim.
\newblock {On a Theory of Learning with Similarity Functions}.
\newblock In \emph{ICML}, pp.\  73--80, 2006.

\bibitem[Bellet et~al.(2012)Bellet, Habrard, and Sebban]{good-sim-learn}
Bellet, Aur\'elien, Habrard, Amaury, and Sebban, Marc.
\newblock {Similarity Learning for Provably Accurate Sparse Linear
  Classification}.
\newblock In \emph{ICML}, 2012.

\bibitem[Brefeld \& Scheffer(2005)Brefeld and Scheffer]{BrefeldS05}
Brefeld, Ulf and Scheffer, Tobias.
\newblock {AUC Maximizing Support Vector Learning}.
\newblock In \emph{ICML workshop on ROC Analysis in Machine Learning}, 2005.

\bibitem[Cao et~al.(2012)Cao, Guo, and Ying]{gen-bound-metric-learn}
Cao, Qiong, Guo, Zheng-Chu, and Ying, Yiming.
\newblock {Generalization Bounds for Metric and Similarity Learning}, 2012.
\newblock arXiv:1207.5437.

\bibitem[Cesa-Bianchi \& Gentile(2008)Cesa-Bianchi and
  Gentile]{online-batch-single-better}
Cesa-Bianchi, Nicol\'o and Gentile, Claudio.
\newblock {Improved Risk Tail Bounds for On-Line Algorithms}.
\newblock \emph{IEEE Trans. on Inf. Theory}, 54\penalty0 (1):\penalty0
  286--390, 2008.

\bibitem[Cesa-Bianchi et~al.(2001)Cesa-Bianchi, Conconi, and
  Gentile]{online-batch-single}
Cesa-Bianchi, Nicol\'o, Conconi, Alex, and Gentile, Claudio.
\newblock {On the Generalization Ability of On-Line Learning Algorithms}.
\newblock In \emph{NIPS}, pp.\  359--366, 2001.

\bibitem[Cl\'{e}men\c{c}on et~al.(2008)Cl\'{e}men\c{c}on, Lugosi, and
  Vayatis]{u-stat-rank}
Cl\'{e}men\c{c}on, St\'{e}phan, Lugosi, G\'{a}bor, and Vayatis, Nicolas.
\newblock {Ranking and empirical minimization of U-statistics}.
\newblock \emph{Annals of Statistics}, 36:\penalty0 844--874, 2008.

\bibitem[Cortes et~al.(2010{\natexlab{a}})Cortes, Mohri, and
  Rostamizadeh]{mkl-bound-cortes}
Cortes, Corinna, Mohri, Mehryar, and Rostamizadeh, Afshin.
\newblock {Generalization Bounds for Learning Kernels}.
\newblock In \emph{ICML}, pp.\  247--254, 2010{\natexlab{a}}.

\bibitem[Cortes et~al.(2010{\natexlab{b}})Cortes, Mohri, and
  Rostamizadeh]{two-stage-cortes}
Cortes, Corinna, Mohri, Mehryar, and Rostamizadeh, Afshin.
\newblock {Two-Stage Learning Kernel Algorithms}.
\newblock In \emph{ICML}, pp.\  239--246, 2010{\natexlab{b}}.

\bibitem[Cristianini et~al.(2001)Cristianini, Shawe-Taylor, Elisseeff, and
  Kandola]{kernel-target-alignment}
Cristianini, Nello, Shawe-Taylor, John, Elisseeff, Andr\'e, and Kandola, Jaz~S.
\newblock {On Kernel-Target Alignment}.
\newblock In \emph{NIPS}, pp.\  367--373, 2001.

\bibitem[Freedman(1975)]{freedman}
Freedman, David~A.
\newblock {On Tail Probabilities for Martingales}.
\newblock \emph{Annals of Probability}, 3\penalty0 (1):\penalty0 100--118,
  1975.

\bibitem[Hazan et~al.(2006)Hazan, Kalai, Kale, and Agarwal]{log-regret}
Hazan, Elad, Kalai, Adam, Kale, Satyen, and Agarwal, Amit.
\newblock {Logarithmic Regret Algorithms for Online Convex Optimization}.
\newblock In \emph{COLT}, pp.\  499--513, 2006.

\bibitem[Jin et~al.(2009)Jin, Wang, and Zhou]{reg-metric-learn}
Jin, Rong, Wang, Shijun, and Zhou, Yang.
\newblock {Regularized Distance Metric Learning: Theory and Algorithm}.
\newblock In \emph{NIPS}, pp.\  862--870, 2009.

\bibitem[Kakade \& Tewari(2008)Kakade and Tewari]{online-batch-strongly-convex}
Kakade, Sham~M. and Tewari, Ambuj.
\newblock {On the Generalization Ability of Online Strongly Convex Programming
  Algorithms}.
\newblock In \emph{NIPS}, pp.\  801--808, 2008.

\bibitem[Kakade et~al.(2008)Kakade, Sridharan, and Tewari]{rad-bounds}
Kakade, Sham~M., Sridharan, Karthik, and Tewari, Ambuj.
\newblock {On the Complexity of Linear Prediction: Risk Bounds, Margin Bounds,
  and Regularization}.
\newblock In \emph{NIPS}, 2008.

\bibitem[Kakade et~al.(2012)Kakade, Shalev-Shwartz, and
  Tewari]{ambuj-learn-matrices}
Kakade, Sham~M., Shalev-Shwartz, Shai, and Tewari, Ambuj.
\newblock {Regularization Techniques for Learning with Matrices}.
\newblock \emph{JMLR}, 13:\penalty0 1865--1890, 2012.

\bibitem[Kumar et~al.(2012)Kumar, Niculescu-Mizil, Kavukcuoglu, and
  III]{two-stage-mkl-general}
Kumar, Abhishek, Niculescu-Mizil, Alexandru, Kavukcuoglu, Koray, and III,
  Hal~Daum\'{e}.
\newblock {A Binary Classification Framework for Two-Stage Multiple Kernel
  Learning}.
\newblock In \emph{ICML}, 2012.

\bibitem[Ledoux \& Talagrand(2002)Ledoux and Talagrand]{talagrand-book}
Ledoux, Michel and Talagrand, Michel.
\newblock \emph{{Probability in Banach Spaces: Isoperimetry and Processes}}.
\newblock Springer, 2002.

\bibitem[Sridharan et~al.(2008)Sridharan, Shalev-Shwartz, and
  Srebro]{fast-rate-reg-obj}
Sridharan, Karthik, Shalev-Shwartz, Shai, and Srebro, Nathan.
\newblock {Fast Rates for Regularized Objectives}.
\newblock In \emph{NIPS}, pp.\  1545--1552, 2008.

\bibitem[Steinwart \& Christmann(2008)Steinwart and Christmann]{steinwart-book}
Steinwart, Ingo and Christmann, Andreas.
\newblock \emph{{Support Vector Machines}}.
\newblock Information Science and Statistics. Springer, 2008.

\bibitem[Vitter(1985)]{vitter-rs}
Vitter, Jeffrey~Scott.
\newblock {Random Sampling with a Reservoir}.
\newblock \emph{ACM Trans. on Math. Soft.}, 11\penalty0 (1):\penalty0 37--57,
  1985.

\bibitem[Wang et~al.(2012)Wang, Khardon, Pechyony, and
  Jones]{online-batch-pairwise}
Wang, Yuyang, Khardon, Roni, Pechyony, Dmitry, and Jones, Rosie.
\newblock {Generalization Bounds for Online Learning Algorithms with Pairwise
  Loss Functions}.
\newblock \emph{JMLR - Proceedings Track}, 23:\penalty0 13.1--13.22, 2012.

\bibitem[Wang et~al.(2013)Wang, Khardon, Pechyony, and
  Jones]{online-batch-pairwise-arxiv}
Wang, Yuyang, Khardon, Roni, Pechyony, Dmitry, and Jones, Rosie.
\newblock {Online Learning with Pairwise Loss Functions}, 2013.
\newblock arXiv:1301.5332.

\bibitem[Xing et~al.(2002)Xing, Ng, Jordan, and Russell]{XingNJR02}
Xing, Eric~P., Ng, Andrew~Y., Jordan, Michael~I., and Russell, Stuart~J.
\newblock {Distance Metric Learning with Application to Clustering with
  Side-Information}.
\newblock In \emph{NIPS}, pp.\  505--512, 2002.

\bibitem[Zhao et~al.(2011)Zhao, Hoi, Jin, and Yang]{oam-icml}
Zhao, Peilin, Hoi, Steven C.~H., Jin, Rong, and Yang, Tianbao.
\newblock {Online AUC Maximization}.
\newblock In \emph{ICML}, pp.\  233--240, 2011.

\bibitem[Zinkevich(2003)]{zinkevich}
Zinkevich, Martin.
\newblock {Online Convex Programming and Generalized Infinitesimal Gradient
  Ascent}.
\newblock In \emph{ICML}, pp.\  928--936, 2003.

\end{thebibliography}


\appendix

\section{Proof of Lemma~\ref{lem:bounded-forward}}
\label{appsec:proof-lem-bounded-forward}

\begin{lem}[Lemma~\ref{lem:bounded-forward} restated]
Let $h_1,\ldots,h_{n-1}$ be an ensemble of hypotheses generated by an online learning algorithm working with a bounded loss function $\ell : \H \times \ZZ \times \ZZ \rightarrow [0,B]$. Then for any $\delta > 0$, we have with probability at least $1 - \delta$,
\begin{eqnarray*}
\lefteqn{\frac{1}{n-1}\sum_{t=2}^n\L(h_{t-1}) \leq \frac{1}{n-1}\sum_{t=2}^n\hat\L_t(h_{t-1})}\\
									 	&& \mbox{} + \frac{2}{n-1}\sum_{t=2}^n\RR_{t-1}(\ell\circ\H) + 3B\sqrt\frac{\log\frac{n}{\delta}}{n-1}.
\end{eqnarray*}
\end{lem}
\begin{proof}
As a first step, we decompose the excess risk in a manner similar to \cite{online-batch-pairwise}. For any $h \in \H$ let
\[
\tilde\L_t(h) := \EE{\vecz_t}{\left.\hat\L_t(h)\right|Z^{t-1}}.
\]
This allows us to decompose the excess risk as follows: \vspace*{-8ex}

\begin{align*}
&\frac{1}{n-1}\sum_{t=2}^n\L(h_{t-1}) - \hat\L_t(h_{t-1})\\ &= \frac{1}{n-1}\br{\sum_{t=2}^n\underbrace{\L(h_{t-1}) - \tilde\L_t(h_{t-1})}_{P_t}
											+\underbrace{\tilde\L(h_{t-1}) - \hat\L_t(h_{t-1})}_{Q_t}}.
\end{align*}
By construction, we have $\EE{\vecz_t}{\left.Q_t\right|Z^{t-1}} = 0$ and hence the sequence $Q_2,\ldots,Q_n$ forms a martingale difference sequence. Since $\abs{Q_t} \leq B$ as the loss function is bounded, an application of the Azuma-Hoeffding inequality shows that with probability at least $1 - \delta$
\begin{eqnarray}
\frac{1}{n-1}\sum_{t=2}^n Q_t \leq B\sqrt\frac{2\log\frac{1}{\delta}}{n-1}.
\label{eq:Q-t-bound}
\end{eqnarray}
We now analyze each term $P_t$ individually. By linearity of expectation, we have for a ghost sample $\tilde Z^{t-1} = \bc{\tilde\vecz_1,\ldots,\tilde\vecz_{t-1}}$,
\begin{align}
\L(h_{t-1}) = \EE{\tilde Z^{t-1}}{\frac{1}{t-1}\sum_{\tau = 1}^{t-1}\EE{\vecz}{\ell(h_{t-1},\vecz,\tilde\vecz_\tau)}}.
\label{eq:second-to-first-order}
\end{align}
The expression of $\L(h_{t-1})$ as a nested expectation is the precursor to performing symmetrization with expectations and plays a crucial role in overcoming coupling problems. This allows us to write $P_t$ as
\begin{eqnarray*}
P_t &=& \EE{\tilde Z^{t-1}}{\frac{1}{t-1}\sum_{\tau = 1}^{t-1}\EE{\vecz}{\ell(h_{t-1},\vecz,\tilde\vecz_\tau)}} - \tilde\L_t(h_{t-1})\\
	&\leq& \underbrace{\underset{h \in \H}{\sup}\bs{\EE{\tilde Z^{t-1}}{\frac{1}{t-1}\sum_{\tau = 1}^{t-1}\EE{\vecz}{\ell(h,\vecz,\tilde\vecz_\tau)}} - \tilde\L_t(h)}}_{g_t(\vecz_1,\ldots,\vecz_{t-1})}.
\end{eqnarray*}
Since $\tilde\L_t(h) = \EE{\vecz}{\left.\frac{1}{t-1}\sum_{\tau = 1}^{t-1}\ell(h,\vecz,\vecz_\tau)\right|Z^{t-1}}$ and $\ell$ is bounded, the expression $g_t(\vecz_1,\ldots,\vecz_{t-1})$ can have a variation of at most $B/(t-1)$ when changing any of its $(t-1)$ variables. Hence an application of McDiarmid's inequality gives us, with probability at least $1 - \delta$,
\[
g_t(\vecz_1,\ldots,\vecz_{t-1}) \leq \EE{Z^{t-1}}{g_t(\vecz_1,\ldots,\vecz_{t-1})} + B\sqrt\frac{\log\frac{1}{\delta}}{2(t-1)}.
\]
For any $h \in \H, \vecz' \in \Z$, let $\wp(h,\vecz') := \frac{1}{t-1}\EE{\vecz}{\ell(h,\vecz,\vecz')}$. Then we can write $\EE{Z^{t-1}}{g(\vecz_1,\ldots,\vecz_{t-1})}$ as
\begin{align*}
& \EE{Z^{t-1}}{\underset{h \in \H}{\sup}\bs{\EE{\tilde Z^{t-1}}{\sum_{\tau = 1}^{t-1}\wp(h,\tilde\vecz_\tau)} - \sum_{\tau = 1}^{t-1}\wp(h,\vecz_\tau)}}\\
\leq {} & \EE{Z^{t-1},\tilde Z^{t-1}}{\underset{h \in \H}{\sup}\bs{\sum_{\tau = 1}^{t-1}\wp(h,\tilde\vecz_\tau) - \sum_{\tau = 1}^{t-1}\wp(h,\vecz_\tau)}}\\ 
   = {} & \EE{Z^{t-1},\tilde Z^{t-1},\{\epsilon_\tau\}}{\underset{h \in \H}{\sup}\bs{\sum_{\tau = 1}^{t-1}\epsilon_\tau\br{\wp(h,\tilde\vecz_\tau) - \wp(h,\vecz_\tau)}}}\\
\leq {} & \frac{2}{t-1}\EE{Z^{t-1},\{\epsilon_\tau\}}{\underset{h \in \H}{\sup}\bs{\sum_{\tau = 1}^{t-1}\epsilon_\tau\EE{\vecz}{\ell(h,\vecz,\vecz_\tau)}}}\\
\leq {} & \frac{2}{t-1}\EE{\vecz,Z^{t-1},\{\epsilon_\tau\}}{\underset{h \in \H}{\sup}\bs{\sum_{\tau = 1}^{t-1}\epsilon_\tau\ell(h,\vecz,\vecz_\tau)}}\\
 ={} & 2\RR_{t-1}(\ell\circ\H).
\end{align*}
Note that in the third step, the symmetrization was made possible by the decoupling step in Eq.~\eqref{eq:second-to-first-order} where we decoupled the ``head'' variable $\vecz_t$ from the ``tail'' variables by absorbing it inside an expectation. This allowed us to symmetrize the true and ghost samples $\vecz_\tau$ and $\tilde\vecz_\tau$ in a standard manner.
Thus we have, with probability at least $1 - \delta$,
\[
P_t \leq 2\RR_{t-1}(\ell\circ\H) + B\sqrt\frac{\log\frac{1}{\delta}}{2(t-1)}.
\]
Applying a union bound on the bounds for $P_t, t = 2,\ldots,n$ gives us with probability at least $1 - \delta$,
\begin{align}
\frac{1}{n-1}\sum_{t=2}^n P_t \leq \frac{2}{n-1}\sum_{t=2}^n\RR_{t-1}(\ell\circ\H) + B\sqrt\frac{2\log\frac{n}{\delta}}{n-1}.
\label{eq:P-t-bound}
\end{align}
Adding Equations \eqref{eq:Q-t-bound} and \eqref{eq:P-t-bound} gives us the result.
\end{proof}

\section{Proof of Theorem~\ref{thm:fast-batch}}
\label{appsec:proof-thm-fast-batch}

\begin{thm}[Theorem~\ref{thm:fast-batch} restated]
Let $\F$ be a closed and convex set of functions over $\X$. Let $\wp(f,\vecx) = p(\ip{f}{\phi(\vecx)}) + r(f)$, for a $\sigma$-strongly convex function $r$, be a loss function with $\P$ and $\hat\P$ as the associated population and empirical risk functionals and $f^\ast$ as the population risk minimizer. Suppose $\wp$ is $L$-Lipschitz and $\norm{\phi(\vecx)}_\ast \leq R, \forall \vecx \in \X$.
Then w.p. $1 - \delta$, for any $\epsilon > 0$, we have for all $f\in \F$,
\begin{align*}
\P(f) - \P(f^\ast) \leq {} & (1 + \epsilon)\br{\hat\P(f) - \hat\P(f^\ast)} + \frac{C_\delta}{\epsilon \sigma n}
\end{align*}
where $C_\delta = C_d^2\cdot (4(1 + \epsilon)LR)^2\br{32 + \log(1/\delta)}$ and $C_d$ is the dependence of the Rademacher complexity of the class $\F$ on the input dimensionality $d$.\end{thm}

\begin{proof}
We begin with a lemma implicit in the proof of Theorem 1 in \cite{fast-rate-reg-obj}. For the function class $\F$ and loss function $\wp$ as above, define a new loss function $\mu : (f,\vecx) \mapsto \wp(f,\vecx) - \wp(f^\ast,\vecx)$ with $\M$ and $\hat\M$ as the associated population and empirical risk functionals. Let $r = \frac{4L^2R^2C_d^2\br{32 + \log(1/\delta)}}{\sigma n}$. Then we have the following
\begin{lem}
\label{lem:fast_conv}
For any $\epsilon > 0$, with probability at least $1 - \delta$, the following happens
\begin{enumerate}
	\item For all $f \in \F$ such that $\M(f) \leq 16\br{1 + \frac{1}{\epsilon}}^2r$, we have $\M(f) \leq \hat\M(f) + 4\br{1 + \frac{1}{\epsilon}}r$.
	\item For all $f \in \F$ such that $\M(f) > 16\br{1 + \frac{1}{\epsilon}}^2r$, we have $\M(f) \leq (1 + \epsilon)\hat\M(f)$.
\end{enumerate}
\end{lem}
The difference in our proof technique lies in the way we combine these two cases. We do so by proving the following two simple results.
\begin{lem}
For all $f$ s.t. $\M(f) \leq 16\br{1 + \frac{1}{\epsilon}}^2r$, we have $\M(f) \leq (1 + \epsilon)\br{\hat\M(f) + 4\br{1 + \frac{1}{\epsilon}}r}$.
\end{lem}
\begin{proof}
We notice that for all $f \in \F$, we have $\M(f) = \P(f) - \P(f^\ast) \geq 0$. Thus, using Lemma~\ref{lem:fast_conv}, Part 1, we have $\hat\M(f) + 4\br{1 + \frac{1}{\epsilon}}r \geq \M(f) \geq 0$. Since for any $a, \epsilon > 0$, we have $a \leq (1 + \epsilon)a$, the result follows.
\end{proof}
\begin{lem}
For all $f$ s.t. $\M(f) > 16\br{1 + \frac{1}{\epsilon}}^2r$, we have $\M(f) \leq (1 + \epsilon)\br{\hat\M(f) + 4\br{1 + \frac{1}{\epsilon}}r}$.
\end{lem}
\begin{proof}
We use the fact that $r > 0$ and thus $4(1 + \epsilon)\br{1 + \frac{1}{\epsilon}}r > 0$ as well. The result then follows from an application of Part 2 of Lemma~\ref{lem:fast_conv}.
\end{proof}
From the definition of the loss function $\mu$, we have for any $f \in \F$, $\M(f) = \P(f) - \P(f^\ast)$ and $\hat\M(f) = \hat\P(f) - \hat\P(f^\ast)$. Combining the above lemmata with this observation completes the proof.
\end{proof}

\section{Proof of Theorem~\ref{thm:strongly-convex-forward}}
\label{appsec:proof-thm-strongly-convex-forward}
\begin{thm}[Theorem~\ref{thm:strongly-convex-forward} restated]
Let $h_1,\ldots,h_{n-1}$ be an ensemble of hypotheses generated by an online learning algorithm working with a $B$-bounded, $L$-Lipschitz and $\sigma$-strongly convex loss function $\ell$. 
 Further suppose the learning algorithm guarantees a regret bound of $\Rk_n$. Let $\Vk_n = \max\bc{{\Rk_n},2C_d^2\log n\log(n/\delta)}$ Then for any $\delta > 0$, we have with probability at least $1 - \delta$,
\begin{eqnarray*}
\lefteqn{\frac{1}{n-1}\sum_{t=2}^n\L(h_{t-1}) \leq \L(h^\ast) + \frac{\Rk_n}{n-1}}\\
											 &&\qquad\qquad + C_d\cdot\O{\frac{\sqrt{\Vk_n\log n\log(n/\delta)}}{n-1}},
\end{eqnarray*}
where the $\O{\cdot}$ notation hides constants dependent on domain size and the loss function such as $L, B$ and $\sigma$.
\end{thm}

\begin{proof}
The decomposition of the excess risk shall not be made explicitly in this case but shall emerge as a side-effect of the proof progression. Consider the loss function $\wp(h,\vecz') := \EE{\vecz}{\ell(h,\vecz,\vecz')}$ with $\P$ and $\hat\P$ as the associated population and empirical risk functionals. Clearly, if $\ell$ is $L$-Lipschitz and $\sigma$-strongly convex then so is $\wp$. As Equation~\eqref{eq:second-to-first-order} shows, for any $h \in \H$, $\P(h) = \L(h)$. Also it is easy to see that for any $Z^{t-1}$, $\hat\P(h) = \tilde\L_t(h)$. Applying Theorem~\ref{thm:fast-batch} on $h_{t-1}$ with the loss function $\wp$ gives us w.p. $1-\delta$,\vspace*{-4ex}

{\small
\setlength{\arraycolsep}{0.0em}
\begin{eqnarray*}
\L(h_{t-1}) - \L(h^\ast) &{}\leq{}& (1 + \epsilon)\br{\tilde\L_t(h_{t-1}) - \tilde\L_t(h^\ast)} + \frac{C_\delta}{\epsilon \sigma (t-1)}
\end{eqnarray*}
}
which, upon summing across time steps and taking a union bound, gives us with probability at least $1 - \delta$,\vspace*{-4ex}


{\small
\setlength{\arraycolsep}{0.0em}
\begin{eqnarray*}
\frac{1}{n-1}\sum_{t=2}^n\L(h_{t-1}) &{}\leq{}& \L(h^\ast) + \frac{C_{(\delta/n)}\log n}{\epsilon \sigma (n-1)}\\
											 &{}{}&\!\! +\frac{1 + \epsilon}{n-1}\sum_{t=2}^n\left(\tilde\L_t(h_{t-1}) - \tilde\L_t(h^\ast)\right).
\end{eqnarray*}
}
Let $\xi_t := \br{\tilde\L_t(h_{t-1}) - \tilde\L_t(h^\ast)} - \br{\hat\L_t(h_{t-1}) - \hat\L_t(h^\ast)}$. Then using the regret bound $\Rk_n$ we can write,
\vspace*{-1ex}
{\small
\begin{align*}
\frac{1}{n-1}\sum_{t=2}^n\L(h_{t-1}) \leq {} & \L(h^\ast) + \frac{1 + \epsilon}{n-1}\br{\Rk_n + \sum_{t=2}^n\xi_t}\\
											 & + \frac{C_{(\delta/n)}\log n}{\epsilon \sigma (n-1)}.
\end{align*}
}
%
%

We now use Bernstein type inequalities to bound the sum $\sum_{t=2}^n\xi_t$ using a proof technique used in \cite{online-batch-strongly-convex, online-batch-single-better}. We first note some properties of the sequence below.
\begin{lem}
The sequence $\xi_2,\ldots,\xi_n$ is a bounded martingale difference sequence with bounded conditional variance.
\end{lem}
\begin{proof}
That $\xi_t$ is a martingale difference sequence follows by construction: we can decompose the term $\xi_t = \phi_t - \psi_t$ where $\phi_t = \tilde\L_t(h_{t-1}) - \hat\L_t(h_{t-1})$ and $\psi_t = \tilde\L_t(h^\ast) - \hat\L_t(h^\ast)$, both of which are martingale difference sequences with respect to the common filtration $\F = \bc{\F_n : n = 0,1,\ldots}$ where $\F_n = \sigma\br{\vecz_i : i = 1,\ldots,n}$.

Since the loss function takes values in $\bs{0,B}$, we have $\abs{\xi_t} \leq 2B$ which proves that our sequence is bounded.

To prove variance bounds for the sequence, we first use the Lipschitz properties of the loss function to get
\begin{eqnarray*}
\xi_t &=& \br{\tilde\L_t(h_{t-1}) - \tilde\L_t(h^\ast)} - \br{\hat\L_t(h_{t-1}) - \hat\L_t(h^\ast)}\\
	  &\leq& 2L\norm{h_{t-1} - h^\ast}.
\end{eqnarray*}
Recall that the hypothesis space is embedded in a Banach space equipped with the norm $\norm{\cdot}$. Thus we have $\E{\left. \xi^2_t \right| Z^{t-1}} \leq 4L^2\norm{h_{t-1} - h^\ast}^2$. Now using $\sigma$-strong convexity of the loss function we have
\setlength{\arraycolsep}{0.0em}
\begin{eqnarray*}
	\frac{\L(h_{t-1}) + \L(h^\ast)}{2} &{}\geq{}& \L\br{\frac{h_{t-1} + h^\ast}{2}} + \frac{\sigma}{8}\norm{h_{t-1} - h^\ast}^2\\
												 &{}\geq{}& \L(h^\ast) + \frac{\sigma}{8}\norm{h_{t-1} - h^\ast}^2.	
\end{eqnarray*}
Let $\sigma^2_t := \frac{16L^2}{\sigma}\br{\L(h_{t-1}) - \L(h^\ast)}$. Combining the two inequalities we get $\E{\left. \xi^2_t \right| Z^{t-1}} \leq \sigma^2_t$.\hfill\qedhere
\end{proof}

We note that although \cite{online-batch-strongly-convex} state their result with a requirement that the loss function be strongly convex in a point wise manner, i.e., for all $\vecz,\vecz' \in \ZZ$, the function $\ell(h,\vecz,\vecz')$ be strongly convex in $h$, they only require the result in expectation. More specifically, our notion of strong convexity where we require the population risk functional $\L(h)$ to be strongly convex actually suits the proof of \cite{online-batch-strongly-convex} as well.

We now use a Bernstein type inequality for martingales proved in \cite{online-batch-strongly-convex}. The proof is based on a fundamental result on martingale convergence due to \citet{freedman}.
\begin{thm}
Given a martingale difference sequence $X_t, t = 1 \ldots n$ that is uniformly $B$-bounded and has conditional variance $\E{X^2_t|X_1,\ldots,X_{t-1}} \leq \sigma^2_t$, we have for any $\delta < 1/e$ and $n \geq 3$, with probability at least $1 - \delta$,
\[
\sum_{t=1}^n X_t \leq \max\bc{2\sigma^\ast,3B\sqrt{\log\frac{4\log n}{\delta}}}\sqrt{\log\frac{4\log n}{\delta}},
\]
where $\sigma^\ast = \sqrt{\sum_{t=1}^n\sigma^2_t}$.
\end{thm}
Let $\Dk_n = \sum_{t=2}^n\br{\L(h_{t-1}) - \L(h^\ast)}$. Then we can write the variance bound as
\setlength{\arraycolsep}{0.0em}
\begin{eqnarray*}
\sigma^\ast &{}={}& \sqrt{\sum_{t=1}^n\sigma^2_t} = \sqrt{\sum_{t=1}^n\frac{16L^2}{\sigma}\br{\L(h_{t-1}) - \L(h^\ast)}}\\
	   &{}={}& 4L\sqrt{\frac{\mathfrak D_n}{\sigma}}.
\end{eqnarray*}

Thus, with probability at least $1 - \delta$, we have
\[
\sum_{t=1}^n \xi_t \leq \max\bc{8L\sqrt{\frac{\mathfrak D_n}{\sigma}},6B\sqrt{\log \frac{4\log n}{\delta}}}\sqrt{\log \frac{4\log n}{\delta}}.
\]
Denoting $\Delta = \sqrt{\log \frac{4\log n}{\delta}}$ for notational simplicity and using the above bound in the online to batch conversion bound gives us
\begin{align*}
\frac{\Dk_n}{n-1} \leq {} & \frac{1 + \epsilon}{n-1}\br{\Rk_n + \max\bc{8L\sqrt{\frac{\mathfrak D_n}{\sigma}},6B\Delta}\Delta} \\
											 & + \frac{C_{(\delta/n)}\log n}{\epsilon \sigma (n-1)}.
\end{align*}
Solving this quadratic inequality is simplified by a useful result given in \citep[Lemma 4]{online-batch-strongly-convex}
\begin{lem}
For any $s,r,d,b,\Delta > 0$ such that
\[
s \leq r + \max\bc{4\sqrt{ds},6b\Delta}\Delta,
\]
we also have
\[
s \leq r + 4\sqrt{dr}\Delta + \max\bc{16d,6b}\Delta^2.
\]
\end{lem}
Using this result gives us a rather nasty looking expression which we simplify by absorbing constants inside the $\O{\cdot}$ notation. We also make a simplifying ad-hoc assumption that we shall only set $\epsilon \in (0,1]$. The resulting expression is given below:
\begin{align*}
\Dk_n \leq {} & \br{1 + \epsilon}\Rk_n + \O{\frac{C_d^2\log n\log(n/\delta)}{\epsilon} + \log\frac{\log n}{\delta}}\\
											 & + \O{\sqrt{\br{\Rk_n + \frac{C_d^2\log n\log(n/\delta)}{\epsilon}}\log\frac{\log n}{\delta}}}.
\end{align*}
Let ${\Vk_n} = \max\bc{\Rk_n,2C_d^2\log n\log\br{n/\delta}}$. Concentrating only on the portion of the expression involving $\epsilon$ and ignoring the constants, we get
\begin{align*}
& \epsilon\Rk_n + \frac{C_d^2\log n\log(n/\delta)}{\epsilon} + \sqrt{\frac{C_d^2\log n\log(n/\delta)}{\epsilon}\log\frac{\log n}{\delta}}\\
\leq {} & \epsilon\Rk_n + \frac{2C_d^2\log n\log(n/\delta)}{\epsilon} \leq \epsilon\Vk_n + \frac{2C_d^2\log n\log(n/\delta)}{\epsilon}\\
\leq {} & 2C_d\sqrt{2\Vk_n\log n\log(n/\delta)},
\end{align*}
where the second step follows since $\epsilon \leq 1$ and the fourth step follows by using $\epsilon = \sqrt\frac{2C_d^2\log n\log\br{n/\delta}}{\Vk_n} \leq 1$. Putting this into the excess risk expression gives us
\begin{align*}
\frac{1}{n-1}\sum_{t=2}^n\L(h_{t-1}) \leq {} & \L(h^\ast) + \frac{\Rk_n}{n-1}\\
											 & + C_d\cdot\O{\frac{\sqrt{\Vk_n\log n\log(n/\delta)}}{n-1}}
\end{align*}
which finishes the proof.
\end{proof}

\section{Generalization Bounds for Finite Buffer Algorithms}
\label{appsec:proof-finite-buffer-thms}

In this section we present online to batch conversion bounds for learning algorithms that work with finite buffers and are able to provide regret bounds $\Rkbufn$ with respect to \emph{finite-buffer} loss functions $\Lbuft$.

Although due to lack of space, Theorem~\ref{thm:finite-buffer-online-to-batch} presents these bounds for bounded as well as strongly convex functions together, we prove them separately for sake of clarity. Moreover, the techniques used to prove these two results are fairly different which further motivates this. Before we begin, we present the problem setup formally and introduce necessary notation.

In our finite buffer online learning model, one observes a stream of elements $\vecz_1,\ldots,\vecz_n$. A \emph{sketch} of these elements is maintained in a buffer $B$ of size $s$, i.e., at each step $t = 2, \ldots, n$, the buffer contains a subset of the elements $Z^{t-1}$ of size at most $s$. At each step $t = 2 \ldots n$, the online learning algorithm posits a hypothesis $h_{t-1} \in \H$, upon which the element $\vecz_t$ is revealed and the algorithm incurs the loss
\[
\Lbuft(h_{t-1}) = \frac{1}{\abs{B_t}}\sum_{\vecz \in B_t}\ell(h_{t-1},\vecz_t,\vecz),
\]
where $B_t$ is the state of the buffer at time $t$. Note that $\abs{B_t} \leq s$. We would be interested in algorithms that are able to give a \emph{finite-buffer} regret bound, i.e., for which, the proposed ensemble $h_1,\ldots,h_{n-1}$ satisfies
\[
\sum_{t=2}^n\Lbuft(h_{t-1}) - \underset{h \in \H}{\inf}\sum_{t=2}^n\Lbuft(h) \leq \Rkbufn.
\]

We assume that the buffer is updated after each step in a \emph{stream-oblivious} manner. For randomized buffer update policies (such as reservoir sampling \cite{vitter-rs}), we assume that we are supplied at each step with some fresh randomness $r_t$ (see examples below) along with the data point $\vecz_t$. Thus the data received at time $t$ is a tuple $\vecw_t = (\vecz_t,r_t)$. We shall refer to the random variables $r_t$ as \textit{auxiliary} variables. It is important to note that stream obliviousness dictates that $r_t$ as a random variable is independent of $z_t$. Let $W^{t-1} := \bc{\vecw_1,\ldots,\vecw_{t-1}}$ and $R^{t-1} := \bc{r_1,\ldots,r_{t-1}}$. Note that $R^{t-1}$ completely decides the indices present in the buffer $B_t$ at step $t$ independent of $Z^{t-1}$. For any $h \in \H$, define
\[
\Ltbuft := \EE{\vecz_t}{\left.\Lbuft\right|W^{t-1}}.
\]

\subsection{Examples of Stream Oblivious Policies}
Below we give some examples of stream oblivious policies for updating the buffer:
\begin{enumerate}
	\item \textbf{FIFO}: in this policy, the data point $\vecz_t$ arriving at time $t > s$ is inducted into the buffer by evicting the data point $\vecz_{(t-s)}$ from the buffer. Since this is a non-randomized policy, there is no need for auxiliary randomness and we can assume that $r_t$ follows the trivial law $r_t \sim \ind_{\bc{r = 1}}$.
	\item \rsorig: the Reservoir Sampling policy was introduced by \citet{vitter-rs}. In this policy, at time $t > s$, the incoming data point $\vecz_t$ is inducted into the buffer with probability $s/t$. If chosen to be induced, it results in the eviction of a random element of the buffer. In this case the auxiliary random variable is 2-tuple that follows the law
	\[
		r_t = (r_t^1,r_t^2) \sim \br{\text{Bernoulli}\br{\frac{s}{t}}, \frac{1}{s}\sum\limits_{i=1}^s \ind_{\bc{r_2 = i}}}.
	\]
	\item \mrs (see Algorithm~\ref{algo:rs-x}): in this policy, the incoming data point $\vecz_t$ at time $t > s$, replaces each data point in the buffer independently with probability $1/t$. Thus the incoming point has the potential to evict multiple buffer points while establishing multiple copies of itself in the buffer. In this case, the auxiliary random variable is defined by a Bernoulli process: $r_t = (r_t^1, r_t^2 \ldots, r_t^s) \sim \br{\text{Bernoulli}\br{\frac{1}{t}},\text{Bernoulli}\br{\frac{1}{t}}, \ldots, \text{Bernoulli}\br{\frac{1}{t}}}$.
	\item \mmrs (see Algorithm~\ref{algo:rs-x-alternate}): this is a variant of \mrs in which the number of evictions is first decided by a Binomial trial and then those many random points in the buffer are replaced by the incoming data point. This can be implemented as follows: $r_t = (r_t^1, r_t^2) \sim \br{\text{Binomial}\br{s,\frac{1}{t}}, \text{Perm}(s)}$ where $\text{Perm}(s)$ gives a random permutation of $s$ elements.
\end{enumerate}

\subsection{Finite Buffer Algorithms with Bounded Loss Functions}
We shall prove the result in two steps. In the first step we shall prove the following uniform convergence style result
\begin{lem}
\label{lem:bounded-forward-buffer}
Let $h_1,\ldots,h_{n-1}$ be an ensemble of hypotheses generated by an online learning algorithm working with a $B$-bounded loss function $\ell$ and a finite buffer of capacity $s$. Then for any $\delta > 0$, we have with probability at least $1 - \delta$,
\begin{align*}
\frac{1}{n-1}\sum_{t=2}^n\L(h_{t-1}) \leq {} & \frac{1}{n-1}\sum_{t=2}^n\Lbuft(h_{t-1}) + B\sqrt\frac{2\log\frac{n}{\delta}}{s}\\
									 	+ {} & \frac{2}{n-1}\sum_{t=2}^{n}\RR_{\min\bc{t-1,s}}(\ell\circ\H).
\end{align*}
\end{lem}
At a high level, our proof progression shall follow that of Lemma~\ref{lem:bounded-forward}. However, the execution of the proof will have to be different in order to accommodate the finiteness of the buffer and randomness used to construct it. Similarly, we shall also be able to show the following result.
\begin{lem}
\label{lem:bounded-converse-buffer}
For any $\delta > 0$, we have with probability at least $1 - \delta$,
\begin{align*}
\frac{1}{n-1}\sum_{t=2}^n\Lbuft(h^\ast) \leq {} & \L(h^\ast) + 3B\sqrt\frac{\log\frac{n}{\delta}}{s}\\
										   + {} & \frac{2}{n-1}\sum_{t=2}^n\RR_{\min\bc{t-1,s}}(\ell\circ\H).
\end{align*}
\end{lem}

Note that for classes whose Rademacher averages behave as $\RR_n(\H) \leq C_d\cdot\O{\frac{1}{\sqrt n}}$, applying Lemma~\ref{thm:contraction-rademacher} gives us $\RR_n(\ell\circ\H) \leq C_d\cdot\O{\frac{1}{\sqrt n}}$ as well which allows us to show
\[
\frac{2}{n-1}\sum_{t=2}^n\RR_{\min\bc{t-1,s}}(\ell\circ\H) = C_d\cdot\O{\frac{1}{\sqrt s}}.
\]

Combining Lemmata~\ref{lem:bounded-forward-buffer} and \ref{lem:bounded-converse-buffer} along with the definition of bounded buffer regret $\Rkbufn$ gives us the first part of Theorem~\ref{thm:finite-buffer-online-to-batch}. We prove Lemma~\ref{lem:bounded-forward-buffer} below:

\begin{proof}[Proof (of Lemma~\ref{lem:bounded-forward-buffer}).]
We first decompose the excess risk term as before\vspace*{-5ex}

{\small
\begin{align*}
&\sum_{t=2}^n\L(h_{t-1}) - \Lbuft(h_{t-1})\\ &= \sum_{t=2}^n\underbrace{\L(h_{t-1}) - \Ltbuft(h_{t-1})}_{P_t}
											+\underbrace{\Ltbuft(h_{t-1}) - \Lbuft(h_{t-1})}_{Q_t}.
\end{align*}
}
By construction, the sequence $Q_t$ forms a martingale difference sequence, i.e., $\EE{\vecz_t}{\left.Q_t\right|Z^{t-1}} = 0$ and hence by an application of Azuma Hoeffding inequality we have
\begin{eqnarray}
\frac{1}{n-1}\sum_{t=2}^n Q_t \leq B\sqrt\frac{2\log\frac{1}{\delta}}{n-1}.
\label{eq:Q-t-bound-buf}
\end{eqnarray}
We now analyze each term $P_t$ individually. To simplify the analysis a bit we assume that the buffer update policy keeps admitting points into the buffer as long as there is space so that for $t \leq s+1$, the buffer contains an exact copy of the preceding stream. This is a very natural assumption satisfied by FIFO as well as reservoir sampling. We stress that our analysis works even without this assumption but requires a bit more work. In case we do make this assumption, the analysis of Lemma~\ref{lem:bounded-forward} applies directly and we have, for any $t \leq s + 1$, with probability at least $1 - \delta$,
\[
P_t \leq \RR_{t-1}(\ell\circ\H) + B\sqrt\frac{\log\frac{1}{\delta}}{2(t-1)}
\]
For $t > s +1$, for an independent ghost sample $\bc{\tilde\vecw_1,\ldots,\tilde\vecw_{t-1}}$ we have,
\begin{align*}
  & \EE{\tilde W^{t-1}}{\Ltbuft} = \EE{\tilde W^{t-1}}{\frac{1}{s}\sum_{\tilde\vecz \in \tilde B_t}\EE{\vecz}{\ell(h_{t-1},\vecz,\tilde\vecz)}}\\
= & \EE{\tilde R^{t-1}}{\EE{\tilde Z^{t-1}}{\left.\frac{1}{s}\sum_{\tilde\vecz \in \tilde B_t}\EE{\vecz}{\ell(h_{t-1},\vecz,\tilde\vecz)}\right|\tilde R^{t-1}}}.
\end{align*}
The conditioning performed above is made possible by stream obliviousness. Now suppose that given $\tilde R^{t-1}$ the indices $\tilde\tau_1,\ldots,\tilde\tau_s$ are present in the buffer $\tilde B_t$ at time $t$. Recall that this choice of indices is independent of $\tilde Z^{t-1}$ because of stream obliviousness. Then we can write the above as
\begin{align*}
  & \EE{\tilde R^{t-1}}{\EE{\tilde Z^{t-1}}{\left.\frac{1}{s}\sum_{\tilde\vecz \in \tilde B_t}\EE{\vecz}{\ell(h_{t-1},\vecz,\tilde\vecz)}\right|\tilde R^{t-1}}}\\
= & \EE{\tilde R^{t-1}}{\EE{\tilde Z^{t-1}}{\frac{1}{s}\sum_{j=1}^s\EE{\vecz}{\ell(h_{t-1},\vecz,\tilde\vecz_{\tilde\tau_j})}}}\\
= & \EE{\tilde R^{t-1}}{\EE{\tilde\vecz_1,\ldots,\tilde\vecz_s}{\frac{1}{s}\sum_{j=1}^s\EE{\vecz}{\ell(h_{t-1},\vecz,\tilde\vecz_j)}}}\\
= & \EE{\tilde R^{t-1}}{\L(h_{t-1})} = \L(h_{t-1}).
\end{align*}
We thus have
\begin{eqnarray}
\EE{\tilde W^{t-1}}{\frac{1}{s}\sum_{\tilde\vecz \in \tilde B_t}\EE{\vecz}{\ell(h_{t-1},\vecz,\tilde\vecz)}} = \L(h_{t-1}).
\label{eq:second-to-first-order-buf}
\end{eqnarray}

We now upper bound $P_t$ as
\begin{eqnarray*}
P_t &=& \L(h_{t-1}) - \Ltbuft(h_{t-1})\\
	&=& \EE{\tilde W^{t-1}}{\frac{1}{s}\sum_{\tilde\vecz \in \tilde B_t}\EE{\vecz}{\ell(h_{t-1},\vecz,\tilde\vecz)}} - \Ltbuft(h_{t-1})\\
	&\leq& \underbrace{\underset{h \in \H}{\sup}\bs{\EE{\tilde W^{t-1}}{\frac{1}{s}\sum_{\tilde\vecz \in \tilde B_t}\EE{\vecz}{\ell(h,\vecz,\tilde\vecz)}} - \Ltbuft(h)}}_{g_t(\vecw_1,\ldots,\vecw_{t-1})}.
\end{eqnarray*}
Now it turns out that applying McDiarmid's inequality to $g_t(\vecw_1,\ldots,\vecw_{t-1})$ directly would yield a very loose bound. This is because of the following reason: since $\Lbuft(h) = \frac{1}{\abs{B_t}}\sum_{\vecz \in B_t}\ell(h,\vecz_t,\vecz)$ depends only on $s$ data points, changing any one of the $(t-1)$ variables $\vecw_i$ brings about a perturbation in $g_t$ of magnitude at most $\O{1/s}$. The problem is that $g_t$ is a function of $(t-1) \gg s$ variables and hence a direct application of McDiarmid's inequality would yield an excess error term of $\sqrt{\frac{t\log(1/\delta)}{s^2}}$ which would in the end require $s = \omega(\sqrt n)$ to give any non trivial generalization bounds. In contrast, we wish to give results that would give non trivial bounds for $s = \tilde\omega(1)$.

In order to get around this problem, we need to reduce the number of variables in the statistic while applying McDiarmid's inequality. Fortunately, we observe that $g_t$ \emph{effectively} depends only on $s$ variables, the data points that end up in the buffer at time $t$. This allows us to do the following. For any $R^{t-1}$, define
\[
\delta(R^{t-1}) := \Prr{Z^{t-1}}{\left.g_t(\vecw_1,\ldots,\vecw_{t-1}) > \epsilon \right|R^{t-1}}.
\]
We will first bound $\delta(R^{t-1})$. This will allow us to show
\[
\Prr{W^{t-1}}{g_t(\vecw_1,\ldots,\vecw_{t-1}) > \epsilon} \leq \EE{R^{t-1}}{\delta(R^{t-1})},
\]
where we take expectation over the distribution on $R^{t-1}$ induced by the buffer update policy. Note that since we are oblivious to the nature of the distribution over $R^{t-1}$, our proof works for any stream oblivious buffer update policy. Suppose that given $R^{t-1}$ the indices $\tau_1,\ldots,\tau_s$ are present in the buffer $B_t$ at time $t$. Then we have
{\small
\begin{align*}
  &\ g_t(\vecw_1,\ldots,\vecw_{t-1}; R^{t-1})\\
=&\ \underset{h \in \H}{\sup}\bs{\EE{\tilde W^{t-1}}{\frac{1}{s}\sum_{\tilde\vecz \in \tilde B_t}\EE{\vecz}{\ell(h,\vecz,\tilde\vecz)}} - \frac{1}{s}\sum_{j=1}^s\EE{\vecz}{\ell(h,\vecz,\vecz_{\tau_j})}}\\
=:&\ \tilde g_t(\vecz_{\tau_1},\ldots,\vecz_{\tau_s}).
\end{align*}
}
The function $\tilde g_t$ can be perturbed at most $B/s$ due to a change in one of $\vecz_{\tau_j}$. Applying McDiarmid's inequality to the function $\tilde g_t$ we get with probability at least $1-\delta$,
\[
\tilde g_t(\vecz_{\tau_1},\ldots,\vecz_{\tau_s}) \leq \EE{Z^{t-1}}{\tilde g_t(\vecz_{\tau_1},\ldots,\vecz_{\tau_s})} + B\sqrt{\frac{\log\frac{1}{\delta}}{2s}}
\]
We analyze $\EE{Z^{t-1}}{\tilde g_t(\vecz_{\tau_1},\ldots,\vecz_{\tau_s})}$ in Figure~\ref{fig:long-cal-simple-buffer}. In the third step in the calculations we symmetrize the true random variable $\vecz_{\tau_j}$ with the ghost random variable $\tilde\vecz_{\tilde\tau_j}$. This is contrasted with traditional symmetrization where we would symmetrize $\vecz_i$ with $\tilde\vecz_i$. In our case, we let the buffer construction dictate the matching at the symmetrization step.
\begin{figure*}
\small
\setlength{\arraycolsep}{0.0em}
\begin{eqnarray*}
\EE{Z^{t-1}}{\tilde g_t(\vecz_{\tau_1},\ldots,\vecz_{\tau_s})}
&{}={}& \EE{Z^{t-1}}{\underset{h \in \H}{\sup}\bs{\EE{\tilde W^{t-1}}{\frac{1}{s}\sum_{\tilde\vecz \in \tilde B_t}\EE{\vecz}{\ell(h,\vecz,\tilde\vecz)}} - \frac{1}{s}\sum_{j=1}^s\EE{\vecz}{\ell(h,\vecz,\vecz_{\tau_j})}}}\\
&{}\leq {} & \EE{\tilde R^{t-1}}{\left.\EE{Z^{t-1},\tilde Z^{t-1}}{\underset{h \in \H}{\sup}\bs{\frac{1}{s}\sum_{j= 1}^s\EE{\vecz}{\ell(h,\vecz,\tilde\vecz_{\tilde\tau_j})} - \frac{1}{s}\sum_{j=1}^s\EE{\vecz}{\ell(h,\vecz,\vecz_{\tau_j})}}}\right|\tilde R^{t-1}}\\
&{}= {} & \EE{\tilde R^{t-1}}{\left.\EE{Z^{t-1},\tilde Z^{t-1},\epsilon_j}{\underset{h \in \H}{\sup}\bs{\frac{1}{s}\sum_{j= 1}^s \epsilon_j\br{\EE{\vecz}{\ell(h,\vecz,\tilde\vecz_{\tilde\tau_j})} - \EE{\vecz}{\ell(h,\vecz,\vecz_{\tau_j})}}}}\right|\tilde R^{t-1}}\\
&{}\leq {} & 2\EE{\tilde R^{t-1}}{\left.\EE{Z^{t-1},\epsilon_j}{\underset{h \in \H}{\sup}\bs{\frac{1}{s}\sum_{j= 1}^s \epsilon_j\EE{\vecz}{\ell(h,\vecz,\vecz_{\tau_j})}}}\right|\tilde R^{t-1}}\\
&{}\leq {} & 2\EE{\tilde R^{t-1}}{\RR_s(\ell\circ\H)} \leq 2\RR_s(\ell\circ\H).
\end{eqnarray*}
\vspace*{-3ex}
\caption{Decoupling training and auxiliary variables for Rademacher complexity-based analysis.}
\label{fig:long-cal-simple-buffer}
\end{figure*}
Thus we get, with probability at least $1 - \delta$ over $\vecz_1,\ldots,\vecz_{t-1}$,
\[
g_t(\vecw_1,\ldots,\vecw_{t-1}; R^{t-1}) \leq 2\RR_s(\ell\circ\H) + B\sqrt{\frac{\log\frac{1}{\delta}}{2s}}
\]
which in turn, upon taking expectations with respect to $R^{t-1}$, gives us with probability at least $1 - \delta$ over $\vecw_1,\ldots,\vecw_{t-1}$,
\[
P_t = g_t(\vecw_1,\ldots,\vecw_{t-1}) \leq 2\RR_s(\ell\circ\H) + B\sqrt{\frac{\log\frac{1}{\delta}}{2s}}.
\]
Applying a union bound on the bounds for $P_t, t = 2,\ldots,n$ gives us with probability at least $1 - \delta$,
\setlength{\arraycolsep}{0.0em}
\begin{eqnarray}
\frac{1}{n-1}\sum_{t=2}^n P_t &{}\leq{}& \frac{2}{n-1}\sum_{t=2}^n\RR_{\min\bc{t-1,s}}(\ell\circ\H)\nonumber\\
&{}{}&\qquad + B\sqrt\frac{\log\frac{n}{\delta}}{2s}.
\label{eq:P-t-bound-buf}
\end{eqnarray}
Adding Equations \eqref{eq:Q-t-bound-buf} and \eqref{eq:P-t-bound-buf} gives us the result.
\end{proof}

\subsection{Finite Buffer Algorithms with Strongly Convex Loss Functions}
In this section we prove faster convergence bounds for algorithms that offer \emph{finite-buffer} regret bounds and use strongly convex loss functions. Given the development of the method of decoupling training and auxiliary random variables in the last section, we can proceed with the proof right away.

Our task here is to prove bounds on the following quantity
\[
\frac{1}{n-1}\sum_{t=2}^n\L(h_{t-1}) - \L(h^\ast).
\]
Proceeding as before, we will first prove the following result
\begin{eqnarray}
\Prr{Z^n}{\left.\frac{1}{n-1}\sum_{t=2}^n\L(h_{t-1}) - \L(h^\ast) > \epsilon\right|R^n} \leq \delta.
\label{eq:strongly-convex-buf-inter}
\end{eqnarray}
This will allow us, upon taking expectations over $R^n$, show the following
\[
\Prr{W^n}{\frac{1}{n-1}\sum_{t=2}^n\L(h_{t-1}) - \L(h^\ast) > \epsilon} \leq \delta,
\]
which shall complete the proof.

In order to prove the statement given in Equation~\eqref{eq:strongly-convex-buf-inter}, we will use Theorem~\ref{thm:fast-batch}. As we did in the case of all-pairs loss functions, consider the loss function $\wp(h,\vecz') := \EE{\vecz}{\ell(h,\vecz,\vecz')}$ with $\P$ and $\hat\P$ as the associated population and empirical risk functionals. Clearly, if $\ell$ is $L$-Lipschitz and $\sigma$-strongly convex then so is $\wp$. By linearity of expectation, for any $h \in \H$, $\P(h) = \L(h)$. Suppose that given $R^{t-1}$ the indices $\tau_1,\ldots,\tau_s$ are present in the buffer $B_t$ at time $t$. Applying Theorem~\ref{thm:fast-batch} on $h_{t-1}$ at the $t\th$ step with the loss function $\wp$ gives us that given $R^{t-1}$, with probability at least $1-\delta$ over the choice of $Z^{t-1}$, 
\begin{align*}
\L(h_{t-1}) - \L(h^\ast) \leq {} & (1 + \epsilon)\br{\Ltbuft(h_{t-1}) - \Ltbuft(h^\ast)}\\
								 & + \frac{C_\delta}{\epsilon \sigma (\min\bc{s,t-1})},
\end{align*}
where we have again made the simplifying (yet optional) assumption that prior to time $t = s+1$, the buffer contains an exact copy of the stream. Summing across time steps and taking a union bound, gives us that given $R^n$, with probability at least $1 - \delta$ over the choice of $Z^n$,
\begin{align*}
\frac{1}{n-1}\sum_{t=2}^n\L(h_{t-1}) \leq {} & \L(h^\ast) + \frac{C_{(\delta/n)}}{\epsilon \sigma}\br{\frac{\log 2s}{n-1} + \frac{1}{s}}\\
											 & + \frac{1 + \epsilon}{n-1}\sum_{t=2}^n\Ltbuft(h_{t-1}) - \Ltbuft(h^\ast).
\end{align*}
Let us define as before \vspace*{-3ex}

{\small
\[
\xi_t := \br{\Ltbuft(h_{t-1}) - \Ltbuft(h^\ast)} - \br{\Lbuft(h_{t-1}) - \Lbuft(h^\ast)}.
\]
}
Then using the regret bound $\Rkbufn$ we can write,
\begin{align*}
\frac{1}{n-1}\sum_{t=2}^n\L(h_{t-1}) \leq {} & \L(h^\ast) + \frac{1 + \epsilon}{n-1}\br{\Rkbufn + \sum_{t=2}^n\xi_t}\\
											 & + \frac{C_{(\delta/n)}}{\epsilon \sigma}\br{\frac{\log 2s}{n-1} + \frac{1}{s}}.
\end{align*}
Assuming $s < n/\log n$ simplifies the above expression to the following:
\begin{align*}
\frac{1}{n-1}\sum_{t=2}^n\L(h_{t-1}) \leq {} & \L(h^\ast) + \frac{1 + \epsilon}{n-1}\br{\Rkbufn + \sum_{t=2}^n\xi_t}\\
											 & + \frac{2C_{(\delta/n)}}{\epsilon \sigma s}.
\end{align*}
Note that this assumption is neither crucial to our proof nor very harsh as for $s = \Om{n}$, we can always apply the results from the \emph{infinite-buffer} setting using Theorem~\ref{thm:strongly-convex-forward}. Moving forward, by using the Bernstein-style inequality from \cite{online-batch-strongly-convex}, one can show with that probability at least $1 - \delta$, we have
\begin{align*}
\sum_{t=1}^n \xi_t \leq \max\bc{8L\sqrt{\frac{\Dk_n}{\sigma}},6B\sqrt{\log \frac{4\log n}{\delta}}}\sqrt{\log \frac{4\log n}{\delta}},
\end{align*}
where $\Dk_n = \sum_{t=2}^n\br{\L(h_{t-1}) - \L(h^\ast)}$. This gives us
\begin{align*}
\frac{\Dk_n}{n-1} \leq {} & \frac{1 + \epsilon}{n-1}\br{\Rkbufn + \max\bc{8L\sqrt{\frac{\mathfrak D_n}{\sigma}},6B\Delta}\Delta} \\
											 & + \frac{2C_{(\delta/n)}}{\epsilon \sigma s}.
\end{align*}
Using \citep[Lemma 4]{online-batch-strongly-convex} and absorbing constants inside the $\O{\cdot}$ notation we get:
\begin{align*}
\Dk_n \leq {} & \br{1 + \epsilon}\Rkbufn + \O{\frac{C_d^2n\log(n/\delta)}{\epsilon s} + \log\frac{\log n}{\delta}}\\
											 & + \O{\sqrt{\br{\Rkbufn + \frac{C_d^2n\log(n/\delta)}{\epsilon s}}\log\frac{\log n}{\delta}}}.
\end{align*}
Let ${\Wk_n} = \max\bc{\Rkbufn,\frac{2C_d^2n\log(n/\delta)}{s}}$. Concentrating only on the portion of the expression involving $\epsilon$ and ignoring the constants, we get
\begin{align*}
& \epsilon\Rkbufn + \frac{C_d^2n\log(n/\delta)}{\epsilon s} + \sqrt{\frac{C_d^2n\log(n/\delta)}{\epsilon s}\log\frac{\log n}{\delta}}\\
\leq {} & \epsilon\Rkbufn + \frac{2C_d^2n\log(n/\delta)}{\epsilon s} \leq \epsilon\Wk_n + \frac{2C_d^2n\log(n/\delta)}{\epsilon s}\\
\leq {} & 2C_d\sqrt{\frac{2\Wk_nn\log(n/\delta)}{s}},
\end{align*}
where the second step follows since $\epsilon \leq 1$ and $s \leq n$ and the fourth step follows by using $\epsilon = \sqrt\frac{2C_d^2n\log(n/\delta)}{\Wk_n s} \leq 1$
Putting this into the excess risk expression gives us
\begin{align*}
\frac{1}{n-1}\sum_{t=2}^n\L(h_{t-1}) \leq {} & \L(h^\ast) + \frac{\Rkbufn}{n-1}\\
											 & + C_d\cdot\O{\sqrt{\frac{\Wk_n\log(n/\delta)}{sn}}},
\end{align*}
which finishes the proof. Note that in case ${\Wk_n} = \Rkbufn$, we get
\begin{align*}
\frac{1}{n-1}\sum_{t=2}^n\L(h_{t-1}) \leq {} & \L(h^\ast) + \frac{\Rkbufn}{n-1}\\
											 & + C_d\cdot\O{\sqrt{\frac{\Rkbufn\log(n/\delta)}{sn}}}.
\end{align*}
On the other hand if $\Rkbufn \leq \frac{2C_d^2n\log(n/\delta)}{s}$, we get
\begin{align*}
\frac{1}{n-1}\sum_{t=2}^n\L(h_{t-1}) \leq {} & \L(h^\ast) + \frac{\Rkbufn}{n-1}\\
											 & + C_d^2\cdot\O{\frac{\log(n/\delta)}{s}}.
\end{align*}

\section{Proof of Theorem~\ref{thm:contraction-rademacher}}
\label{appsec:proof-thm-contraction-rademacher}
Recall that we are considering a composition classes of the form $\ell\circ\H := \bc{(\vecz,\vecz') \mapsto \ell(h,\vecz,\vecz'), h \in \H}$ where $\ell$ is some Lipschitz loss function. We have $\ell(h,z_1,z_2) = \phi\br{h(x_1,x_2)Y(y_1,y_2)}$ where $Y(y_1,y_2) = y_1 - y_2$ or $Y(y_1,y_2) = y_1y_2$ and $\phi : \R \rightarrow \R$ involves some margin loss function. We also assume that $\phi$ is point wise $L$-Lipschitz. Let $Y = \underset{y_1,y_2 \in \Y}{\sup}\abs{Y(y_1,y_2)}$.

\begin{thm}[Theorem~\ref{thm:contraction-rademacher} restated]
\[
\RR_n(\ell\circ\H) \leq LY\RR_n(\H)
\]
\end{thm}
\begin{proof}
Let $\tilde\phi(x) = \phi(x) - \phi(0)$. Note that $\tilde\phi(\cdot)$ is point wise $L$-Lipschitz as well as satisfies $\tilde\phi(0) = 0$. Let $Y = \underset{y,y' \in \Y}{\sup}\abs{Y(y,y')}$.

We will require the following contraction lemma that we state below.
\begin{thm}[Implicit in proof of \cite{talagrand-book}, Theorem 4.12]
Let $\H$ be a set of bounded real valued functions from some domain $\X$ and let $\vecx_1,\ldots,\vecx_n$ be arbitrary elements from $\X$. Furthermore, let $\phi_i : \R \rightarrow \R$, $i = 1,\ldots,n$ be $L$-Lipschitz functions such that $\phi_i(0) = 0$ for all $i$. Then we have
\[
\E{\underset{h \in \H}{\sup}\frac{1}{n}\sum_{i=1}^n \epsilon_i\phi_i(h(\vecx_i))} \leq L \E{\underset{h \in \H}{\sup}\frac{1}{n}\sum_{i=1}^n \epsilon_ih(\vecx_i)}.
\]
\end{thm}
Using the above inequality we can state the following chain of (in)equalities:
\begin{eqnarray*}
\RR_n(\ell\circ\H) &=& \E{\underset{h \in \H}{\sup}\frac{1}{n}\sum_{i=1}^n \epsilon_i{\ell(h,\vecz,\vecz_i)}}\\
		  &=& \E{\underset{h \in \H}{\sup}\frac{1}{n}\sum_{i=1}^n \epsilon_i{\phi\br{h(\vecx,\vecx_i)Y(y,y_i)}}}\\
		  &=& \E{\underset{h \in \H}{\sup}\frac{1}{n}\sum_{i=1}^n \epsilon_i{\tilde\phi\br{h(\vecx,\vecx_i)Y(y,y_i)}}}\\
		  && +\ \phi(0)\E{\frac{1}{n}\sum_{i=1}^n \epsilon_i}\\
		  &=& \E{\underset{h \in \H}{\sup}\frac{1}{n}\sum_{i=1}^n \epsilon_i{\tilde\phi\br{h(\vecx,\vecx_i)Y(y,y_i)}}}\\
		  &\leq& LY\E{\underset{h \in {\H}}{\sup}\frac{1}{n}\sum_{i=1}^n \epsilon_ih(\vecx,\vecx_i)}\\
		  &=& LY\RR_n(\H),
\end{eqnarray*}
where the fourth step follows from linearity of expectation. The fifth step is obtained by applying the contraction inequality to the functions $\psi_i : x \mapsto \tilde\phi(a_ix)$ where $a_i = Y(y,y_i)$. We exploit the fact that the contraction inequality is actually proven for the empirical Rademacher averages due to which we can take $a_i = Y(y,y_i)$ to be a constant dependent only on $i$, use the inequality, and subsequently take expectations. We also have, for any $i$ and any $x,y \in \R$,
\begin{eqnarray*}
\abs{\psi_i(x) - \psi_i(y)} &=& \abs{\tilde\phi(a_ix) - \tilde\phi(a_iy)}\\
							&\leq& L\abs{a_ix - a_iy}\\
							&\leq& L\abs{a_i}\abs{x - y}\\
							&\leq& LY\abs{x - y},
\end{eqnarray*}
which shows that every function $\psi_i(\cdot)$ is $LY$-Lipschitz and satisfies $\psi_i(0) = 0$. This makes an application of the contraction inequality possible on the empirical Rademacher averages which upon taking expectations give us the result.
\end{proof}

\section{Applications}
\label{appsec:apps-supp}

In this section we shall derive Rademacher complexity bounds for hypothesis classes used in various learning problems. Crucial to our derivations shall be the following result by \cite{rad-bounds}. Recall the usual definition of Rademacher complexity of a \emph{univariate} function class $\F = \bc{f : \X \rightarrow \R}$
\[
\RR_n(\F) = \E{\underset{f \in \F}{\sup}\ \frac{1}{n} \sum_{i=1}^n \epsilon_if(\vecx_i)}.
\]
\begin{thm}[\cite{rad-bounds}, Theorem 1]
\label{thm:rad-bounds}
Let $\W$ be a closed and convex subset of some Banach space equipped with a norm $\norm{\cdot}$ and dual norm $\norm{\cdot}_\ast$. Let $F : \W \rightarrow \R$ be $\sigma$-strongly convex with respect to $\norm{\cdot}_\ast$. Assume $\W \subseteq \bc{\vecw : F(\vecw) \leq W_\ast^2}$. Furthermore, let $\X = \bc{\vecx : \norm{\vecx} \leq X}$ and $\F_\W := \bc{\vecw \mapsto \ip{\vecw}{\vecx} : \vecw \in \W, \vecx \in \X}$. Then, we have
\[
\RR_n(\F_\W) \leq XW_\ast\sqrt{\frac{2}{\sigma n}}.
\]
\end{thm}
We note that Theorem~\ref{thm:rad-bounds} is applicable only to first order learning problems since it gives bounds for univariate function classes. However, our hypothesis classes consist of bivariate functions which makes a direct application difficult. Recall our extension of Rademacher averages to bivariate function classes:
\[
\RR_n(\H) = \E{\underset{h \in \H}{\sup}\ \frac{1}{n}\sum_{i=1}^n \epsilon_ih(\vecz,\vecz_i)}
\]
where the expectation is over $\epsilon_i$, $\vecz$ and $\vecz_i$. To overcome the above problem we will use the following two step proof technique:
\begin{enumerate}
	\item \textbf{Order reduction}: We shall cast our learning problems in a modified input domain where predictors behave linearly as univariate functions. More specifically, given a hypothesis class $\H$ and domain $\X$, we shall construct a modified domain $\tilde\X$ and a map $\psi : \X \times \X \rightarrow \tilde\X$ such that for any $\vecx, \vecx' \in \X$ and $h \in \H$, we have $h(\vecx,\vecx') = \ip{h}{\psi(\vecx,\vecx')}$.
	\item \textbf{Conditioning}: For every $\vecx \in \X$, we will create a function class $\F_{\vecx} = \bc{\vecx' \mapsto \ip{h}{\psi(\vecx,\vecx')}: h \in \H}$. Since $\F_\vecx$ is a univariate function class, we will use Theorem~\ref{thm:rad-bounds} to bound $\RR_n(\F_\vecx)$. Since $\RR_n(\H) = \EE{\vecx}{\RR_n(\F_\vecx)}$, we shall obtain Rademacher complexity bounds for $\H$.
\end{enumerate}
 We give below some examples of learning situations where these results may be applied.

As before, for any subset $X$ of a Banach space and any norm $\norm{\cdot}_p$, we define $\norm{X}_p := \underset{\vecx \in X}{\sup} \norm{\vecx}_p$. We also define norm bounded balls in the Banach space as $\B_p(r) := \bc{\vecx : \norm{\vecx}_p \leq r}$ for any $r > 0$. Let the domain $\X$ be a subset of $\R^d$.

For sake of convenience we present the examples using loss functions for classification tasks but the same can be extended to other learning problems such as regression, multi-class classification and ordinal regression.

\subsection{AUC maximization for Linear Prediction}

\begin{table}[t]
	\centering
	\begin{tabular}{|c|c|}
		\hline
		Hypothesis class & Rademacher Complexity\\\hline
		$\B_q(\norm{\W}_q)$ & $2\norm{\X}_p\norm{\W}_q\sqrt{\frac{p - 1}{n}}$\\\hline
		$\B_1(\norm{\W}_1)$& $2\norm{\X}_\infty \norm{\W}_1\sqrt{\frac{e\log d}{n}}$\\\hline
	\end{tabular}
	\caption{Rademacher complexity bounds for AUC maximization. We have $1/p+1/q = 1$ and $q > 1$.}
	\label{tab:rad-bounds-auc}
\end{table}

In this case the goal is to maximize the area under the ROC curve for a linear classification problem at hand. This translates itself to a learning situation where $\W, \X \subseteq \R^d$. We have $h_{\vecw}(\vecx,\vecx') = \vecw^\top\vecx-\vecw^\top\vecx'$ and $\ell(h_\vecw,z_1,z_2) = \phi\br{(y-y')h_{\vecw}(\vecx,\vecx')}$ where $\phi$ is the hinge loss or the exponential loss \cite{oam-icml}.

In order to apply Theorem~\ref{thm:rad-bounds}, we rewrite the hypothesis as $h_{\vecw}(\vecx,\vecx') = \vecw^\top(\vecx - \vecx')$ and consider the input domain $\tilde\X = \bc{\vecx - \vecx' : \vecx,\vecx' \in \X}$ and the map $\psi : (\vecx,\vecx') \mapsto \vecx - \vecx'$. Clearly if $\X \subseteq \bc{\vecx : \norm{\vecx} \leq X}$ then $\tilde\X \subseteq \bc{\vecx : \norm{\vecx} \leq 2X}$ and thus we have $\norm{\tilde\X} \leq 2\norm{\X}$ for any norm $\norm{\cdot}$. It is now possible to regularize the hypothesis class $\W$ using a variety of norms.

If we wish to define our hypothesis class as $\B_q(\cdot), q > 1$, then in order to apply Theorem~\ref{thm:rad-bounds}, we can use the regularizer $F(\vecw) = \norm{\vecw}_q^2$. If we wish the sparse hypotheses class, $\B_1(W_1)$, we can use the regularizer $F(\vecw) = \norm{\vecw}_q^2$ with $q = \frac{\log d}{\log d - 1}$ as this regularizer is strongly convex with respect to the $L_1$ norm \cite{ambuj-learn-matrices}. Table~\ref{tab:rad-bounds-auc} gives a succinct summary of such possible regularizations and corresponding Rademacher complexity bounds.

\textit{Kernelized AUC maximization}: Since the $L_2$ regularized hypothesis class has a dimension independent Rademacher complexity, it is possible to give guarantees for algorithms performing AUC maximization using kernel classifiers as well. In this case we have a Mercer kernel $K$ with associated reproducing kernel Hilbert space $\H_K$ and feature map $\Phi_K : \X \rightarrow \H_K$. Our predictors lie in the RKHS, i.e., $\vecw \in \H_K$ and we have $h_{\vecw}(\vecx,\vecx') = \vecw^\top\br{\Phi_K(\vecx)-\Phi_K(\vecx')}$. In this case we will have to use the map $\psi : (\vecx,\vecx') \mapsto \Phi_K(\vecx) - \Phi_K(\vecx') \in \H_K$. If the kernel is bounded, i.e., for all $\vecx, \vecx' \in \X$, we have $\abs{K(\vecx,\vecx')} \leq \kappa^2$, then we can get a Rademacher average bound of $2\kappa \norm{\W}_2\sqrt{\frac{1}{n}}$.

\subsection{Linear Similarity and Mahalanobis Metric learning}

\begin{table}[t]
	\centering
	\begin{tabular}{|c|c|}
		\hline
		Hypothesis Class & Rademacher Complexity\\\hline
		$\B_{2,2}(\norm{\W}_{2,2})$ & $\norm{\X}_2^2\norm{\W}_{2,2}\sqrt{\frac{1}{n}}$\\\hline
		$\B_{2,1}(\norm{\W}_{2,1})$ & $\norm{\X}_2\norm{\X}_\infty \norm{\W}_{2,1}\sqrt{\frac{e\log d}{n}}$\\\hline
		$\B_{1,1}(\norm{\W}_{1,1})$ & $\norm{\X}_\infty^2 \norm{\W}_{1,1}\sqrt{\frac{2e\log d}{n}}$\\\hline
		$\B_{S(1)}(\norm{\W}_{S(1)})$ & $\norm{\X}_2^2\norm{\W}_{S(1)}\sqrt{\frac{e\log d}{n}}$\\\hline
	\end{tabular}
	\caption{Rademacher complexity bounds for Similarity and Metric learning}
	\label{tab:metric-rad-bounds}
\end{table}
A variety of applications, such as in vision, require one to fine tune one's notion of proximity by learning a similarity or metric function over the input space. We consider some such examples below. In the following, we have $\vecW \in \R^{d \times d}$.
\begin{enumerate}
	\item \textit{Mahalanobis metric learning}: in this case we wish to learn a metric $M_\vecW(\vecx,\vecx') = (\vecx-\vecx')^\top\vecW(\vecx-\vecx')$ using the loss function $\ell(M_\vecW,\vecz,\vecz') = \phi\br{yy'\br{1- M_\vecW^2(\vecx,\vecx')}}$ \cite{reg-metric-learn}.
	\item \textit{Linear kernel learning}: in this case we wish to learn a linear kernel function $K_\vecW(\vecx,\vecx') = \vecx^\top \vecW \vecx', \vecW \succeq 0$. A variety of loss functions have been proposed to aid the learning process
	\begin{enumerate}
		\item \emph{Kernel-target Alignment}: the loss function used is $\ell(K_\vecW,\vecz,\vecz') = \phi\br{yy'K_\vecW(\vecx,\vecx')}$ where $\phi$ is used to encode some notion of alignment \cite{kernel-target-alignment, two-stage-cortes}.
		\item \emph{S-Goodness}: this is used in case one wishes to learn a \emph{good} similarity function that need not be positive semi definite \cite{good-sim-learn, good-sim} by defining $\ell(K_\vecW,\vecz) = \phi\br{y\EE{(\vecx',y')}{y'K_\vecW(\vecx,\vecx')}}$.
	\end{enumerate}
\end{enumerate}
In order to apply Theorem~\ref{thm:rad-bounds}, we will again rewrite the hypothesis and consider a different input domain. For the similarity learning problem, write the similarity function as $K_\vecW(\vecx,\vecx') = \ip{\vecW}{\vecx\vecx'^\top}$ and consider the input space $\tilde\X = \bc{{\vecx\vecx'^\top} : \vecx,\vecx' \in \X} \subseteq \R^{d \times d}$ along with the map $\psi : (\vecx,\vecx') \mapsto \vecx\vecx'^\top$. For the metric learning problem, rewrite the metric as $M_\vecW(\vecx,\vecx') = \ip{\vecW}{(\vecx-\vecx')(\vecx-\vecx')^\top}$ and consider the input space $\tilde\X = \bc{{(\vecx-\vecx')(\vecx-\vecx')^\top} : \vecx,\vecx' \in \X} \subseteq \R^{d \times d}$ along with the map $\psi : (\vecx,\vecx') \mapsto (\vecx-\vecx')(\vecx-\vecx')^\top$.

In this case it is possible to apply a variety of matrix norms to regularize the hypothesis class. We consider the following (mixed) matrix norms : $\norm{\cdot}_{1,1}$, $\norm{\cdot}_{2,1}$ and $\norm{\cdot}_{2,2}$.
We also consider the Schatten norm $\norm{\vecX}_{S(p)} := \norm{\vecsigma(\vecX)}_p$ that includes the widely used \emph{trace norm} $\norm{\vecsigma(\vecX)}_1$. As before, we define norm bounded balls in the Banach space as follows: $\B_{p,q}(r) := \bc{\vecx : \norm{\vecx}_{p,q} \leq r}$.

Using results on construction of strongly convex functions with respect to theses norms from \cite{ambuj-learn-matrices}, it is possible to get bounds on the Rademacher averages of the various hypothesis classes. However these bounds involve norm bounds for the modified domain $\tilde\X$. We make these bounds explicit by expressing norm bounds for $\tilde\X$ in terms of those for $\X$. From the definition of $\tilde\X$ for the similarity learning problems, we get, for any $p,q \geq 1$, $\norm{\tilde\X}_{p,q} \leq \norm{\X}_p\norm{\X}_q$. Also, since every element of $\tilde\X$ is of the form $\vecx\vecx'^\top$, it has only one non zero singular value $\norm{\vecx}_2\norm{\vecx'}_2$ which gives us $\norm{\tilde X}_{S(p)} \leq \norm{\X}_2^2$ for any $p \geq 1$.

For the metric learning problem, we can similarly get $\norm{\tilde X}_{p,q} \leq 4\norm{\X}_p\norm{\X}_q$ and $\norm{\tilde X}_{S(p)} \leq 4\norm{\X}_2^2$ for any $p \geq 1$ which allows us to get similar bounds as those for similarity learning but for an extra constant factor. We summarize our bounds in Table~\ref{tab:metric-rad-bounds}. We note that \cite{gen-bound-metric-learn} devote a substantial amount of effort to calculate these values for the mixed norms on a case-by-case basis (and do not consider Schatten norms either) whereas, using results exploiting strong convexity and strong smoothness from \cite{ambuj-learn-matrices}, we are able to get the same as simple corollaries.

\subsection{Two-stage Multiple kernel learning}

\begin{table}[t]
	\centering
	\begin{tabular}{|c|c|}
		\hline
		Hypothesis Class & Rademacher Avg. Bound\\\hline
		$\S_2(1)$ & $\kappa^2\sqrt{\frac{p}{n}}$\\\hline
		$\Delta(1)$ & $\kappa^2\sqrt{\frac{e\log p}{n}}$\\\hline
	\end{tabular}
	\caption{Rademacher complexity bounds for Multiple kernel learning}
	\label{tab:mkl-rad-bounds}
\end{table}

The analysis of the previous example can be replicated for learning non-linear Mercer kernels as well. Additionally, since all Mercer kernels yield Hilbertian metrics, these methods can be extended to learning Hilbertian metrics as well. However, since Hilbertian metric learning has not been very popular in literature, we restrict our analysis to kernel learning alone. We present this example using the framework proposed by \cite{two-stage-mkl-general} due to its simplicity and generality.

We are given $p$ Mercer kernels $K_1,\ldots,K_p$ that are bounded, i.e., for all $i$, $\abs{K_i(\vecx,\vecx')} \leq \kappa^2$ for all $\vecx,\vecx' \in \X$ and our task is to find a combination of these kernels given by a vector $\vecmu \in \R^p, \vecmu \geq 0$ such that the kernel $K_{\vecmu}(\vecx,\vecx')=\sum_{i=1}^p\vecmu_iK_i(\vecx,\vecx')$ is a \emph{good} kernel \cite{good-sim}. In this case the loss function used is $\ell(\vecmu,\vecz,\vecz') = \phi\br{yy'K_{\vecmu}(\vecx,\vecx')}$ where $\phi(\cdot)$ is meant to encode some notion of alignment. \citet{two-stage-mkl-general} take $\phi(\cdot)$ to be the hinge loss.

To apply Theorem~\ref{thm:rad-bounds}, we simply use the ``K-space'' construction proposed in \cite{two-stage-mkl-general}. We write $K_{\vecmu}(\vecx,\vecx') = \ip{\vecmu}{z(\vecx,\vecx')}$ where $z(\vecx,\vecx') = \br{K_1(\vecx,\vecx'),\ldots,K_p(\vecx,\vecx')}$. Consequently our modified input space looks like $\tilde\X = \bc{z(\vecx,\vecx') : \vecx,\vecx' \in \X} \subseteq \R^p$ with the map $\psi : (\vecx,\vecx') \mapsto z(\vecx,\vecx')$. Popular regularizations on the kernel combination vector $\vecmu$ include the sparsity inducing $L_1$ regularization that constrains $\vecmu$ to lie on the unit simplex $\Delta(1) = \bc{\vecmu : \norm{\vecmu}_1 = 1, \vecmu \geq 0}$ and $L_2$ regularization that restricts $\vecmu$ to lie on the unit sphere $\S_2(1) = \bc{\vecmu : \norm{\vecmu}_2 = 1, \vecmu \geq 0}$. Arguments similar to the one used to discuss the case of AUC maximization for linear predictors give us bounds on the Rademacher averages for these two hypothesis classes in terms of $\norm{\tilde X}_2$ and $\norm{\tilde X}_\infty$. Since $\norm{\tilde X}_2 \leq \kappa^2\sqrt p$ and $\norm{\tilde X}_\infty \leq \kappa^2$, we obtain explicit bounds on the Rademacher averages that are given in Table~\ref{tab:mkl-rad-bounds}.

We note that for the $L_1$ regularized case, our bound has a similar dependence on the number of kernels, i.e., $\sqrt{\log p}$ as the bounds presented in \cite{mkl-bound-cortes}. For the $L_2$ case however, we have a worse dependence of $\sqrt p$ than \citet{mkl-bound-cortes} who get a $\sqrt[4]{p}$ dependence. However, it is a bit unfair to compare the two bounds since \citet{mkl-bound-cortes} consider single stage kernel learning algorithms that try to learn the kernel combination as well as the classifier in a single step whereas we are dealing with a two-stage process where classifier learning is disjoint from the kernel learning step.

\section{Regret Bounds for Reservoir Sampling Algorithms}
\label{appsec:regret-rs}

The Reservoir Sampling algorithm \cite{vitter-rs} essentially performs sampling without replacement which means that the samples present in the buffer are not i.i.d. samples from the preceding stream. Due to this, proving regret bounds by way of uniform convergence arguments becomes a bit more difficult. However, there has been a lot of work on analyzing learning algorithms that learn from non-i.i.d. data such as data generated by ergodic processes. Of particular interest is a result by Serfling \footnote{R. J. Serfling, Probability Inequalities for the Sum in Sampling without Replacement, \emph{The Annals of Statistics}, 2(1):39-48, 1974.} that gives Hoeffding style bounds for data generated from a finite population without replacement.

Although Serfling's result does provide a way to analyze the \rsorig algorithm, doing so directly would require using arguments that involve covering numbers that offer bounds that are dimension dependent and that are not tight. It would be interesting to see if equivalents of the McDiarmid's inequality and Rademacher averages can be formulated for samples obtained without replacement to get tighter results. For our purposes, we remedy the situation by proposing a new sampling algorithm that gives us i.i.d. samples in the buffer allowing existing techniques to be used to obtain regret bounds (see Appendices~\ref{appsec:rs-x-analysis} and \ref{appsec:proof-thm-rs-x-regret}).

\section{Analysis of the \mrs Algorithm}
\label{appsec:rs-x-analysis}
In this section we analyze the \mrs substream sampling algorithm and prove its statistical properties. Recall that the \mrs algorithm simply admits a point into the buffer if there is space. It performs a \emph{Repopulation step} at the first instance of overflow which involves refilling the buffer by sampling with replacement from all the set of points seen so far (including the one that caused the overflow). In subsequent steps, a \emph{Normal update step} is performed. The following theorem formalizes the properties of the sampling algorithm

\begin{thm}
\label{thm:rs-x-prop}
Suppose we have a stream of elements $\vecz_1,\ldots,\vecz_n$ being sampled into a buffer $B$ of size $s$ using the \mrs algorithm. Then at any time $t \geq s + 2$, each element of $B$ is an i.i.d. sample from the set $Z^{t-1}$.
\end{thm}
\begin{proof}
To prove the results, let us assume that the buffer contents are addressed using the variables $\zeta_1, \ldots, \zeta_s$. We shall first concentrate on a fixed element, say $\zeta_1$ (which we shall call simply $\zeta$ for notational convenience) of the buffer and inductively analyze the probability law $\P_t$ obeyed by $\zeta$ at each time step $t \geq s + 2$.

We will prove that the probability law obeyed by $\zeta$ at time $t$ is $\P_t(\zeta) = \frac{1}{t-1}\sum_{\tau=1}^{t-1}\ind_{\bc{\zeta = \vecz_\tau}}$. The law is interpreted as saying the following: for any $\tau \leq t-1$, $\Pr{\zeta = \vecz_\tau} = \frac{1}{t-1}$ and shows that the element $\zeta$ is indeed a uniform sample from the set $Z^{t-1}$. We would similarly be able to show this for all locations $\zeta_2, \ldots, \zeta_s$ which would prove that the elements in the buffer are indeed identical samples from the preceding stream. Since at each step, the \mrs algorithm updates all buffer locations independently, the random variables $\zeta_1, \ldots, \zeta_s$ are independent as well which would allow us to conclude that at each step we have $s$ i.i.d. samples in the buffer as claimed.

We now prove the probability law for $\zeta$. We note that the repopulation step done at time $t = s+1$ explicitly ensures that at step $t = s+2$, the buffer contains $s$ i.i.d samples from $Z^{s+1}$ i.e. $\P_{s+2}(\zeta) = \frac{1}{s+1}\sum_{\tau=1}^{s+1}\ind_{\bc{\zeta = \vecz_\tau}}$. This forms the initialization of our inductive argument. Now suppose that at the $t\th$ time step, the claim is true and $\zeta$ obeys the law $\P_t(\zeta) = \frac{1}{t-1}\sum_{\tau=1}^{t-1}\ind_{\bc{\zeta = \vecz_\tau}}$. At the $t\th$ step, we would update the buffer by making the incoming element $\vecz_t$ replace the element present at the location indexed by $\zeta$ with probability $1/(t+1)$. Hence $\zeta$ would obey the following law after the update
\[
\br{1 - \frac{1}{t}}\P_t(\zeta) + \frac{1}{t} \ind_{\bc{\zeta = \vecz_t}} = \frac{1}{t}\sum_{\tau=1}^{t}\ind_{\bc{\zeta = \vecz_\tau}}
\]
which shows that at the $(t+1)\th$ step, $\zeta$ would follow the law $\P_{t+1}(\zeta) = \frac{1}{t}\sum_{\tau=1}^{t}\ind_{\bc{\zeta = \vecz_\tau}}$ which completes the inductive argument and the proof.
\end{proof}

\section{Proof of Theorem~\ref{thm:rs-x-regret}}
\label{appsec:proof-thm-rs-x-regret}
We now prove Theorem~\ref{thm:rs-x-regret} that gives a high confidence regret bound for the \olp learning algorithm when used along with the \mrs buffer update policy. Our proof proceeds in two steps: in the first step we prove a uniform convergence type guarantee that would allow us to convert regret bounds with respect to the \emph{finite-buffer} penalties $\Lbuft$ into regret bounds in in terms of the \emph{all-pairs} loss functions $\hat\L_t$. In the second step we then prove a regret bound for \olp with respect to the \emph{finite-buffer} penalties.

We proceed with the first step of the proof by proving the lemma given below. Recall that for any sequence of training examples $\vecz_1,\ldots,\vecz_n$, we define, for any $h \in \H$, the all-pairs loss function as $\hat\L_t(h) = \frac{1}{t-1}\sum_{\tau = 1}^{t-1}\ell(h,\vecz_t,\vecz_\tau)$. Moreover, if the online learning process uses a buffer, the we also define the \emph{finite-buffer} loss function as $\Lbuft(h_{t-1}) = \frac{1}{\abs{B_t}}\sum_{\vecz \in B_t}\ell(h_{t-1},\vecz_t,\vecz)$.

\begin{lem}
\label{lem:rs-x-unif-conv}
Suppose we have an online learning algorithm that incurs buffer penalties based on a buffer $B$ of size $s$ that is updated using the \mrs algorithm. Suppose further that the learning algorithm generates an ensemble $h_1,\ldots,h_{n-1}$. Then for any $t \in [1,n-1]$, with probability at least $1 - \delta$ over the choice of the random variables used to update the buffer $B$ until time $t$, we have
\[
\hat\L_t(h_{t-1}) \leq \Lbuft(h_{t-1}) + C_d\cdot\O{\sqrt{\frac{\log\frac{1}{\delta}}{s}}}
\]
\end{lem}
\begin{proof}
Suppose $t \leq s + 1$, then since at that point the buffer stores the stream exactly, we have
\[
\hat\L_t(h_{t-1}) = \Lbuft(h_{t-1})
\]
which proves the result. Note that, as Algorithm~\ref{algo:olp} indicates, at step $t = s+1$ the buffer is updated (using the repopulation step) only after the losses have been calculated and hence step $t = s+1$ still works with a buffer that stores the stream exactly.

We now analyze the case $t > s + 1$. At each step $\tau > s$, the \mrs algorithm uses $s$ independent Bernoulli random variables (which we call \emph{auxiliary random variables}) to update the buffer, call them $r^\tau_1,\ldots,r^\tau_s$ where $r^\tau_j$ is used to update the $j\th$ item $\zeta_j$ in the buffer. Let $\rk^t_j := \{r^{s+1}_j,r^2_j,\ldots,r^t_j\} \in \bc{0,1}^t$ denote an ensemble random variable composed of $t-s$ independent Bernoulli variables. It is easy to see that the element $\zeta_j$ is completely determined at the $t\th$ step given $\rk^{t-1}_j$.

Theorem~\ref{thm:rs-x-prop} shows, for any $t > s + 1$, that the buffer contains $s$ i.i.d. samples from the set $Z^{t-1}$. Thus, for any \emph{fixed} function $h \in \H$, we have for any $j \in [s]$,
\[
\EE{\rk^{t-1}_j}{\ell(h,\vecz_t,\zeta_j)} = \frac{1}{t-1}\sum_{\tau=1}^{t-1}\ell(h,\vecz_t,\vecz_\tau)
\]
which in turn shows us that
\[
\EE{\rk^{t-1}_1,\ldots,\rk^{t-1}_s}{\Lbuft(h)} = \frac{1}{t-1}\sum_{\tau=1}^{t-1}\ell(h,\vecz_t,\vecz_\tau) = \hat\L_t(h)
\]
Now consider a ghost sample of auxiliary random variables $\tilde\rk^{t-1}_1,\ldots,\tilde\rk^{t-1}_s$. Since our hypothesis $h_{t-1}$ is independent of these ghost variables, we can write
\[
\EE{\tilde\rk^{t-1}_1,\ldots,\tilde\rk^{t-1}_s}{\Lbuft(h_{t-1})} = \hat\L_t(h_{t-1})
\]
We recall that error in the proof presented in \citet{oam-icml} was to apply such a result on the \emph{true} auxiliary variables upon which $h_{t-1}$ is indeed dependent. Thus we have
\begin{align*}
& \hat\L_t(h_{t-1}) - \Lbuft(h_{t-1})\\
= {} & \EE{\tilde\rk^{t-1}_1,\ldots,\tilde\rk^{t-1}_s}{\Lbuft(h_{t-1})} - \Lbuft(h_{t-1})\\
\leq {} & \underbrace{\underset{h \in \H}{\sup}\bs{\EE{\tilde\rk^{t-1}_1,\ldots,\tilde\rk^{t-1}_s}{\Lbuft(h)} - \Lbuft(h)}}_{g_t(\rk^{t-1}_1,\ldots,\rk^{t-1}_s)}
\end{align*}
Now, the perturbation to any of the ensemble variables $\rk_j$ (a perturbation to an ensemble variable implies a perturbation to one or more variables forming that ensemble) can only perturb only the element $\zeta_j$ in the buffer. Since $\Lbuft(h_{t-1}) = \frac{1}{s}\sum_{\vecz \in B_t}\ell(h_{t-1},\vecz_t,\vecz)$ and the loss function is $B$-bounded, this implies that a perturbation to any of the ensemble variables can only perturb $g(\rk^{t-1}_1,\ldots,\rk^{t-1}_s)$ by at most $B/s$. Hence an application of McDiarmid's inequality gives us, with probability at least $1 - \delta$,
\begin{align*}
g_t(\rk^{t-1}_1,\ldots,\rk^{t-1}_s) \leq & {} \EE{\rk^{t-1}_j}{g_t(\rk^{t-1}_1,\ldots,\rk^{t-1}_s)} + B\sqrt\frac{\log\frac{1}{\delta}}{2s}
\end{align*}
Analyzing the expectation term we get
\begin{align*}
& \EE{\rk^{t-1}_j}{g_t(\rk^{t-1}_1,\ldots,\rk^{t-1}_s)}\\
= {} & \EE{\rk^{t-1}_j}{\underset{h \in \H}{\sup}\bs{\EE{\tilde\rk^{t-1}_1,\ldots,\tilde\rk^{t-1}_s}{\Lbuft(h)} - \Lbuft(h)}}\\
\leq {} & \EE{\rk^{t-1}_j,\tilde\rk^{t-1}_j}{\underset{h \in \H}{\sup}\bs{\frac{1}{s}\sum_{j=1}^s\ell(h,\vecz_t,\tilde\zeta_j) - \ell(h,\vecz_t,\zeta_j)}}\\
= {} & \EE{\rk^{t-1}_j,\tilde\rk^{t-1}_j,\epsilon_j}{\underset{h \in \H}{\sup}\bs{\frac{1}{s}\sum_{j=1}^s \epsilon_j\br{\ell(h,\vecz_t,\tilde\zeta_j) - \ell(h,\vecz_t,\zeta_j)}}}\\
\leq {} & 2\EE{\rk^{t-1}_j,\tilde\rk^{t-1}_j,\epsilon_j}{\underset{h \in \H}{\sup}\bs{\frac{1}{s}\sum_{j=1}^s \epsilon_j\ell(h,\vecz_t,\zeta_j)}}\\
\leq {} & 2\RR_s(\ell\circ\H)
\end{align*}
where in the third step we have used the fact that symmetrizing a pair of true and ghost ensemble variables is equivalent to symmetrizing the buffer elements they determine. In the last step we have exploited the definition of Rademacher averages with the (empirical) measure $\frac{1}{t-1}\sum_{\tau=1}^{t-1}\delta_{\vecz_\tau}$ imposed over the domain $\ZZ$.

For hypothesis classes for which we have $\hat\RR_s(\ell\circ\H) = C_d\cdot\O{\sqrt{\frac{1}{s}}}$, this proves the claim.
\end{proof}
Using a similar proof progression we can also show the following:
\begin{lem}
\label{lem:rs-x-unif-conv-converse}
For any fixed $h \in \H$ and any $t \in [1,n-1]$, with probability at least $1 - \delta$ over the choice of the random variables used to update the buffer $B$ until time $t$, we have
\[
\Lbuft(h) \leq \hat\L_t(h) + C_d\cdot\O{\sqrt{\frac{\log\frac{1}{\delta}}{s}}}
\]
\end{lem}

Combining Lemmata~\ref{lem:rs-x-unif-conv} and \ref{lem:rs-x-unif-conv-converse} and taking a union bound over all time steps, the following corollary gives us a \emph{buffer to all-pairs} conversion bound.
\begin{lem}
\label{lem:rs-x-buf-to-all-pairs}
Suppose we have an online learning algorithm that incurs buffer penalties based on a buffer $B$ of size $s$ that is updated using the \mrs algorithm. Suppose further that the learning algorithm generates an ensemble $h_1,\ldots,h_{n-1}$. Then with probability at least $1 - \delta$ over the choice of the random variables used to update the buffer $B$, we have
\[
\Rk_n \leq \Rkbufn + C_d\br{n-1}\cdot\O{\sqrt{\frac{\log\frac{n}{\delta}}{s}}},
\]
where we recall the definition of the \emph{all-pairs} regret as
\[
\Rk_n := \sum_{t=2}^n\hat\L_t(h_{t-1}) - \underset{h \in \H}{\inf}\sum_{t=2}^n\hat\L_t(h)
\]
and the \emph{finite-buffer} regret as
\[
\Rkbufn := \sum_{t=2}^n\Lbuft(h_{t-1}) - \underset{h \in \H}{\inf}\sum_{t=2}^n\Lbuft(h).
\]
\end{lem}
\begin{proof}
Let $\hat h : = \underset{h \in \H}{\arg\inf}\ \sum_{t=2}^n\hat\L_t(h)$. Then Lemma~\ref{lem:rs-x-unif-conv-converse} gives us, upon summing over $t$ and taking a union bound,
\begin{align}
\sum_{t=2}^n\Lbuft(\hat h) \leq \sum_{t=2}^n\hat\L_t(\hat h) + C_d(n-1)\cdot\O{\sqrt{\frac{\log\frac{n}{\delta}}{s}}},
\label{eq:modified-rs-regret-competitor-guarantee}
\end{align}
whereas Lemma~\ref{lem:rs-x-unif-conv} similarly guarantees
\begin{align}
\sum_{t=2}^n\hat\L_t(h_{t-1}) \leq \sum_{t=2}^n\Lbuft(h_{t-1}) + C_d(n-1)\cdot\O{\sqrt{\frac{\log\frac{n}{\delta}}{s}}},
\label{eq:modified-rs-regret-ensemble-guarantee}
\end{align}
where both results hold with high confidence. Adding the Equations~\eqref{eq:modified-rs-regret-competitor-guarantee} and \eqref{eq:modified-rs-regret-ensemble-guarantee} and using $\sum_{t=2}^n\Lbuft(h_{t-1}) \leq \underset{h \in \H}{\inf}\sum_{t=2}^n\Lbuft(\hat h) + \Rkbufn$ completes the proof.
\end{proof}

As the final step of the proof, we give below a \emph{finite-buffer} regret bound for the \olp algorithm.

\begin{lem}
\label{lem:finite-buffer-regret-bound-olp}
Suppose the \olp algorithm working with an $s$-sized buffer generates an ensemble $\vecw_1,\ldots,\vecw_{n-1}$. Further suppose that the loss function $\ell$ being used is $L$-Lipschitz and the space of hypotheses $\W$ is a compact subset of a Banach space with a finite diameter $D$ with respect to the Banach space norm. Then we have
\[
\Rkbufn \leq LD\sqrt{n-1}
\]
\end{lem}
\begin{proof}
We observe that the algorithm \olp is simply a variant of the GIGA algorithm \cite{zinkevich} being applied with the loss functions $\ell^{\text{GIGA}}_t : \vecw \mapsto \Lbuft(\vecw)$. Since $\ell^{\text{GIGA}}_t$ inherits the Lipschitz constant of $\Lbuft$ which in turn inherits it from $\ell$, we can use the analysis given by \citet{zinkevich} to conclude the proof.
\end{proof}
Combining Lemmata~\ref{lem:rs-x-buf-to-all-pairs} and \ref{lem:finite-buffer-regret-bound-olp} gives us the following result:
\begin{thm}[Theorem~\ref{thm:rs-x-regret} restated]
Suppose the \olp algorithm working with an $s$-sized buffer generates an ensemble $\vecw_1,\ldots,\vecw_{n-1}$. Then with probability at least $1 - \delta$,
\[
\frac{\Rk_n}{n-1} \leq \O{C_d\sqrt{\frac{\log \frac{n}{\delta}}{s}} + \sqrt{\frac{1}{n-1}}}
\]
\end{thm}

\section{Implementing the \mrs Algorithm}
\label{appsec:rs-x-implementation}
Although the \mrs algorithm presented in the paper allows us to give clean regret bounds, it suffers from a few drawbacks. From a theoretical point of view, the algorithm is inferior to Vitter's \rsorig algorithm in terms of randomness usage. The \rsorig algorithm (see \cite{oam-icml} for example) uses a Bernoulli random variable and a discrete uniform random variable at each time step. The discrete random variable takes values in $[s]$ as a result of which the algorithm uses a total of $\O{\log s}$ random bits at each step.

The \mrs algorithm as proposed, on the other hand, uses $s$ Bernoulli random variables at each step (to decide which buffer elements to replace with the incoming point) taking its randomness usage to $\O{s}$ bits. From a practical point of view this has a few negative consequences:
\begin{enumerate}
	\item Due to increased randomness usage, the variance of the resulting algorithm increases.
	\item At step $t$, the Bernoulli random variables required all have success probability $1/t$. This quantity drops down to negligible values for even moderate values of $t$. Note that Vitter's \rsorig on the other hand requires a Bernoulli random variable with success probability $s/t$ which dies down much more slowly.
	\item Due to the requirement of such high precision random variables, the imprecisions of any pseudo random generator used to simulate this algorithm become apparent resulting in poor performance. 
\end{enumerate}

\begin{algorithm}[t]
	\caption{\small \mmrs: An Alternate Implementation of the \mrs Algorithm}
	\label{algo:rs-x-alternate}
	\begin{algorithmic}[1]
		\small{
			\REQUIRE Buffer $B$, new point $\vecz_t$, buffer size $s$, timestep $t$
			\ENSURE Updated buffer $B_\text{new}$
			\IF[There is space]{$|B| < s$}
				\STATE $B_\text{new} \leftarrow B \cup \bc{\vecz_t}$
			\ELSE[Overflow situation]
				\IF[Repopulation step]{$t = s+1$}
					\STATE $\text{TMP} = B \cup \bc{\vecz_t}$
					\STATE $B_\text{new} = \phi$
					\FOR{$i = 1$ \TO $s$}
						\STATE Select random $\vecr \in \text{TMP}$ with replacement
						\STATE $B_\text{new} \leftarrow B_\text{new} \cup \bc{\vecr}$
					\ENDFOR
				\ELSE[Normal update step]
					\STATE $B_\text{new} \leftarrow B$
					\STATE Sample $k \sim \text{Binomial}(s,1/t)$
					\STATE Remove $k$ random elements from $B_\text{new}$
					\STATE $B_\text{new} \leftarrow B_\text{new} \cup \br{\coprod_{i=1}^{k} \bc{\vecz_t}}$
				\ENDIF
			\ENDIF
			\RETURN{$B_\text{new}$}
		}
	\end{algorithmic}
\end{algorithm}

In order to ameliorate the situation, we propose an alternate implementation of the \emph{normal} update step of the \mrs algorithm in Algorithm~\ref{algo:rs-x-alternate}. We call this new sampling policy \mmrs\!. We shall formally demonstrate the equivalence of the \mrs and the \mmrs policies by showing that both policies result in a buffer whose each element is a uniform sample from the preceding stream with replacement. This shall be done by proving that the joint distribution of the buffer elements remains the same whether the \mrs \emph{normal} update is applied or the \mmrs \emph{normal} step is applied (note that \mrs and \mmrs have identical \emph{repopulation steps}). This will ensure that any learning algorithm will be unable to distinguish between the two update mechanisms and consequently, our regret guarantees shall continue to hold.



First we analyze the randomness usage of the \mmrs update step.
The update step first samples a number $K_t \sim B(s,1/t)$ from the binomial distribution and then replaces $K_t$ random locations with the incoming point. Choosing $k$ locations without replacement from a pool of $s$ locations requires at most $k\log s$ bits of randomness. Since $K_t$ is sampled from the binomial distribution $B(s,1/t)$, we have $K_t = \O{1}$ in expectation (as well as with high probability) since $t > s$ whenever this step is applied. Hence our randomness usage per update is at most $\O{\log s}$ random bits which is much better than the randomness usage of \mrs and that actually matches that of Vitter's \rsorig upto a constant.

To analyze the statistical properties of the \mmrs update step, let us analyze the state of the buffer after the update step. In the \mrs algorithm, the state of the buffer after an update is completely specified once we enumerate the locations that were replaced by the incoming point. Let the indicator variable $R_i$ indicate whether the $i\th$ location was replaced or not. Let $r \in \bc{0,1}^s$ denote a \emph{fixed} pattern of replacements. Then the original implementation of the update step of \mrs guarantees that
\[
\Prr{\text{\mrs}}{\bigwedge_{i=1}^s \br{R_i = r_i}} = \br{\frac{1}{t}}^{\norm{r}_1}\br{1 - \frac{1}{t}}^{s - \norm{r}_1}
\]
To analyze the same for the alternate implementation of the \mmrs update step, we first notice that choosing $k$ items from a pool of $s$ without replacement is identical to choosing the first $k$ locations from a random permutation of the $s$ items. Let us denote $\norm{r}_1 = k$. Then we have,
\begin{align*}
\Prr{\text{\mmrs}}{\bigwedge_{i=1}^s \br{R_i = r_i}} &= \sum_{j = 1}^s\Pr{\bigwedge_{i=1}^s \br{R_i = r_i} \wedge K_t = j}\\
									  &= \Pr{\bigwedge_{i=1}^s \br{R_i = r_i} \wedge K_t = k}\\
									  &= \Pr{\left.\bigwedge_{i=1}^s \br{R_i = r_i}\right|K_t = k}\Pr{K_t = k}
\end{align*}
We have
\[
\Pr{K_t = k} = \binom{s}{k}\br{\frac{1}{t}}^k\br{1 - \frac{1}{t}}^{s-k}
\]
The number of arrangements of $s$ items such that some specific $k$ items fall in the first $k$ positions is $k!(s-k)!$. Thus we have
\begin{align*}
\Prr{\text{\mmrs}}{\bigwedge_{i=1}^s \br{R_i = r_i}} &= \binom{s}{k}\br{\frac{1}{t}}^k\br{1 - \frac{1}{t}}^{s-k}\frac{k!(s-k)!}{s!}\\
													 &= \br{\frac{1}{t}}^k\br{1 - \frac{1}{t}}^{s-k}\\
													 &= \Prr{\text{\mrs}}{\bigwedge_{i=1}^s \br{R_i = r_i}}
\end{align*}
which completes the argument.

\section{Additional Experimental Results}
\label{appsec:expt-supp}
Here we present experimental results on 14 different benchmark datasets (refer to Figure~\ref{fig:auc-expts-additional}) comparing the \olp algorithm using the \mmrs buffer policy with the OAM$_{\text{gra}}$ algorithm using the \rsorig buffer policy. We continue to observe the trend that \olp performs competitively to OAM$_{\text{gra}}$ while enjoying a slight advantage in small buffer situations in most cases.

\begin{figure*}[t]
	\centering
	\subfigure[Fourclass\hspace*{-6ex}]{
		\includegraphics[width=0.27\linewidth]{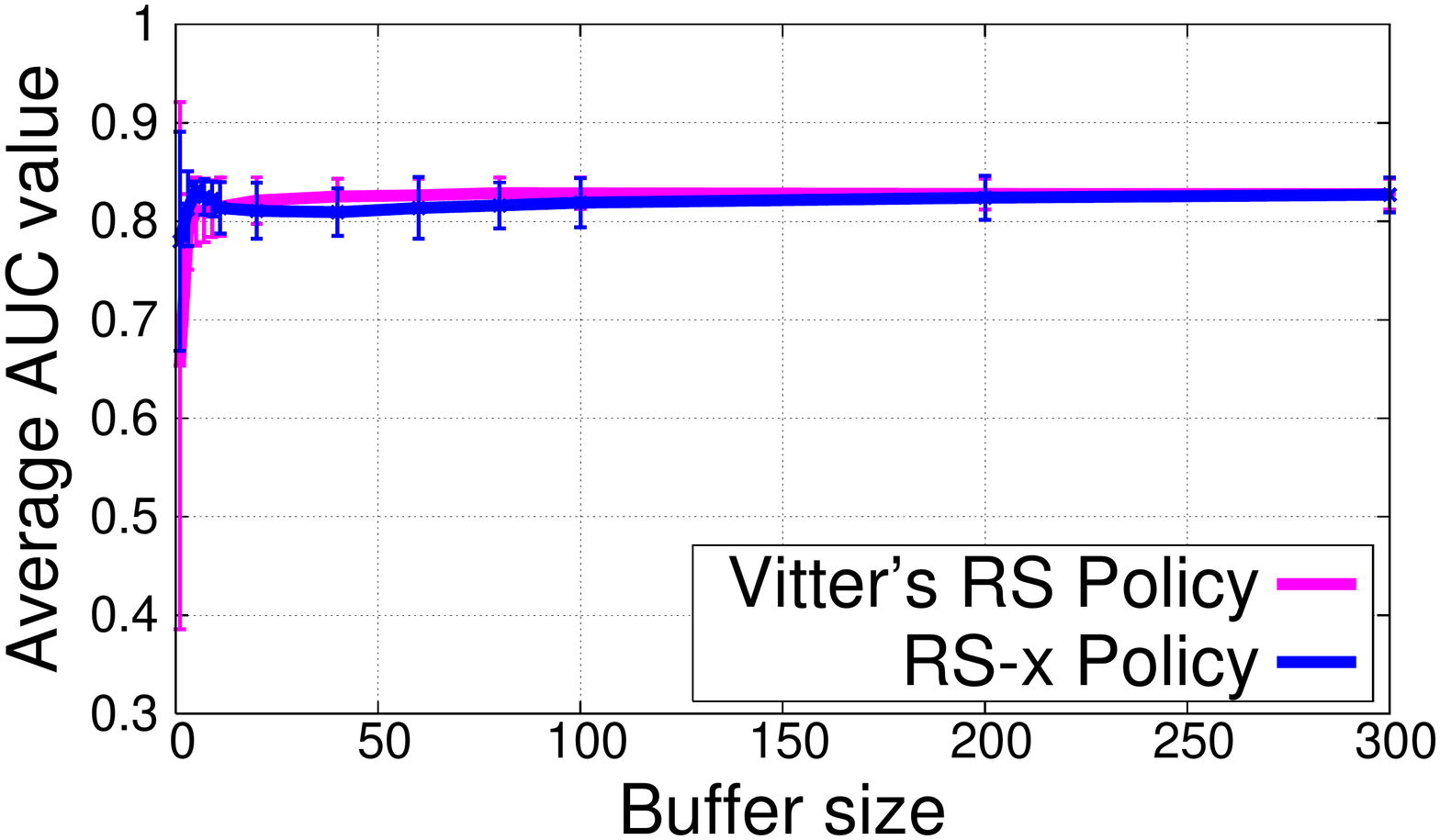}\hspace*{-6ex}
		\label{subfig:fourclass-res}
	}
	\subfigure[Liver Disorders\hspace*{-6ex}]{
		\includegraphics[width=0.27\linewidth]{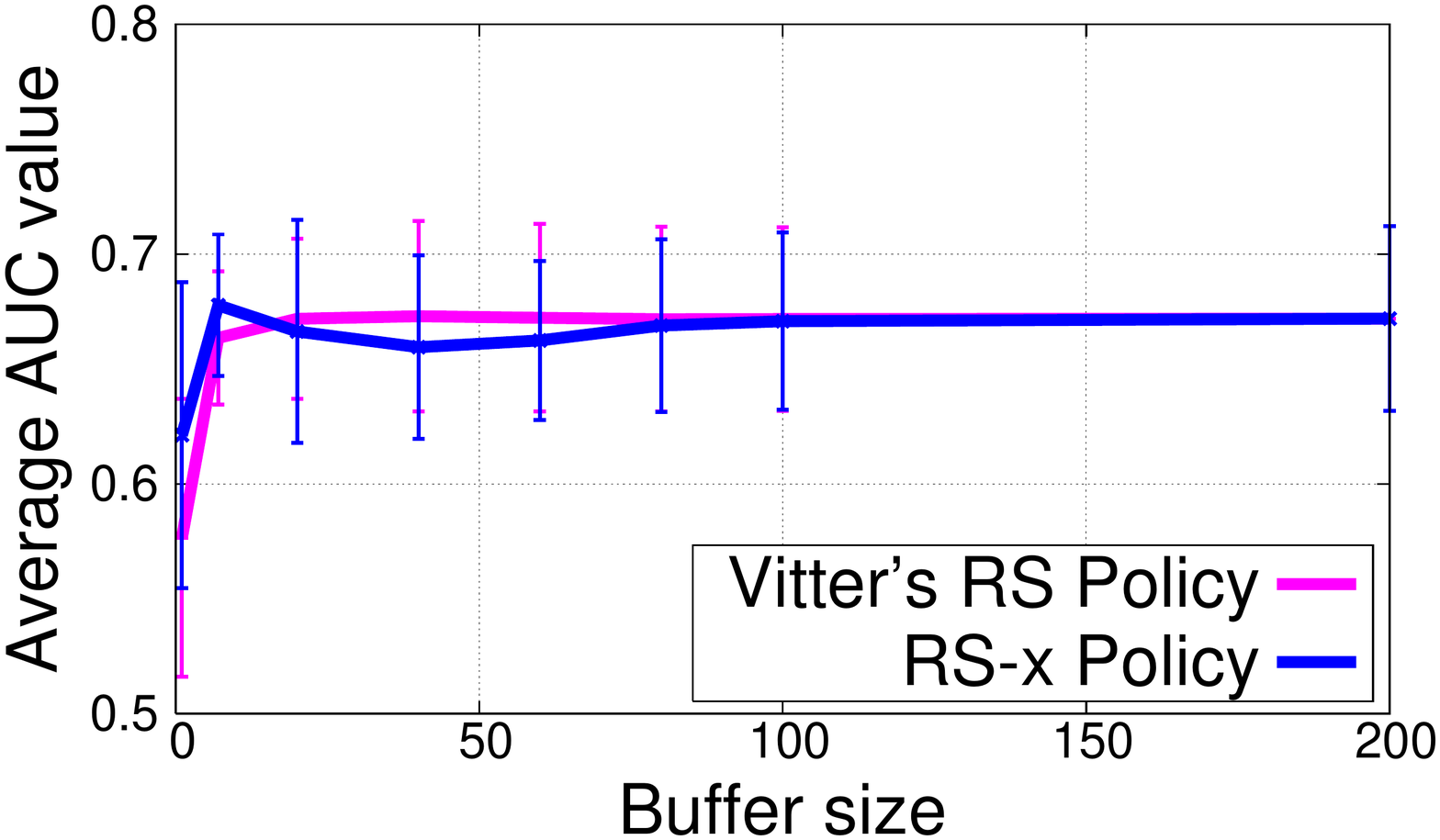}\hspace*{-6ex}
		\label{subfig:liver-disorders-res}
	}
	\subfigure[Heart\hspace*{-6ex}]{
		\includegraphics[width=0.27\linewidth]{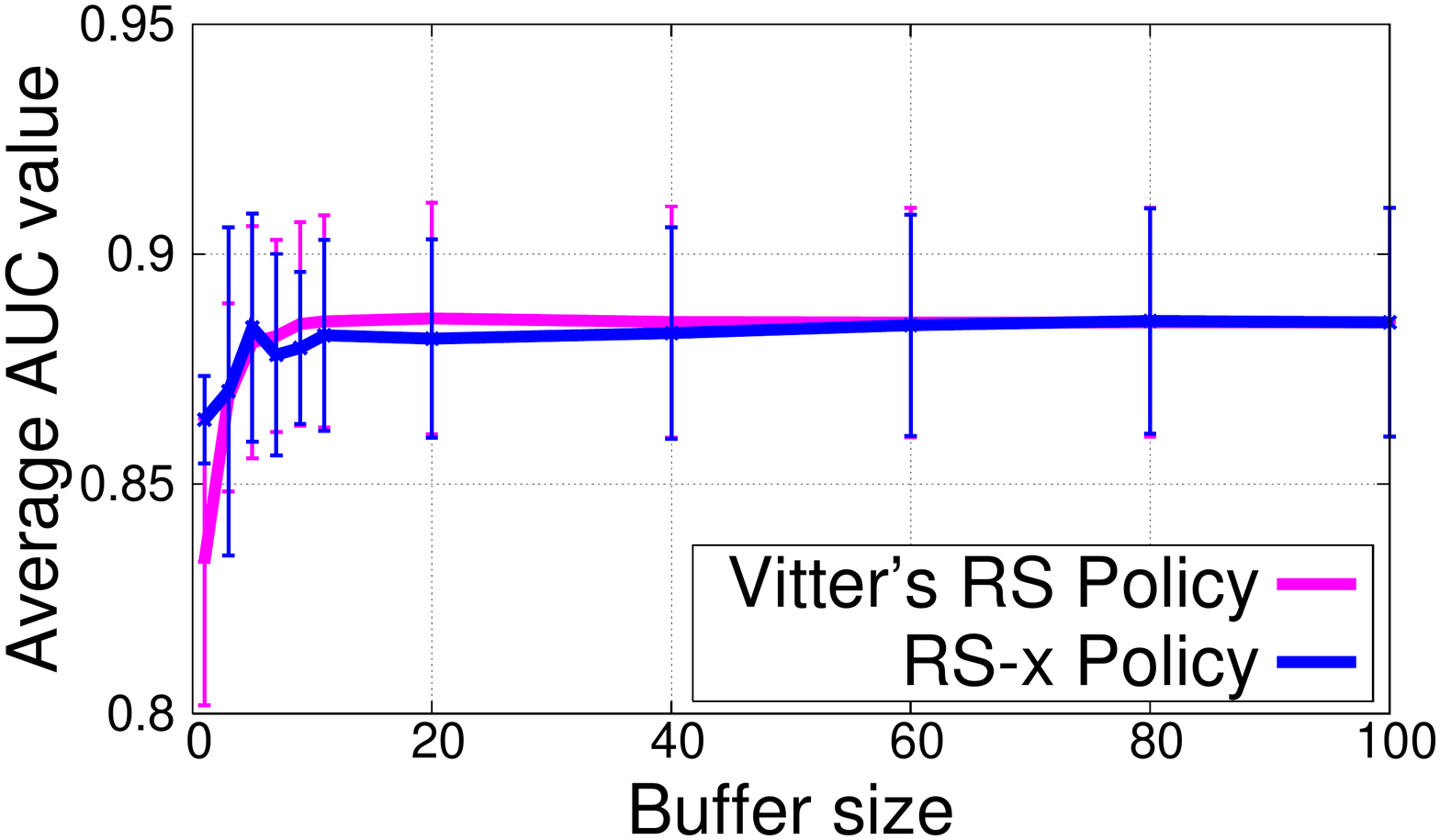}\hspace*{-6ex}
		\label{subfig:heart-res}
	}
	\subfigure[Diabetes]{
		\includegraphics[width=0.27\linewidth]{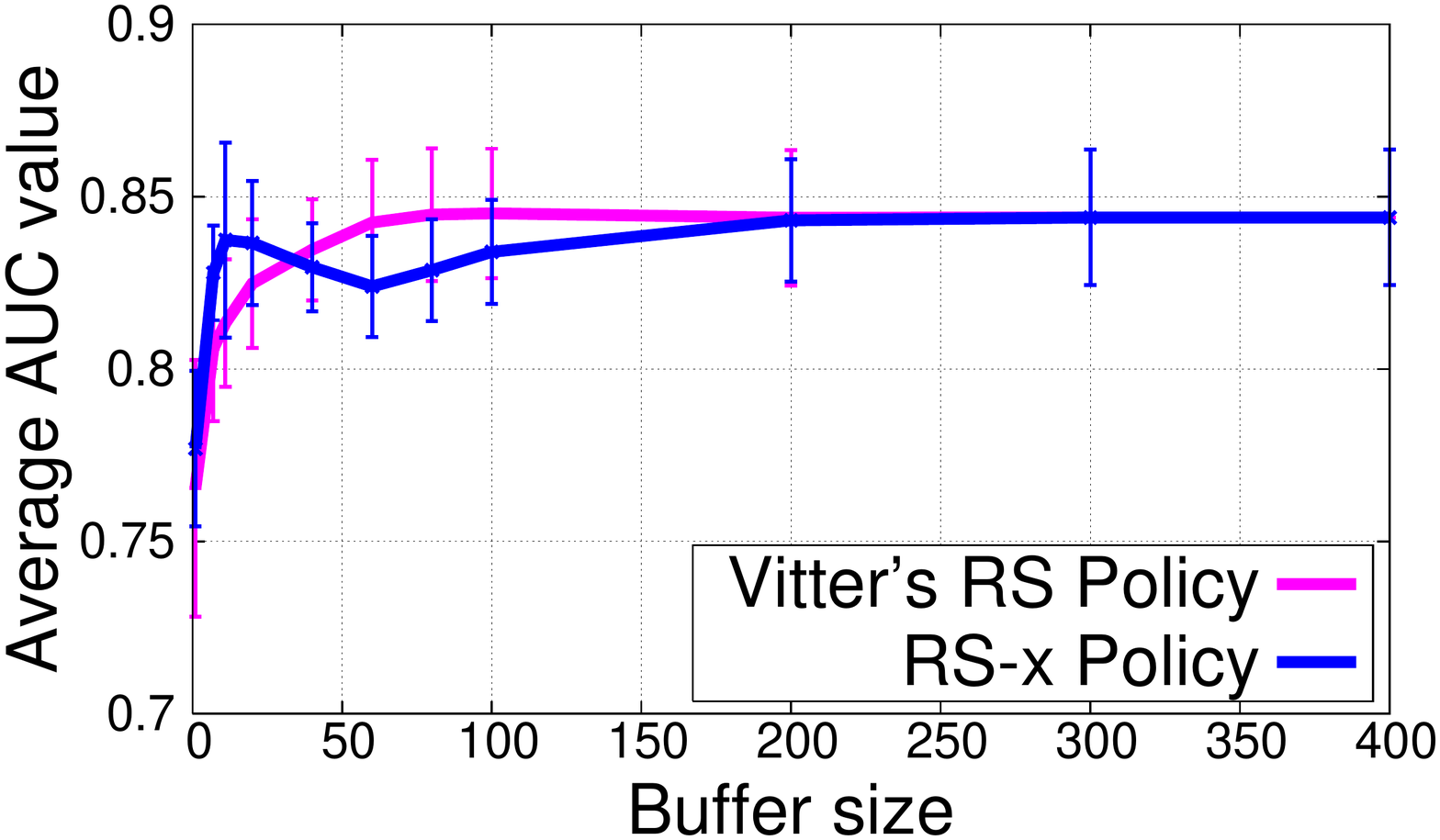}
		\label{subfig:diabetes-res}
	}
	\subfigure[Breast Cancer\hspace*{-6ex}]{
		\includegraphics[width=0.27\linewidth]{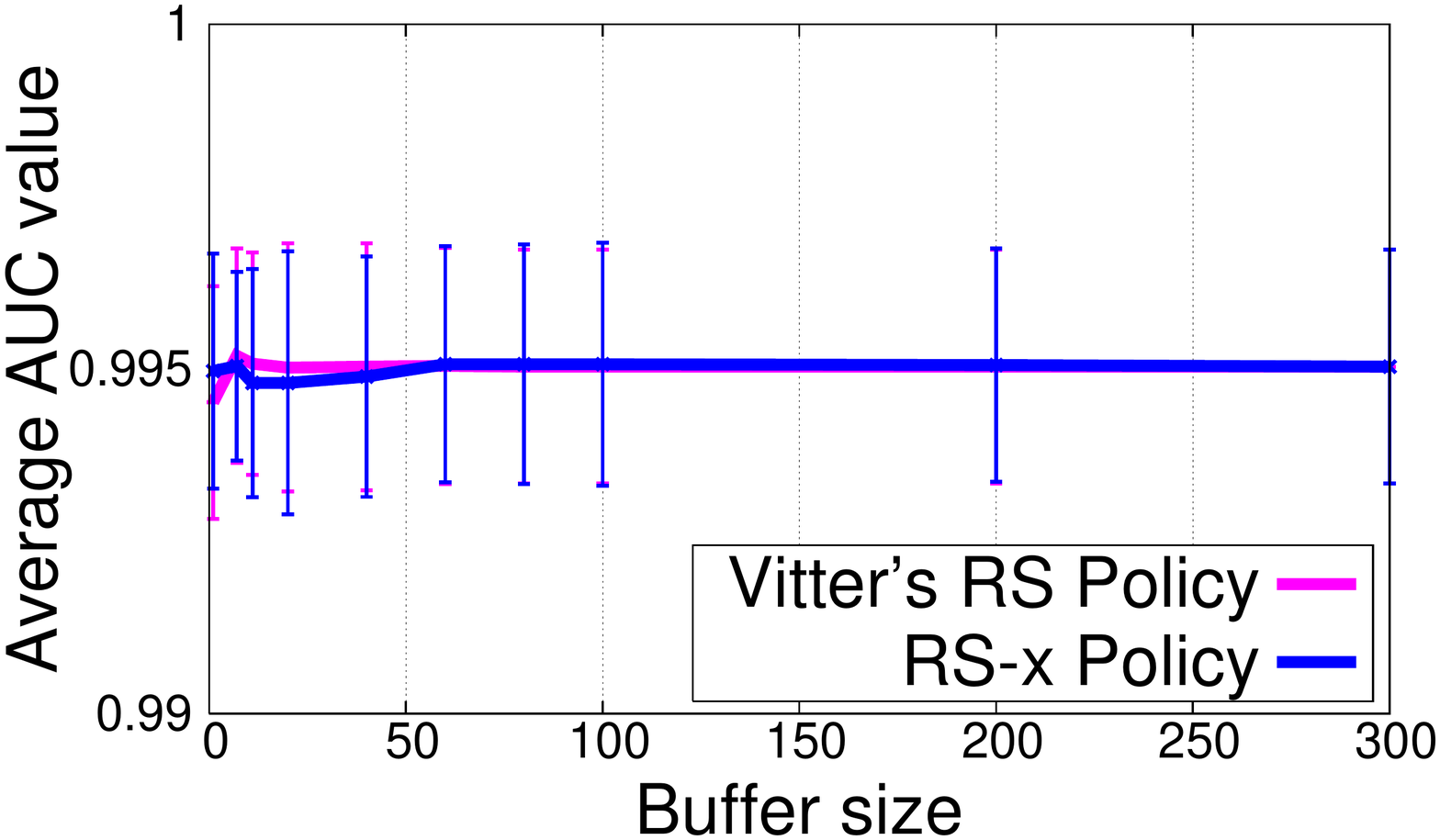}\hspace*{-6ex}
		\label{subfig:breast-cancer-res}
	}
	\subfigure[Vowel\hspace*{-6ex}]{
		\includegraphics[width=0.27\linewidth]{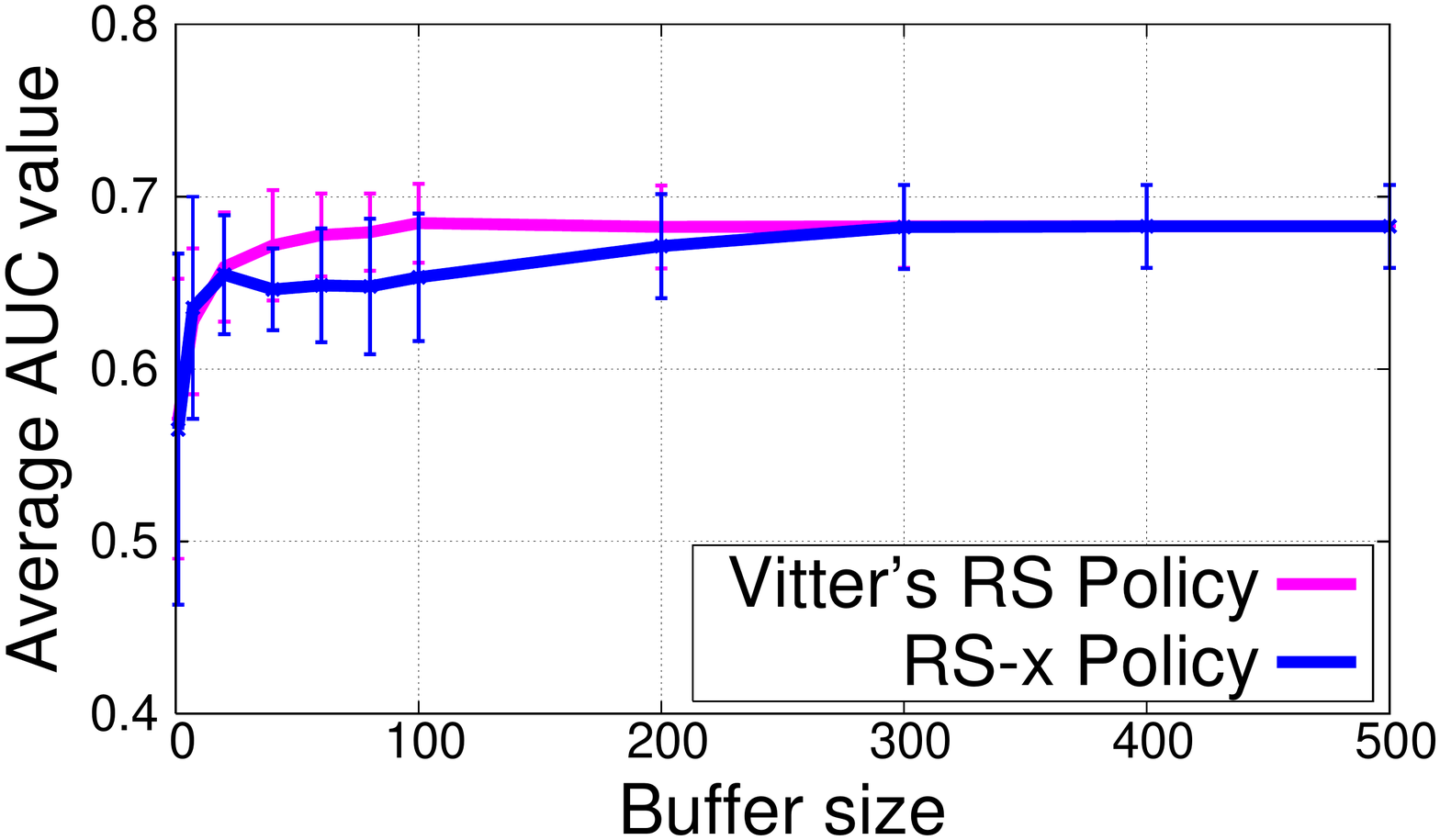}\hspace*{-6ex}
		\label{subfig:vowel-res}
	}
	\subfigure[Ionosphere\hspace*{-6ex}]{
		\includegraphics[width=0.27\linewidth]{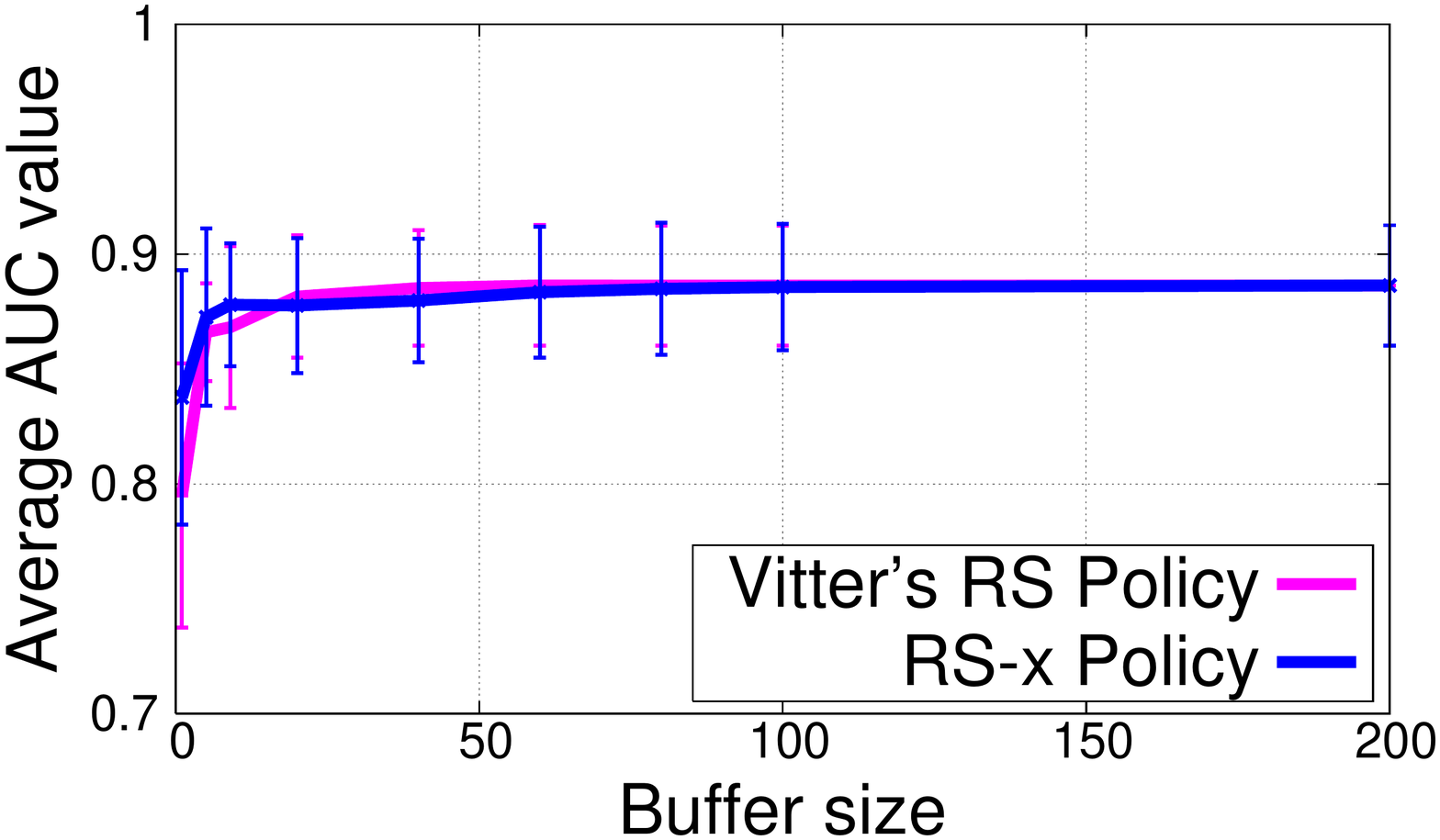}\hspace*{-6ex}
		\label{subfig:ionosphere-res}
	}
	\subfigure[German]{
		\includegraphics[width=0.27\linewidth]{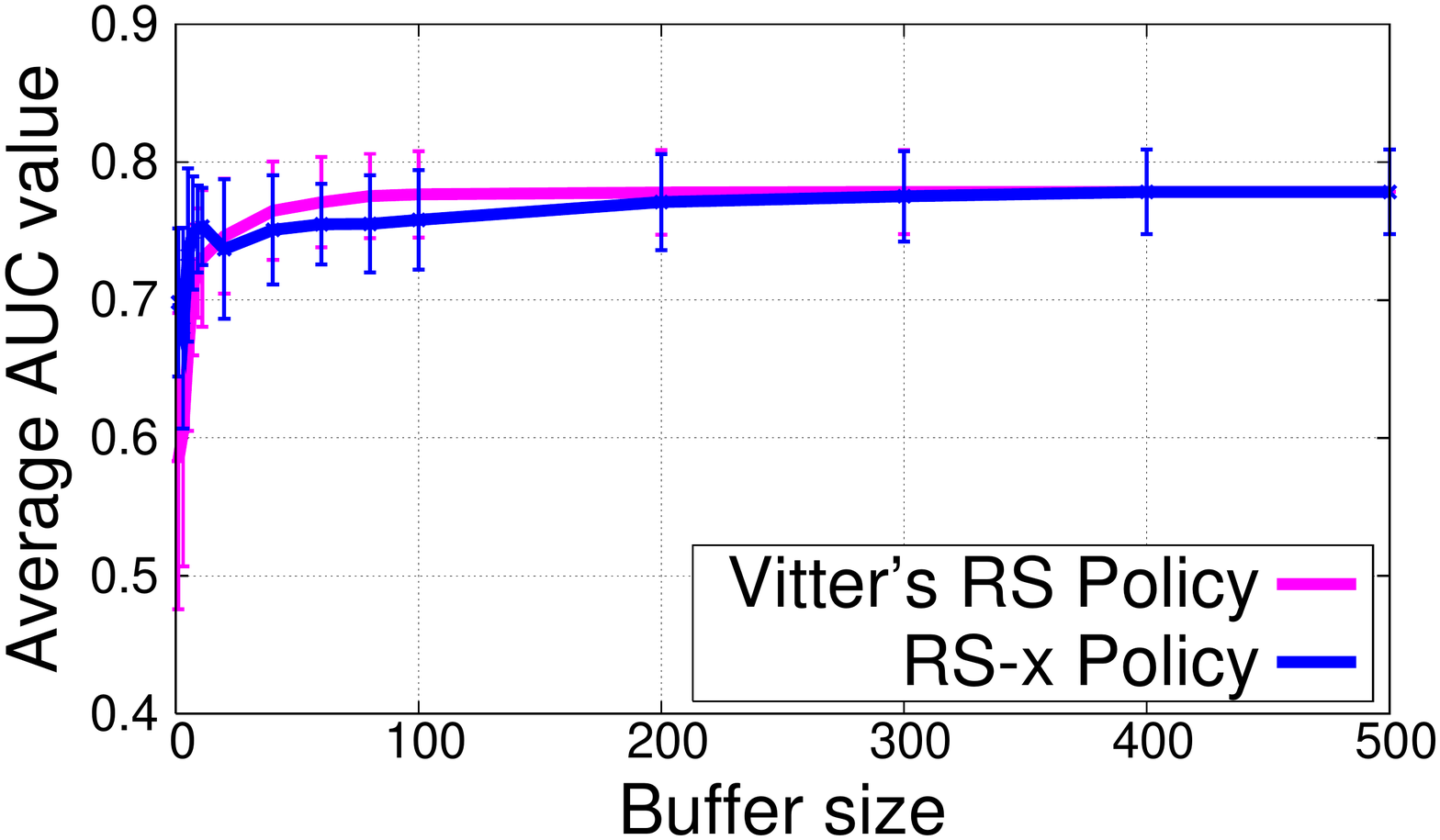}
		\label{subfig:german-res}
	}
	\subfigure[SVMguide3\hspace*{-6ex}]{
		\includegraphics[width=0.27\linewidth]{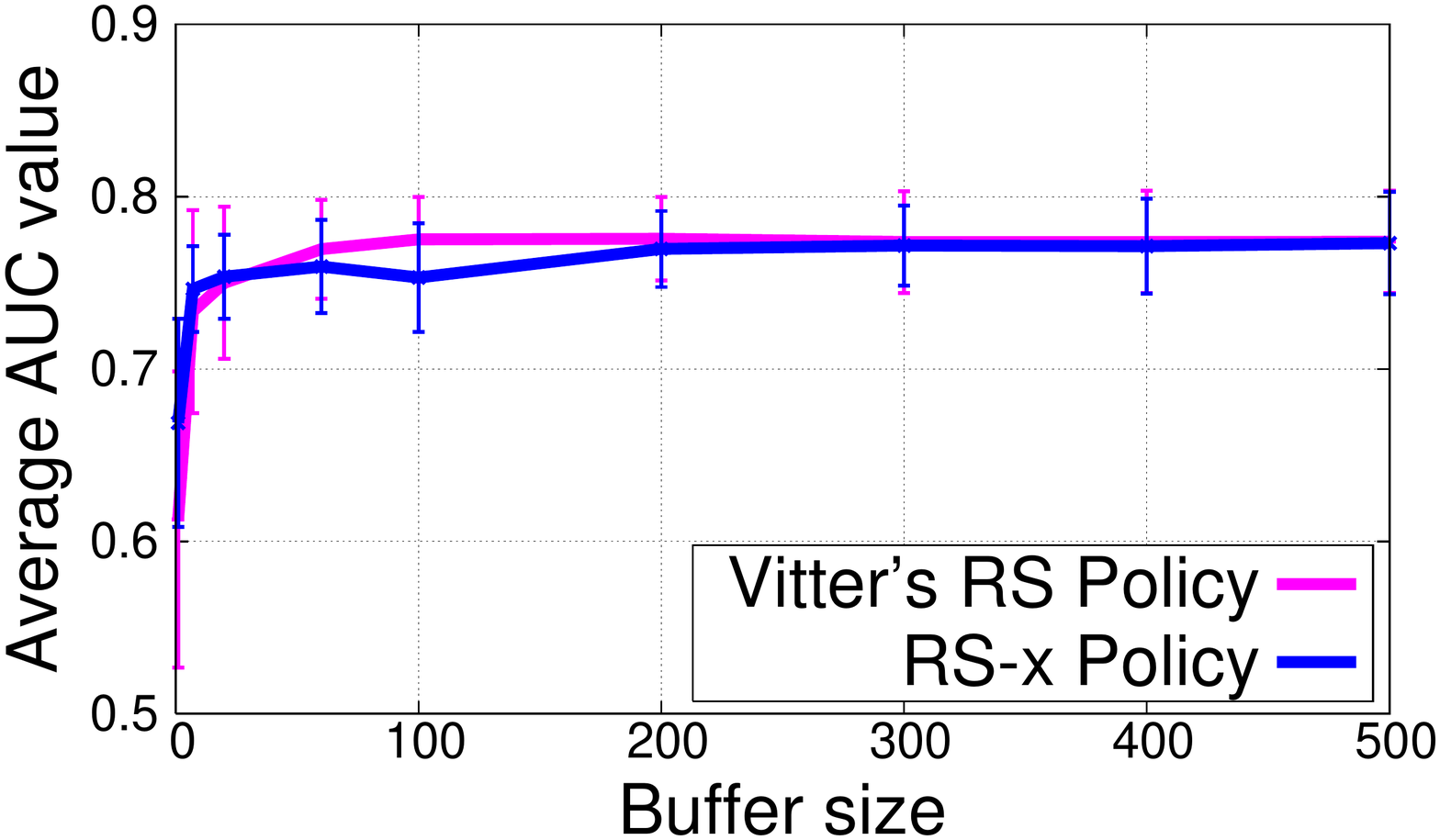}\hspace*{-6ex}
		\label{subfig:svmguide3-res}
	}
	\subfigure[SVMguide1\hspace*{-6ex}]{
		\includegraphics[width=0.27\linewidth]{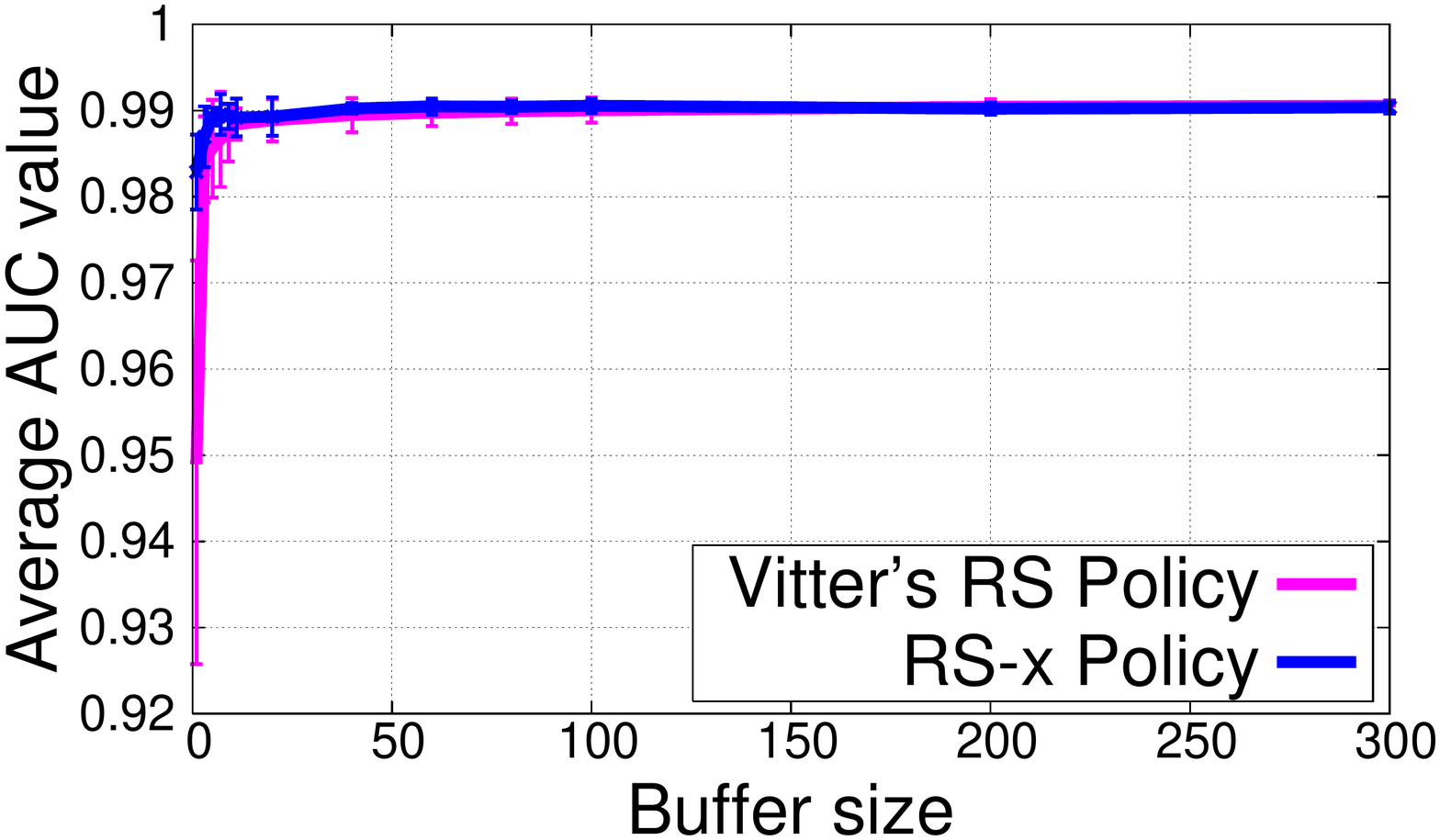}\hspace*{-6ex}
		\label{subfig:svmguide1-res}
	}
	\subfigure[Statlog\hspace*{-6ex}]{
		\includegraphics[width=0.27\linewidth]{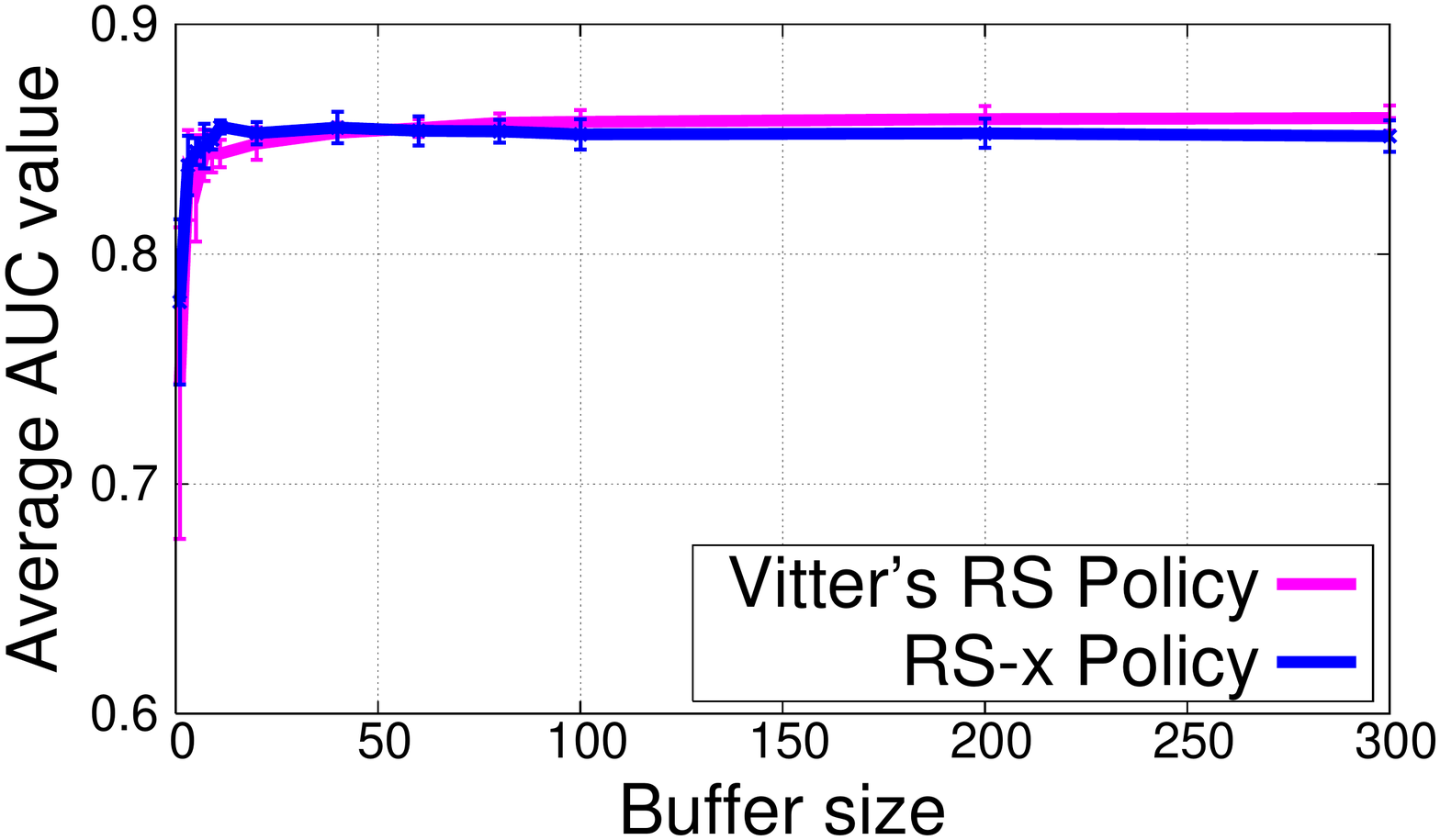}\hspace*{-6ex}
		\label{subfig:statlog-res}
	}
	\subfigure[Cod RNA]{
		\includegraphics[width=0.27\linewidth]{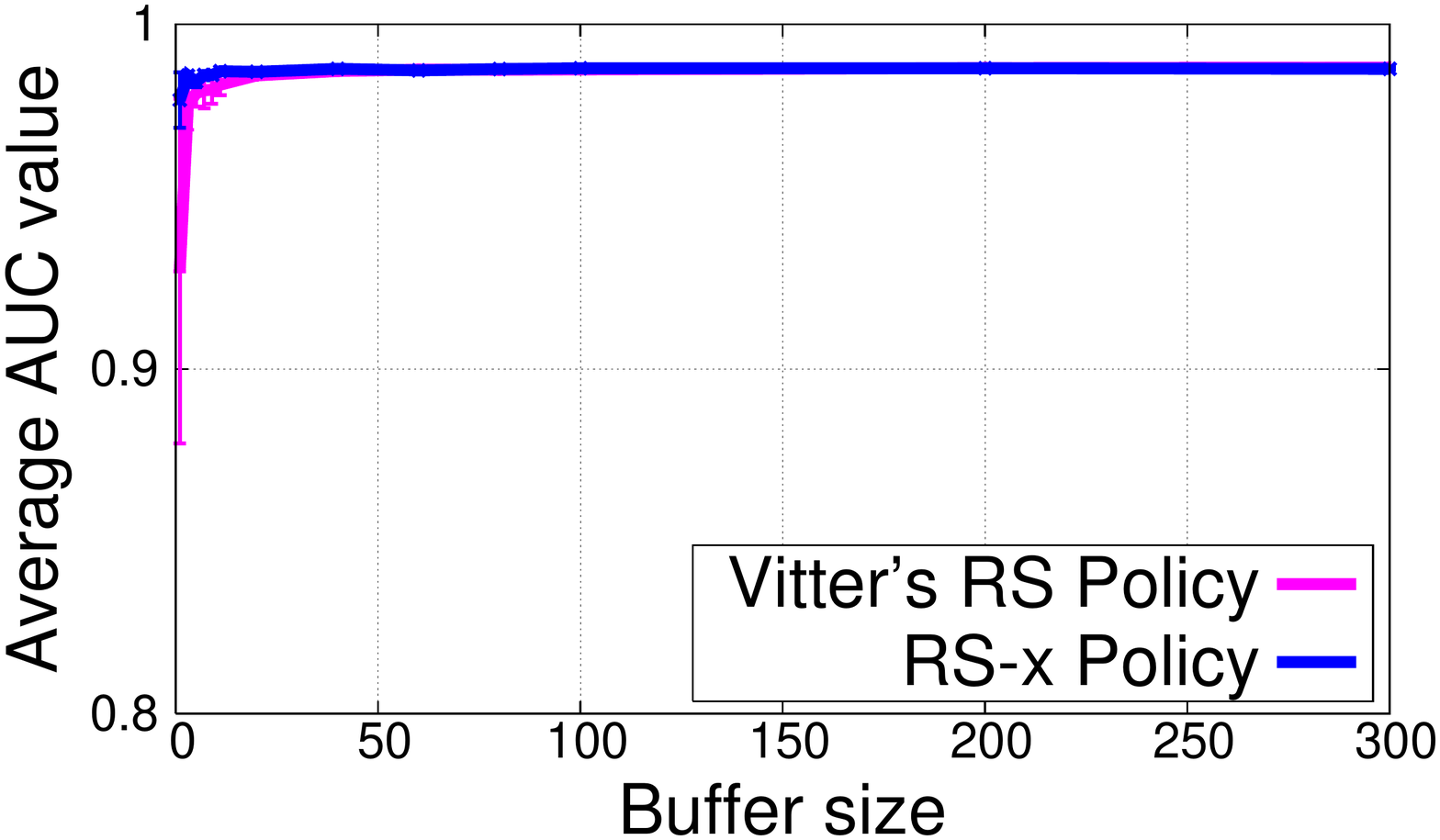}
		\label{subfig:cod-rna-res}
	}
	\subfigure[Letter\hspace*{-6ex}]{
		\includegraphics[width=0.27\linewidth]{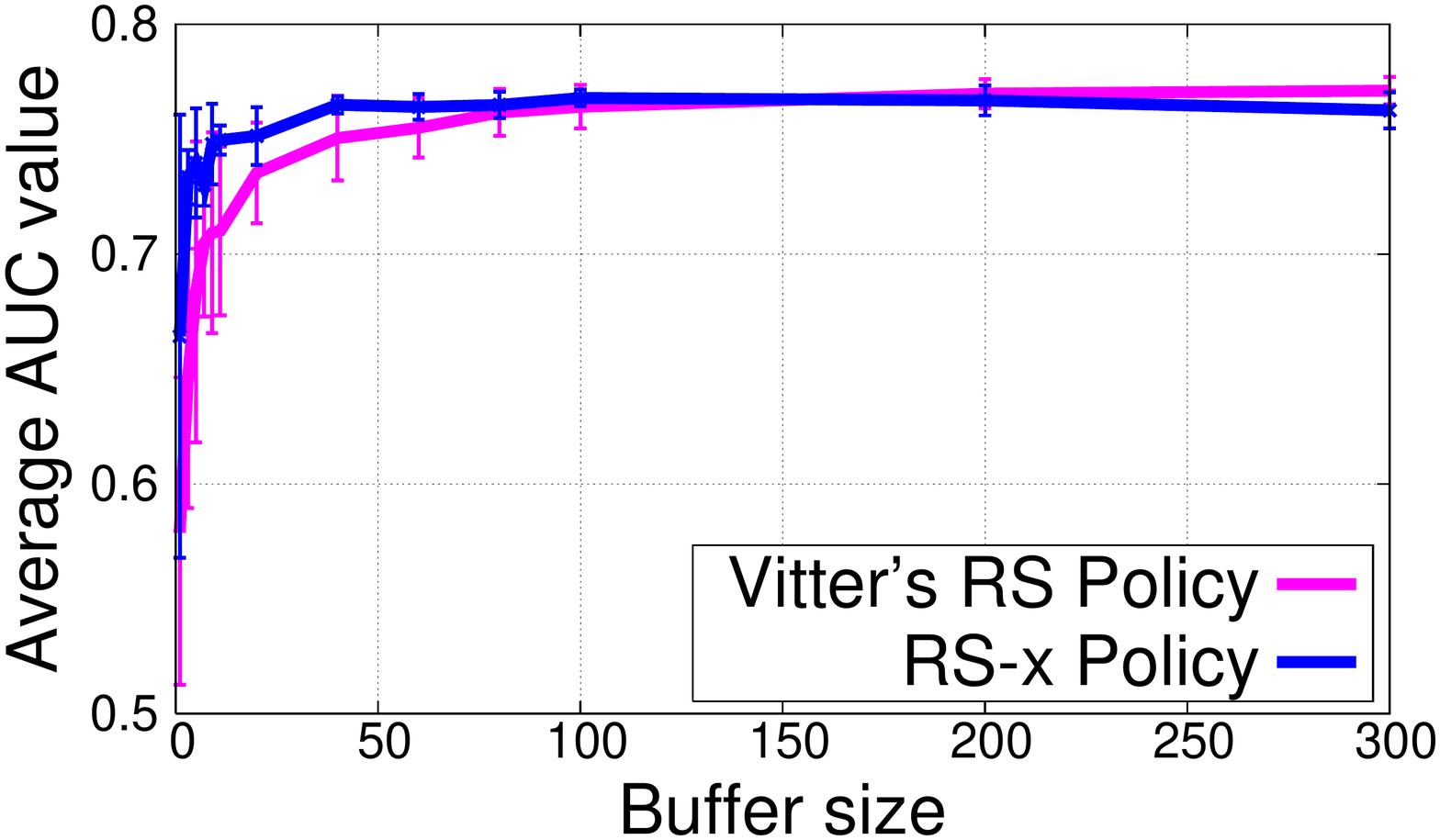}\hspace*{-6ex}
		\label{subfig:letter-res}
	}
	\subfigure[Adult\hspace*{-6ex}]{
		\includegraphics[width=0.27\linewidth]{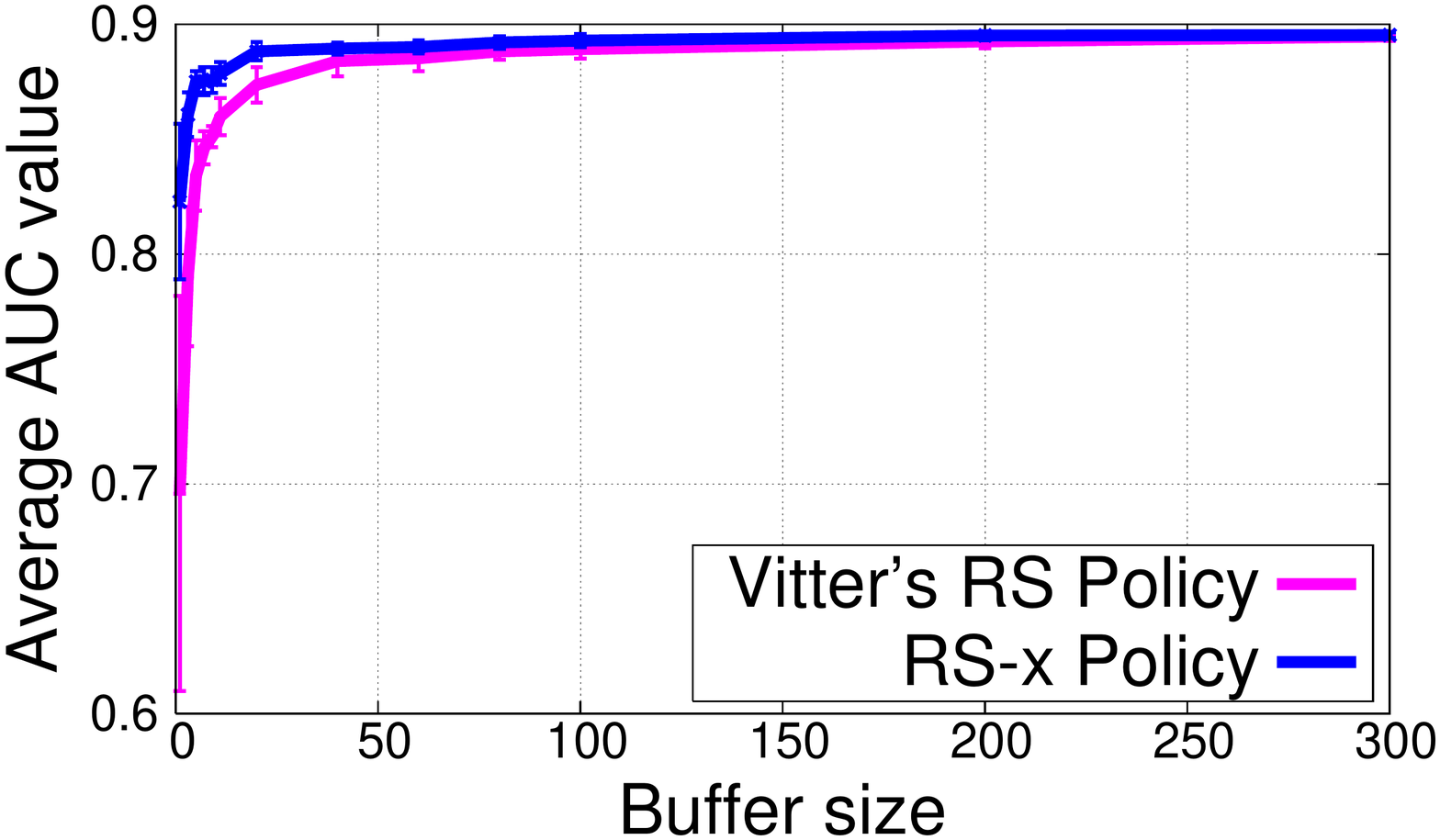}\hspace*{-6ex}
		\label{subfig:a9a-res}
	}
	\caption{Comparison between OAM$_{\text{gra}}$ (using \rsorig policy) and \olp (using \mrs policy) on AUC maximization tasks.}
	\label{fig:auc-expts-additional}
\end{figure*}

\end{document}